\documentclass[11pt]{article}
\usepackage[a4paper,top=3cm,bottom=2cm,left=3cm,right=3cm,marginparwidth=1.75cm]{geometry}

\usepackage[defaultlines=4,all]{nowidow}
\setnowidow[4]
\setnoclub[4]

\usepackage{libertine}

\usepackage{microtype}
\usepackage{graphicx}
\usepackage{subfigure}
\usepackage{booktabs} 

\usepackage{amsmath}
\usepackage{amssymb}
\usepackage{mathtools}
\usepackage{amsthm}

\usepackage{amsmath,amsfonts,bm}
\usepackage{amsthm}
\usepackage{thmtools,thm-restate}

\def\bbR{\mathbb{R}}
\def\wce{\mathbf{wce}}
\def\wge{\mathbf{wge}}
\def\err{\mathbf{err}}

\def\wb{\mathbf{Weibull}}

\newtheorem{theorem}{Theorem}
\newtheorem{corollary}{Corollary}

\def\eqref#1{equation~\ref{#1}}

\def\1{\bm{1}}

\DeclareMathAlphabet{\mathsfit}{\encodingdefault}{\sfdefault}{m}{sl}
\SetMathAlphabet{\mathsfit}{bold}{\encodingdefault}{\sfdefault}{bx}{n}

\DeclareMathOperator*{\argmax}{arg\,max}
\DeclareMathOperator*{\argmin}{arg\,min}

\usepackage{booktabs} 
\usepackage{caption}
\newtheorem{assumption}{Assumption}
\newtheorem{definition}{Definition}
\newtheorem{lemma}{Lemma}
\usepackage{amssymb}

\usepackage{xcolor}
\definecolor{mydarkblue}{rgb}{0,0.3,0.5}

\usepackage[round]{natbib}

\usepackage{url}
\usepackage{graphicx}

\usepackage[pagebackref=true,colorlinks=true,linkcolor=mydarkblue,citecolor=mydarkblue,urlcolor =mydarkblue]{hyperref}

\usepackage{url}
\usepackage{caption}

\usepackage[capitalize,noabbrev]{cleveref}

\title{\textbf{Why does Throwing Away Data Improve Worst-Group Error?}}

\def\andd{%
        \end{tabular}\hfil\linebreak[4]\hfil%
        \begin{tabular}[t]{c}\ignorespaces%
}

\date{}
\author{
  Kamalika Chaudhuri\thanks{Equal contribution} \\
  FAIR (Meta AI) and UC San Diego\\
  \and
  Kartik Ahuja\footnotemark[1]\\
  FAIR (Meta AI)\\
  \andd
  Martin Arjovsky\\
  Inria\\
  \and
  David Lopez-Paz\\
  FAIR (Meta AI)\\
}
\begin{document}
\maketitle

\begin{abstract}
  When facing data with imbalanced classes or groups, practitioners follow an intriguing strategy to achieve best results. They throw away examples until the classes or groups are balanced in size, and then perform empirical risk minimization on the reduced training set.
  This opposes common wisdom in learning theory, where the expected error is supposed to decrease as the dataset grows in size.
  In this work, we leverage extreme value theory to address this apparent contradiction.
  Our results show that the tails of the data distribution play an important role in determining the worst-group-accuracy of linear classifiers.
  When learning on data with heavy tails, throwing away data restores the geometric symmetry of the resulting classifier, and therefore improves its worst-group generalization. 
\end{abstract}

\section{Introduction}

Imbalances are ubiquitous in real-world data.
On the one hand, class imbalance is common in rare-event data such as medical diagnosis, intrusion detection, spam classification, or credit fraud \citep{johnson2019survey}.
On the other hand, imbalances may also exist within the classes of our problem, if each of these comprises hidden groups with different proportions.
For instance, in an image classification dataset with balanced classes, most pictures for each class are commonly taken in wealthy countries~\citep{rojas2022the}.
In all of these situations, simply minimizing average training error may result in classifiers that perform very well on majority groups, while having poor  performance on minority groups~\citep{buolamwini2018gender}.
For example, the $1\%$ test error of a classifier could mean that \emph{all} examples from a particular minority group are misclassified.
As such, this work is concerned with the `worst-group-error' of classification rules on imbalanced data.

While many sophisticated methods have been proposed to address imbalanced classification problems~\citep{sagawa2019distributionally}, none of these offer a clear advantage to simple subsampling~\citep{IdrissiSimple2022}.
In subsampling, we throw away data from large groups until they match the smallest group in size.
Then, we perform empirical risk minimization on the (often drastically) reduced data.
This is a surprising finding, as classical wisdom in learning theory tells us that the expected error of classifiers should decrease as training data grows in size.
As it stands, there is no theoretical explanation as to why the popular strategy of subsampling works so well when addressing imbalanced classification problems.
In particular, commonplace PAC-learning results~\citep{haussler1990probably} upper-bound the error of the worst of the classifiers correctly classifying all of the training data. While this analysis can be extended to other metrics such as worst-group error, since throwing away examples increases the amount of such compatible classifiers, the error bound can only worsen, falling short in explaining the empirical benefits of data subsampling.

This work is an initial effort to resolve this apparent tension between theory and practice.
More specifically, we focus our analysis on linear maximum-margin binary classifiers~\citep{steinwart2008support}.
Under this setup, it is well known~\citep{BB00} that maximum-margin classifiers are equidistant from the convex hulls delineated by each of the two classes.
In turn, the shape of these convex hulls is determined by the extremal properties of the probability distributions of each class.
These extremal properties are the subject of study of a branch of mathematics called extreme value theory (EVT, \citet{de2007extreme}).
By borrowing results from EVT, we show that the location of the maximum-margin classifier depends on the tail properties of the data distribution.
In particular, when facing data distributions with heavy tails, we observe geometric imbalances when groups differ in sample size.
These imbalances skew the maximum-margin classifier, leading to suboptimal worst-group-error. This ends up being the reason why subsampling works, as balancing groups in size  restores geometric symmetry across groups.

\begin{figure}
     \centering
     \includegraphics[trim=0 1em 0 0, width=\textwidth]{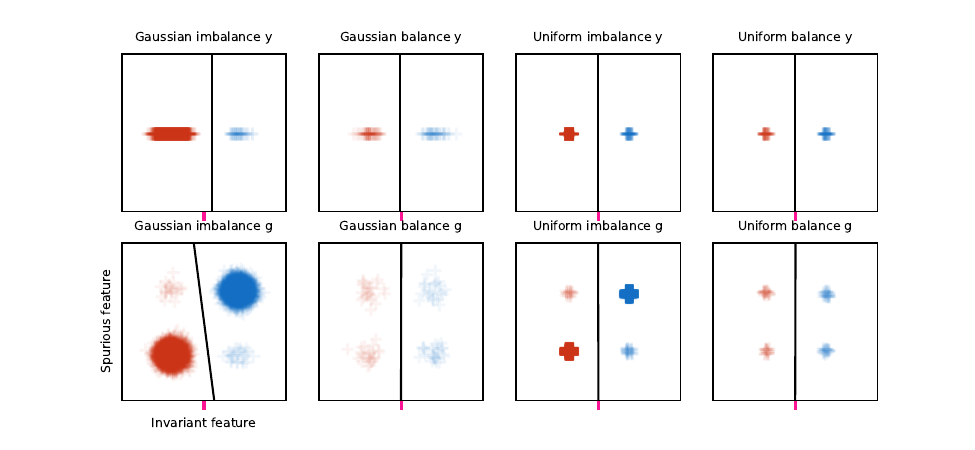}
     \caption{%
      \small{ Illustration of the phenomena studied in this paper.
       The tails of the data distribution can bias the maximum-margin classifier, and throwing away is a tool to restore balance and achieve the optimal solution.
       The top four plots illustrate a classification problem with two imbalanced classes.
       In the Gaussian case, throwing away data to balance class sizes leads to the optimal maximum-margin classifier, which is aligned with magenta checkmark.
       In the Uniform case, throwing away data does not make any difference.
       The bottom four plots illustrate a binary classification problem, where each class is balanced but contains two groups of different proportions.
       Once again, in the Gaussian case, throwing away data to balance group sizes leads to the optimal maximum-margin classifier.
       Subsampling does not provide any benefit in the Uniform case.}
     }
     \label{fig:main_figure}
\end{figure}

We illustrate these phenomena in Figure~\ref{fig:main_figure}.
When dealing with imbalanced classes, we find that the bias term in the maximum-margin classifier is shifted towards the small class.
When reducing the data with subsampling to balance the two classes, the maximum-margin classifier converges to the  unbiased one with optimal worst-group-error. 
When dealing with imbalanced groups, the direction of the maximum-margin classifier is biased towards the smaller group, increasing its error in test examples from the tails of the distribution.
Figure~\ref{fig:main_figure} shows that, while these phenomena happen when dealing with data distributions with Gaussian tails, they are inexisting when dealing with Uniform data distributions, with no tails.

\paragraph{Contributions}
This work proposes a novel theoretical analysis to understand imbalanced classification, as well as to justify the popularity of data subsampling in this regime (Section~\ref{sec:preliminaries}).
To this end, we introduce the use of extreme value theory to construct a new type of generalization analysis that focusses on distributional tails (Section~\ref{sec:evt}).
These results allow us to characterize the impact of distributional tails on the worst-group error for both ERM and subsampling strategies.
We conduct separate analyses to understand the case of imbalanced classes (Section~\ref{sec:classes}) and groups (Section~\ref{sec:groups}).
In particular,
\begin{itemize}
  \item Subsampling outperforms ERM in worst-group-error when learning from imbalanced classes with tails (such as Gaussians, Theorem \ref{thm:gaussians}), while it makes no difference when learning from distributions without tails (such as Uniforms, Theorem \ref{thm:weibull1d}).
  \item Similar results follow for balanced classes but imbalanced groups (Theorem  \ref{thm: spurious_gumbel} for groups with tails, Theorem  \ref{thm: invariant_weibull} for groups without tails).
  \item We extend these results to the high-dimensional case where there exist a multitude of ``noise'' dimensions polluting the data (Theorems \ref{thm: imb_class_highd},  \ref{thm: spurious_d}, \ref{thm: invariant_d}) and provide empirical support for our theories using Waterbirds and CelebA, the two most common datasets to benchmark worst-group-error (Section \ref{sec: expmts}).
\end{itemize}

\section{Our learning setup}
\label{sec:preliminaries}

We consider the binary classification of examples $(x, y)$, where instances $x \in \bbR^d$ and labels $y \in \{ -1, +1\}$. We assume that instances with label $y$ are drawn independently from a class-conditional distribution that we denote by $D^y$. 
We look at two settings---imbalanced classes and imbalanced groups. In the setting of imbalanced classes, there is a majority class with $y = -1$ and a minority class with $y = +1$. Training data $\mathcal{S}$ consists of $p$ samples drawn independently and identically distributed (iid) from the majority class-conditional distribution, and $m$ from the minority class-conditional distribution. We assume $p$ is much larger than $m$.
For imbalanced groups, we extend each example to be a triplet $(x, y, a)$, where $a \in \{-1, 1\}$ is a binary attribute. Label-attribute combinations induce four groups $g = (y, a)$. Then, the data distribution consists of $p$ points drawn iid from the majority groups $g = (1, 1)$ and $(-1, -1)$, and $m$ points drawn from the minority groups $g = (-1, 1)$ and $(1, -1)$. Once again, $p$ is much larger than $m$.
While describing data distributions, we use the notation $D(\rho)$ to denote a distribution with mean $\rho$.

Commonly, in machine learning we are interested in the generalization error of a classifier $\theta$ on some data distribution $D$,  $\err(\theta) = \Pr_D(\theta(x) \neq y)$, where we drop the subscript $D$ when this causes no confusion.
For problems with imbalanced classes, we are chiefly interested in measuring \emph{worst-class error}:  $\wce(\theta) = \max_{\tilde{y}}\, \Pr(\theta(x) \neq y \mid y = \tilde{y}).$ For imbalanced groups, the metric of interest is the {\em{worst-group error}}: $  \wge(\theta) = \max_{\tilde{g}}\, \Pr(\theta(x) \neq y \mid g = \tilde{g}).$

To enable a fine-grained analysis, we focus on a linear classifier, namely $\theta(x) = w^\top x + b$, where $w$ is the weight and $b$ is the bias.  We assume that the training data is linearly separable with high probability---even though the entire distribution may not be. This mirrors what happens in deep learning, where zero-training-error is easy to achieve, while zero-test-error may be elusive. 
The usual way to train a classifier is to follow the Empirical Risk Minimization (ERM) principle to learn the parameters $(w, b)$ separating the training data $\mathcal{S}$.
This amounts to finding $(w, b)$ that minimizes 
\begin{equation}
  L(w, b, \mathcal{S}) = \frac{1}{|\mathcal{S}|} \sum_{i=1}^{|\mathcal{S}|} \ell\left((w^\top x_i + b) y_i\right),
  \label{eq:erm}
\end{equation}
where $\ell : \mathbb{R} \to [0, \infty]$ is a loss function that penalizes classification errors. The most popular loss for classification in deep learning is the logistic loss $\ell(u) = \log(1 + \exp(-u))$.

\paragraph{Support vector machines}
When our training data is linearly separable,~\cite{soudry} has shown that the logistic loss minimizer converges to the well-known maximum-margin classifier also referred to as the linear hard-margin Support Vector Machine~\citep[SVM]{vapnik1999nature}. The sequel therefore analyzes the properties of the SVM separator directly. The SVM separator characterized by $w^*$ and $b^*$ is known to be equidistant from the convex hulls of the positive and negative classes~\citep{BB00}; it can be obtained by solving:
\begin{equation}
\begin{split}
&     w^{*} = \argmax_{\|w\|=1}  \Big(\inf_{x  \in B} w^{\top}x - \sup_{x  \in A} w^{\top}x  \Big), \\ 
& b^{*} = -\frac{1}{2}\Big(\inf_{x  \in B} w^{\top}x + \sup_{x  \in A} w^{\top}x  \Big), 
\end{split}
\end{equation}
where $B$ is set of positively labeled points and $A$ is the set of negatively labeled points. Our analysis connects the optimal $w^*$ and $b^*$ to the tail properties of the data distribution.

\paragraph{ERM versus subsampling}
We will compare two algorithms.
On the one hand, given a training dataset with imbalanced classes or groups, the ERM algorithm directly solves the SVM optimization problem on the entire dataset to get the classifier $\theta_{\text{erm}}$.
On the other hand, the subsampling algorithm solves the SVM problem on a {\em{subsample}} of the training data that balances out classes or groups. Specifically, subsampled training sets consists of the entire minority class or group, and a randomly drawn sample of size $m$ from the majority class or group. Solving the SVM on this reduced dataset gives us the classifier $\theta_{\text{ss}}$.

\paragraph{Why are PAC bounds not enough?}
We will show that subsampling leads in certain cases to a strictly better worst-class or worst-group error compared with plain ERM on the entire training data.
In other words, throwing away data strictly helps! 
This kind of result cannot be directly obtained through standard PAC-style analysis in learning theory. Basic PAC-style analysis in the realizable case provides an upper bound on the worst-case error of any classifier in the {\em{version space}}---which is the set of all classifiers that perfectly classify the training points. While this analysis can be adapted to other metrics (such as wce and wge), throwing away data expands the version space, and hence the worst-case error of this expanding set must also increase.

To address this apparent contradiction, the sequel relies on the geometric properties of the maximum-margin classifier. We will use the fact that the maximum-margin separator is equidistant from the convex hulls of the two classes~\cite{BB00}, and geometrically analyze properties of these convex hulls. This analysis leverages the fact that properties of the convex hull of a set of random points are related to extreme value statistics~\cite{BB00} of the distribution. To this end, our analysis makes use of extreme value theory, a branch of probability theory concerned with maxima and minima of distributions. 

\section{Basics of extreme value theory}
\label{sec:evt}

The branch of mathematics studying extreme deviations is known as Extreme Value Theory~\citep[EVT]{extremalbook}. We borrow tools from EVT to analyze the worst-group error in SVMs, which is a first in the research literature.
Specifically, we will make use of a central result from EVT: the Fisher-Tippett-Gnedenko  theorem~\cite{extremalbook}. Suppose we have $n$ iid examples $X_1, \ldots, X_n$ drawn from an fixed distribution with CDF $F$; the Fishet-Tippett-Gnedenko theorem  characterizes the maximum value $M_n$ of $X_1, \ldots, X_n$ provided the distribution $F$ belongs to one of the following types. 

\begin{definition} 
  Let $x_F = \sup_x \{x \mid F(x) < 1 \}$ be the largest value not saturating $F$. Then $F$ is of the family:
	\begin{itemize}
    \item \emph{Weibull}, if $x_F < \infty$ and $\lim_{h \rightarrow 0} \frac{1 - F(x_F - xh)}{1 - F(x_F - h)} = x^{\alpha}$, for $\alpha > 0$ and $x > 0$.
    \item \emph{Gumbel}, if 
		$\lim_{t \rightarrow x_F} \frac{1 - F(t + x g(t))}{1 - F(t)} = e^{-x}$, for all real $x$ where $g(t) = \frac{ \int_{t}^{x_F} (1 - F(u)) du }{1 - F(t)}$ for $t < x_F$. 
	\end{itemize}
  \label{def:threetails}
\end{definition}

Fisher-Tippett-Gnedenko theorem also provides the characterization for a third type of distribution referred to as the Frechet  distributions; in this work we only study Gumbel and Weibull distributions. Gumbel-type distributions have light tails, and include Gaussians and Exponentials. Weibull-type distributions have finite maximum points, and include Uniforms. Frechet-type distributions have heavy tails, including the Pareto and Frechet distributions. Some distributions, such as the Bernoulli, do not belong to any of these types.  If the distribution $F$ is of either of these three types, then the Fisher-Tippett-Gnedenko theorem shows that there exist sequence $a_n$ and $b_n$ such that the CDF $F'$ of $(M_n - b_n)/a_n$ converges to a limit distribution $G$.
\begin{equation*}
F'\bigg(\frac{M_n - b_n}{a_n}\bigg) \rightarrow G,
\end{equation*}
where $a_n$, $b_n$ and $G$ take values that depend on the specific distribution $F$. Finally, we define the ``tail function'' of a distribution with CDF $F$ to measure the spread of its tail.

\begin{definition}[Tail Function]
The tail function $U$ of a distribution with CDF $F$ is $U(t) = F^{-1}(1 - 1/t)$. 
\end{definition}

Observe that the tail function $U(t)$ is an increasing function of $t$; in addition, by definition, we have $F(U(t)) = 1 - 1/t$. We now have all the necessary tools to introduce the Fisher-Tippet-Gnedenko theorem formally.

\begin{theorem}[Fisher-Tippett-Gnedenko Theorem]

	\begin{enumerate}
		\item If $F$ is of the Frechet type, then $G(x)$ is the Frechet distribution with the following CDF:
			\begin{eqnarray*}
				G(x) = & \exp(-x^{-\alpha}), \quad & x \geq 0 \\
				= & 0, \quad & x < 0.
			\end{eqnarray*}
			Additionally, $a_n = U(n)$ and $b_n = 0$. 
		\item If $F$ is of the Weibull type, then $G(x)$ is the reverse Weibull distribution with the following CDF:
			\begin{eqnarray*}
				G(x) =& 1, & x \geq 0 \\
        = & \exp(-(-x)^{\alpha}), & x < 0.
			\end{eqnarray*}
			Additionally, $a_n = x_F - U(n)$ and $b_n = x_F$.
		\item If $F$ is of the Gumbel type, then $G(x)$ is the Gumbel distribution with the following CDF:
		\[ G(x) = \exp(-e^{-x}), \quad x \in [-\infty, \infty]. \]
      Additionally, $a_n = g(U(n))$ and $b_n = U(n)$. 
	\end{enumerate}
  \label{thm:ftg}
\end{theorem}

\section{Analysis of imbalanced classes}
\label{sec:classes}
We first look at data from imbalanced classes. The basic setup is as follows. To keep the main message of our analysis simple, we assume that the positive (minority) class and the negative (majority) class are both distributed according to a distribution $D(\cdot)$, but with shifted means. Specifically, the positive class is distributed according to $D(\mu)$, and the negative class according to $D(-\mu)$, where $\|\mu\|>0$. Additionally, we assume that $D(\mu)$ is symmetric about its mean $\mu$ -- although this symmetry is not strictly needed for our results to hold. Finally, recall that we have $p$ points from the majority class and $m$ from the minority with $p \gg m$.  To build intuition, we first look at a simple one-dimensional setting where the feature $x_i$'s are scalars. In this case,  weight $w$ is set to one, we use $\theta$ to refer to bias $b$. This case has two interesting properties that makes the analysis intuitive.  First, due to symmetry of the class-conditional distributions, the classifier that minimizes worst-class error has a bias of zero.   Second, the bias of the maximum-margin classifier $\theta_{\text{erm}}$ is the mean of the maximum training point with a negative label and the minimum training point with a positive label. 

How does ERM behave under these conditions? Geometry suggests that if, due to class imbalance, training data from the positive majority class is {\em{spread out}} enough to push the bias of $\theta_{\text{erm}}$ away from zero, then ERM will have poor worst-class accuracy, and subsampling will help. In contrast, if  training data from both classes are equally spread-out, then ERM will retain its symmetry. 
We formalize this notion of {\em{spreading out}} through a Concentration Condition below.
\begin{assumption}[Concentration Condition]
\label{assm: concentration interval}
Suppose $x_1, \ldots, x_n$ are scalars drawn i.i.d from $D(0)$. There exist  maps $X_{\max}:\mathbb{Z}_+\times [0,1]\rightarrow \mathbb{R}$ , $c:\mathbb{Z}_+\times [0,1]\rightarrow \mathbb{R}$, and $C:\mathbb{Z}_+\times [0,1]\rightarrow \mathbb{R}$ such that for all $n\geq n_0$,  for every $\delta \in (0,1)$, with probability $\geq 1 - \delta$,   
\[ \max_{i\in \{1\cdots,n\}} x_i \in \big[X_{\max}(n, \delta) - c(n, \delta), X_{\max}(n, \delta) + C(n, \delta)\big] \]    
Also, $\lim_{n\rightarrow \infty }C(n,\delta) =0$, $\lim_{n\rightarrow \infty}c(n,\delta) =0$. 
\end{assumption}

The spread quantities $X_{\max}(n, \delta)$, $c(n, \delta)$ and $C(n, \delta)$ are distribution specific.

For example, for standard Normals, $X_{\max}(n, \delta) = \sqrt{2 \log n} - \frac{\log\log n + \log (4 \pi)}{\sqrt{2 \log n}}$, $c(n, \delta) = \frac{\log \log (6/\delta)}{\sqrt{2 \log n}}$ and $C(n, \delta) = \frac{\log (6/\delta)}{\sqrt{2 \log n}}$. For standard uniforms, $X_{\max}(n, \delta) = 1$, $C(n, \delta) = 0$ and $c(n, \delta) = \frac{\log(1/\delta)}{n}$. 

Using these tools, we now characterize two kinds of classification behavior below.
These correspond to the Gumbel and Weibull distribution families, as described in the Fisher-Tippett-Gnedenko theorem.

\subsection{Low dimensional Gumbel type} 

For the Gumbel type, the maximum $M_n = \max(x_1, \ldots, x_n)$ of $n$ iid examples converges to $a_n Z + b_n$, where $a_n$ and $b_n$ are distribution-specific quantities and $Z$ is Gumbel-distributed.
For these distributions, $X_{\max}(p, \delta)$ is typically higher than $X_{\max}(m, \delta)$ when $p \gg m$, producing a lack of symmetry in the maximum-margin classifier, and its increasing worst-class error. In this case, throwing away data to balance classes restores the symmetry and helps recover the lost worst-class error. This is formalized below.

\begin{restatable}{theorem}{ldgc}[General Gumbel Distributions]
\label{thm:gumbel1d}
Let $D$ be a distribution of the Gumbel type with cumulative density function $F$ and tail function $U(\cdot)$ and constants $a_n$ and $b_n$ in the Fisher-Tippett-Gnedenko theorem. Let $\lambda = \frac{\max(a_m, a_p) \log (3/\delta)}{U(p) - U(m)}$, and let $\mu = U(p) + a_p \log(3/\delta)$. Then, for large enough $m$ and $n$, with probability $\geq 1 - 2\delta$ over the training samples, we have: 
	\begin{eqnarray*}
		\theta_{\text{erm}} & \geq &  \frac{1}{2} (U(p) - U(m))(1 - \lambda), \\
		|\theta_{\text{ss}}| & \leq & \frac{1}{2} \lambda (U(p) - U(m)) 
	\end{eqnarray*}
In addition, the worst-class errors satisfy:
	\begin{eqnarray*}
		\wce(\theta_{\text{erm}}) & \geq &  1 - F\bigg(
  \frac{U(p) (1 + 3\lambda) + U(m) (1 - 3\lambda)}{2}\bigg)\\
		\wce(\theta_{\text{ss}}) & \leq &  1 - F(U(p)(1 - \lambda)).
	\end{eqnarray*}
\end{restatable}


%
A few remarks are in order.  First, ensuring that the training data is linearly separable requires $\mu$ to grow with $p$. 
Second, observe that usually we expect $\lambda \leq 1$, and it may even be $o(1)$ for certain growth rates of $m$ and $p$. When this is the case, the theorem implies that $|\theta_{\text{ss}}|$, which is of the order of $\lambda (U(p) - U(m))$ is an order of magnitude closer to the origin than $\theta_{\text{erm}}$, which is $\approx -\frac{1}{2}(U(p) - U(m))$. 
This, in turn, contributes to $\theta_{\text{ss}}$ smaller worst-class-error. If $\lambda = o(1)$, this worst-class-error would be $\approx 1 - F(U(p)) \approx \frac{1}{p}$. In contrast, the worst-class-error of $\theta_{\text{erm}}$ would approach $\approx 1 - F ( \frac{U(p) + U(m)}{2})$, which is somewhere between $\frac{1}{m}$ and $\frac{1}{p}$, depending on the exact form of $F$. 

Third, observe that both worst-class error are tighter than a standard PAC-style analysis that would give a bound of $\approx 1/m$ on the worst-class error. Next, we present a corollary to tighten the previous result for the important Gaussian case.

\begin{restatable}{theorem}{gausclass}
    \label{thm:gaussians}
	Let $0 < \epsilon, \delta, \gamma < 1$ be constants and suppose $m = \beta p$. There exists an $p_0$ such that the following holds. If $p \geq p_0$ and $\beta \geq 1/p^{3/4}$, then with probability greater or equal than $1 - 2 \epsilon - 2 \delta - 3 \gamma$:
\begin{align*}
  | \theta_{\text{erm}}| &\geq \frac{1}{2 \sqrt{2 \log (\beta p)}} \Big{(} \frac{2}{3} \log(1/\beta) - 2 \log(1/\gamma) \Big{)},\\
  | \theta_{\text{ss}} | &\leq \frac{\log (1/\gamma)}{2 \sqrt{2 \log (\beta p)}}.
\end{align*}
When $p \beta^2 \geq \epsilon$, this implies: 
\begin{equation*}
  \wce(\theta_{\text{ss}}) \leq \frac{2\epsilon}{\gamma p}, \quad \wce(\theta_{\text{erm}}) \geq \frac{\epsilon \gamma^{1/4}}{2 p \beta^{1/12} }.
\end{equation*}
\end{restatable}


For Gaussians, if $\beta \rightarrow 0$ with $\beta p \rightarrow \infty$, the relative gap between $\theta_{\text{ss}}$ and $\theta_{\text{erm}}$ widens -- $\theta_{\text{ss}}$ lies in an interval of length $\approx \frac{1}{2 \sqrt{2 \log (\beta p)}}$ around the origin, while $\theta_{\text{erm}}$ lies $\approx \frac{\log(1/\beta)}{2 \sqrt{2 \log (\beta p)}}$ away.  This leads to a widening of the worst-class error between the two SVM solutions, with $\theta_{\text{ss}}$ having considerably lower error than $\theta_{\text{erm}}$.

\subsection{Low dimensional Weibull type} 

Recall that for these distributions, the extremal point of the distribution is finite, and the maximum $M_n$ converges to $a_n Z + b_n$, where $a_n$ and $b_n$ are distribution-specific quantities and $Z$ is a reverse Weibull random variable with parameter $\alpha$. For these distributions, $X_{\max}(p, \delta) \approx X_{\max}(m, \delta)$ even when $p \gg m$, and hence the maximum-margin classifier remains symmetric even when the majority class size $p \gg m$. This means that ERM and subsampling perform equally well. The following theorem formalizes the result. 

\begin{restatable}{theorem}{weibullclass}[General Weibull Distributions]
\label{thm:weibull1d}
Suppose $D$ is a distribution of the Weibull-type with parameter $\alpha$ and extremal point $x_F$. Let $\mu = x_F$, and let $m, p \rightarrow \infty$. Then, for any $0 < \delta \leq 1/4$, with probability $\geq \frac{1}{4 \cdot 2^{2^\alpha}} - \delta$, 
	\begin{eqnarray*} 
		|\theta_{\text{ss}}| & \geq &  \frac{1}{2}(x_F - U(m))(\ln 2)^{1/\alpha}  \\
		|\theta_{\text{erm}}| & \leq & \frac{1}{2}(x_F - U(m))(\ln 2)^{1/\alpha} . 
	\end{eqnarray*}
\end{restatable}

Once again, we pause for some remarks. First, we observe that the extremal value of the Weibull distribution is finite, unlike Gumbel, ensuring that training data is linearly separable only requires $\mu = x_F$; the distributions themselves therefore do not change with $p$. Second, observe that our theorem states that with constant probability, the $|\theta_{\text{erm}}|$ is lower than $|\theta_{\text{ss}}|$ and so is the worst-class error. This implies that in the Weibull case, we cannot hope to get a high probability theorem such as Theorem~\ref{thm:gumbel1d}. A final remark is the dependence of the lower bound on $|\theta_{\text{ss}}|$ on the parameter $\alpha$ of the Weibull distribution; unfortunately, this dependence is inevitable, since the concentration properties of the difference between two Weibull random variables depend on $\alpha$. The following corollary makes the result concrete for uniform distributions. 

\begin{corollary}\label{cor:uniform}
Suppose $D(0)$ is the uniform distribution on $[-1/2, 1/2]$. Let $\mu = 1/2$, and let $m, p \rightarrow \infty$. Then, for any $0 < \delta \leq 1/4$, we have that for $m$ and $p$ large enough, with probability $\geq \frac{1}{16} - \delta$, 
\[ |\theta_{\text{ss}}| \geq  \frac{\ln 2}{2m} , |\theta_{\text{erm}}| \leq \frac{\ln 2}{2m}. \]
This implies that with probability $\geq \frac{1}{16} - \delta$,
\[ \wce(\theta_{\text{ss}}) \geq \frac{\ln 2}{m} \geq \wce(\theta_{\text{erm}}). \]
\end{corollary}

\subsection{High dimensional case}

We next look at a higher dimensional case where the feature vector $x \in \bbR^d$. Our basic setup is as follows. As in the previous section, we assume that the class conditional distribution for each class is spherically symmetric around its mean $\mu$; points $x$ from class $y$ follow $D(y \mu)$. The classifier that minimizes worst-class error for this setting is $\theta^*(x) = \mu^{\top} x$. Additionally, symmetry of the classes ensures that any linear classifier that passes through the origin will have equal error on each class. Thus showing $\theta_{\text{erm}}$ has high worst-class error involves showing that it has a non-zero bias term. 
This, in turn, will happen when the tails of the class-conditional distributions ``spread out'' with increasing sample size. This is formalized by the following high-dimensional concentration condition similar to the low dimensional cases. Notice that the difference here is that the concentration applies to all directions, and that the terms depend on the dimension $d$ in addition to $n$ and $\delta$. 

\begin{assumption}[Concentration Condition]
\label{assm: concentration_high_dim}
	Suppose $x_1, \ldots, x_n$ are $d$ dimensional random variables drawn i.i.d from $D(0)$. There exist  maps $X_{\max}:\mathbb{Z}_+ \times [0,1] \times \mathbb{Z}_+\rightarrow \mathbb{R}$, $c:\mathbb{Z}_+ \times [0,1] \times \mathbb{Z}_+ \rightarrow \mathbb{R}$, and $C:\mathbb{Z}_+\times [0,1] \times \mathbb{Z}_+ \rightarrow \mathbb{R}$ such that for all $n\geq n_0$,  for every $\delta \in (0,1)$, and for all directions $v\in \mathbb{R}^{d}$, with probability $\geq 1 - \delta$,   
	\begin{eqnarray*}
		\max_{i\in \{1\cdots,n\}} \{ v^{\top}x_i\} \in \big[X_{\max}(n, \delta, d) - c(n, d, \delta), \\
		X_{\max}(n, \delta, d) + C(n, \delta,d)\big].
	\end{eqnarray*}
 Also, $\lim_{n\rightarrow \infty \infty}C(n,\delta)=0$ and $\lim_{n\rightarrow \infty} c(n,\delta)=0$.  
\end{assumption}

Unlike low dimensions, our high dimensional analysis requires one more technical condition to bound the $s$-th order statistic from the distribution.
\begin{assumption}
\label{assm: concentration_qth_order_main}
    Let $\zeta: \mathbb{Z}_{+} \times \mathbb{R} \times \mathbb{Z}_{+} \times  \mathbb{Z}_{+} \rightarrow \mathbb{R}$ be a map such that for all $n\geq n_0$ with probability at least $1-\delta$ and 
$$ \hat{\mu}^{\top} x^{(s)} \geq \zeta(n,\delta, d, s) $$
where $\hat{\mu}$ is the unit vector along $\mu$, $\hat{\mu}^{\top} x^{(s)}$ is $s^{th}$ largest value among $n$ iid values of $\hat{\mu}^{\top} x$, where $x\sim D(0)$. 
\end{assumption}
Define $q=d \log(\log p \max_{x_i\in A}\|x_i\|) + \log(1/\delta)$. We are now ready to state the main result. 

\begin{restatable}{theorem}{imbclassd}
\label{thm: imb_class_highd}
$D(0)$ satisfies the concentration conditions in Assumption \ref{assm: concentration_high_dim} and \ref{assm: concentration_qth_order_main}. Suppose $\|\mu\| > X_{\max}(p,\delta,d)$ and $\zeta(p, \delta, d, q)  > 4X_{\max}(m,\delta,d)$.  If $p$ and $m$ are sufficiently large, then with probability at least $1-8\delta$, the worst class error rate achieved by ERM is worse than the worst class error achieved by subsampling the classes. 
\end{restatable}


We pause for a few remarks. We require $\|\mu\| > X_{\max}(p,\delta,d)$ and $\zeta(p, \delta, d, q)  > 4X_{\max}(m,\delta,d)$  to ensure that the direction of the classifier learned by ERM is sufficiently aligned with the direction ($\hat{\mu}$) of the classifier that achieves optimal worst class-error. As a result, we only need to compare the bias term between the subsampling and ERM.  Since $\zeta(p, \delta, d, q)  > 4X_{\max}(m,\delta,d)$ it ensures that $p$ is sufficiently larger than $m$, which causes the bias under ERM to be large. Also, $\|\mu\| > X_{\max}(p,\delta,d)$  ensures that the two classes are separable. We now illustrate the above theorem for Gaussians and uniform distribution.

\begin{restatable}{corollary}{imbclasscor}\label{cor imb_class_highd}
       Let $D(0)$ be a symmetric Gaussian in $\mathbb{R}^2$. The bias for the ERM classifier $\theta_{\text{erm}}$ lies in an arbitrarily small interval centered at $\sqrt{2\log(p/\delta)} - \sqrt{2\log(m/\delta)}$. In contrast, the bias for the classifier under subsampling $\theta_{\text{ss}}$ lies in an arbitrarily small interval centerd at zero. If $m 
 =\log p$ and $\|\mu\|>\sqrt{2\log(p/\delta)}$, then with probability $1-8\delta$, ERM has a worse worst class error than subsampling. 
Let $D(0)$ be a symmetric uniform in $\mathbb{R}^2$. The bias term for both the ERM classifier and the subsampling classifier lies in an arbitrarily small interval centered at zero. 
\end{restatable}

 

The above corollary shows how the bias for Gaussian is much larger than in the uniform distribution. As a result, subsampling is guaranteed to help the Gaussians but does not help uniform distributions as shown in Figure~\ref{fig:main_figure}.

\section{Analysis of imbalanced groups}
\label{sec:groups}

We now look at data from imbalanced groups. Recall  our basic setup, where label $y$ and the attribute $a$ induce four groups $g = (y, a)$ that a data point belongs to. To keep our analysis simple, we assume that each of the four groups have the same distribution $D(\cdot)$, but with shifted means, and that $D(\cdot)$ is spherically symmetric about its mean. Specifically, this means that for a group $g = (y, a)$, the group conditional distribution is $D( y \mu + a \psi)$, where $\mu$ and $\psi$ are vectors in $\bbR^d$ with a non-zero norm. 
Also, recall  that we have $p$ points from each of the majority groups $(1, 1)$ and $(-1, -1)$ and $m$ from the minority groups $(1, -1)$ and $(-1, 1)$ with $p \gg m$.

We start with a simple two-dimensional case where each feature vector $x_i \in \bbR^2$.
We set the parameter vector $\mu =\|\mu\| (1, 0)^{\top}$ and $\psi = \|\psi\|(0, 1)^{\top}$.	
As a result, the first coordinate of $x$ is aligned with the label $y$: $\mathbb{E}[x_1|y] = y\|\mu\|$. Following~\citet{nagarajan20}, we call this the {\em{invariant feature}}. The corresponding classifier $\theta^*_{\text{inv}} = (w^{*}_{\text{inv}} = (1, 0)^{\top}, b^{*}_{\text{inv}} = 0)$ is the ideal linear classifier that achieves the best worst-group accuracy (\Cref{thm:wge}). We call this the {\em{invariant classifier}}.
The second coordinate of $x$ is aligned with the attribute $a$ and not the label $y$:  $\mathbb{E}[x_2|a] = a\|\psi\|$.  The imbalance in sampling rates of the different groups can cause the SVM classifier to have a component along this coordinate, even though it is trained to target the label $y$. We call this coordinate the {\em{spurious feature}}. If the minority group were entirely absent, the ideal classifier would involve the spurious feature, and would have the weight vector $w^{*}_{\text{spu}} = \frac{\|\mu\|}{\sqrt{\|\mu\|^2+\|\psi\|^2}} (1, 0)^{\top} + \frac{\|\psi\|}{\sqrt{\|\mu\|^2+\|\psi\|^2}} (0, 1)^{\top}$ and bias $0$. We call this the {\em{spurious classifier}}. Next we show  that the max-margin classifier converges to either the invariant or spurious classifier depending on the tail properties 
of the group conditional distributions.

\subsection{Low dimensional Gumbel type}

Finally, we are ready to state our main results. First, we look at the case analogous to the Gumbel types in Section~\ref{sec:classes}, where the distribution tails ``spread out'' with more data. 

\begin{restatable}{theorem}{spugumbgroup}\label{thm: spurious_gumbel}
	Suppose $D(0)$ satisfies the concentration condition in Assumption \ref{assm: concentration_high_dim} and   $X_{\max}(p, \delta, 2) -X_{\max}(m, \delta, 2) \geq  2\|\psi\| + c(p, \delta, 2) +   C(m, \delta, 2)$. If $p\rightarrow \infty$, then with probability at least $ 1 - 4 \delta$, the ERM solution converges to the spurious solution $w_{\text{spu}}^{*}$.
	In addition, with probability at least $ 1 - 12 \delta$, $ \wge(\theta_{\text{ss}}) <  \wge(\theta_{\text{erm}})$.  
\end{restatable}


First, observe that bounding $X_{\max}(p, \delta, 2) -X_{\max}(m, \delta, 2)$ from below essentially means that the maximum of $p$ samples is considerably larger than the maximum of $m$ samples when $p \gg m$; this is the {\em{spreading out}} condition in Figure~\ref{fig:main_figure}. The first part of the theorem thus says that ERM has poor worst-group error in this case, while the second part shows that throwing away data through subsampling helps. Second, observe that unlike Section~\ref{sec:classes}, where the bias term in the max-margin solution changes with the tail properties, for imbalanced groups, it is the direction of the max-margin classifier that changes (as illustrated in Figure~\ref{fig:main_figure}). 

We now make this result concrete when the group-conditional distributions are symmetric Gaussians centered at zero in $\mathbb{R}^2$. Let $\|\mu\| = \sqrt{3\log \frac{p}{\delta}}$ and 
$\|\psi\| = \sqrt{\kappa/4  \log \frac{p}{\delta}}$. The ratio of the weight associated with the spurious feature to the invariant feature in $w_{\text{spu}}^{*}$ is $\sqrt{\frac{\kappa}{12}}$. Let $m=p^{\tau}$, where $\tau<1$. In this case, since $\|\mu\|>X_{\max}(p, \delta, 2)$, the two classes are linearly separable. If $\kappa<2\big(1+\tau -2\sqrt{\tau}\big)$, then the condition in the above theorem is satisfied and thus we can conclude that for this family of Gaussians the max-margin solution converges to the spurious solution. Let us contrast this with uniform $D(0)$, which has no tails. Since $X_{\max}(p, \delta, 2) =X_{\max}(m, \delta, 2)=1$ the condition in the above theorem is not satisfied for uniform distribution.

\subsection{Low dimensional Weibull type}

We next look at our analogue of the Weibull case in Section~\ref{sec:classes}, where the tails of the group-conditional distributions grow slowly with more data. In this case, we show that the max-margin solution converges to the invariant classifier that achieves the optimal worst-group error.  

\begin{restatable}{theorem}{weibulgroup}
    \label{thm: invariant_weibull}
Suppose $D(0)$ satisfies the concentration condition in Assumption \ref{assm: concentration_high_dim} and as $p$ and $m$ approach $\infty$
	\begin{eqnarray*}
		\frac{X_{\max}(p, \delta, 2) - X_{\max}(m, \delta, 2)}{2\|\psi\|} \rightarrow 0.\end{eqnarray*}
  If $m,p \rightarrow \infty$,	then with probability at least $1 - 4 \delta$, the ERM solution converges to the invariant solution $w_{\text{inv}}^{*}$.
\end{restatable}


Some remarks are in order. Observe that this theorem requires that the difference between the tails $X_{\max}(p, \delta, 2) -X_{\max}(m, \delta, 2)$ shrinks as $p$ and $m$ go to infinity. 
A concrete example where Theorem~\ref{thm: invariant_weibull} applies is symmetric uniform $D(0)$, where $X_{\max}(p, \delta, 2) =X_{\max}(m, \delta, 2)=1$.

\subsection{Higher dimensional case}

Moving on to higher dimensions, we again consider a setup where the feature vectors $x \in \bbR^d$. We assume that the group-conditional distributions $D$ have the same form but with shifted means, and are spherically symmetric around their means. We select the first unit vector $e_1 = (1, 0, \ldots)$ as the invariant feature, and the second one $e_2 = (0, 1, 0, \ldots)$ as the spurious feature. Thus the group-conditional distribution for points in group $g = (y, a)$ is $D(y \|\mu\| e_1 + a \|\psi\| e_2)$. As earlier in the Section~\ref{sec:groups}, we can similarly define the invariant classifier  and the spurious classifiers.  The main results here, which we state below,  mirror the theorems for the low dimensional cases.

\begin{restatable}{theorem}{gumbelgroupd}[Gumbel type]\label{thm: spurious_d}
    	Suppose $D(0)$ satisfies the concentration condition in Assumption \ref{assm: concentration_high_dim} and    $X_{\max}(p, \delta, d) -X_{\max}(m, \delta, d) \geq $$ $$ 2\|\psi\| + c(p, \delta, d) +   C(m, \delta, d). $ If $p\rightarrow \infty$, then with probability at least $ 1 - 4 \delta$, the ERM solution converges to the spurious solution $w_{\text{spu}}^{*}$.	In addition, with probability at least $ 1 - 12 \delta$, $ \wge(\theta_{\text{ss}}) <  \wge(\theta_{\text{erm}})$.  
\end{restatable}


\begin{restatable}{theorem}{weibullgroupd}[Weibull type]\label{thm: invariant_d}
    Suppose $D(0)$ satisfies the concentration condition in Assumption \ref{assm: concentration_high_dim} and as $p$ and $m \rightarrow \infty$ $$\frac{X_{\max}(p, \delta, d) - X_{\max}(m, \delta, d) }{2\|\psi\|} \rightarrow 0,$$ 
 $$\frac{X_{\max}(p, \delta, d) - X_{\max}(m, \delta, d) }{2\|\mu\|} \rightarrow 0,$$ 
	and $\frac{C(m,\delta,d) + c(m,\delta,d) +  \frac{1}{2}\sqrt{2 C(m, \delta, d) + c(m, \delta, d)\|\psi\|}}{\|\mu\|} \rightarrow 0$.  If $m,p \rightarrow \infty$,	then with probability at least $1 - 4 \delta$, the ERM solution converges to the invariant solution $w_{\text{inv}}^{*}$.
\end{restatable}


\section{Empirical Implications}
\label{sec: expmts}

We next investigate the empirical implications of the proposed theory. Specifically, we ask:

	$\bullet$ Our theory is most applicable in low to moderate dimensions. Does throwing away data improve the worst group error in real data when applied to the top few features?
 
	$\bullet$ What do the tails of the top feature distributions look like?

\begin{table}

	\caption{Performance of ERM on imbalanced and balanced datasets on high and low-dimensional versions of Waterbirds and CelebA. Worst group accuracy of ERM trained on balanced data substantially improves.}
	\label{table1}
	\centering
 \renewcommand{\arraystretch}{1}
	\begin{tabular}{llll}
		\toprule
		 Dataset & Method   & Avg. Accuracy    & WG accuracy              \\  \hline
		Waterbirds  & ERM  & $0.89 \pm 0.00$ &  $0.63 \pm 0.01$ \\
      Waterbirds & SS  & $0.93 \pm 0.00$ & $0.88 \pm 0.01$     \\  
    Waterbirds & ERM-PCA &  $0.88 \pm 0.01$ & $0.66 \pm 0.03$ \\ 
      Waterbirds & SS-PCA  & $0.94 \pm 0.01$ & $0.89 \pm 0.01$     \\   \hline
		CelebA     & ERM & $0.95 \pm 0.00$ & $0.36 \pm 0.05$ \\   
  		CelebA     & SS & $0.91 \pm 0.00$ & $0.83 \pm 0.01$\\ 
  CelebA  & ERM-PCA & $0.95 \pm 0.00$ & $0.40 \pm 0.01$  \\ 
 CelebA     & SS-PCA & $0.90 \pm 0.00$ & $0.83 \pm 0.01$\\  
		\bottomrule
	\end{tabular}
\end{table}

\paragraph{Datasets \& Baselines.} These questions are considered in the context of Waterbirds \citep{sagawa2019distributionally} and CelebA \citep{liu2015deep}, the two most commonly used datasets for studying group imbalance. The Waterbirds data consists of two target classes -- Waterbirds and Landbirds, and two background types -- Water and Land. Most waterbirds appear on water, and most landbirds appear on land. In CelebA data, the target is to predict hair type -- Blond or Non-Blond, where the frequency of blond women is much higher than blond men.  Our ERM baseline consists of a ImageNet-pretrained ResNet-50 model  that is finetuned on Waterbirds (4795 data points) and CelebA datasets (162770 data points) respectively.  To understand the impact of dimensionality reduction, we compare this with the ERM-PCA baseline, that takes the PCA of the last layer ($2048$ dimension) of ResNet-50 model and trains a linear classifier on the first four principal components that explain $\approx 99 \%$ of the variance in the data. \citet{kirichenko2022last} showed that if we freeze the fine-tuned representations and retrain just the last layer on balanced data obtained by subsampling that suffices to improve the worst group error. We call this method SS.  We compare it with SS-PCA that trains a linear layer on balanced four-dimensional data (top four PCA components). 
 
\paragraph{Results.} Table \ref{table1} shows the results. We see that as expected for the high dimensional data, ERM performs much worse in terms of worst-group accuracy (WG accuracy) than SS. This confirms the findings of prior work~\cite{kirichenko2022last, IdrissiSimple2022}. We also see that the same pattern holds for ERM-PCA and SS-PCA. This confirms that throwing away data improve the worst group error in real data when applied to the top few features. In Figure~\ref{fig:tails_wb}, we visualize the tails of the top PCA feature for Waterbirds. Specifically, we plot a histogram of data from each group projected along the feature with the highest PCA value. The results on other features and CelebA are plotted in the Appendix. The results show that the groups are indeed long-tailed, in the sense that they do not look like the uniform distribution. This suggests that the theoretical phenomenon that we describe in this paper might contribute to the success of subsampling.

\begin{figure}
\centering
\includegraphics[width=3.5in]{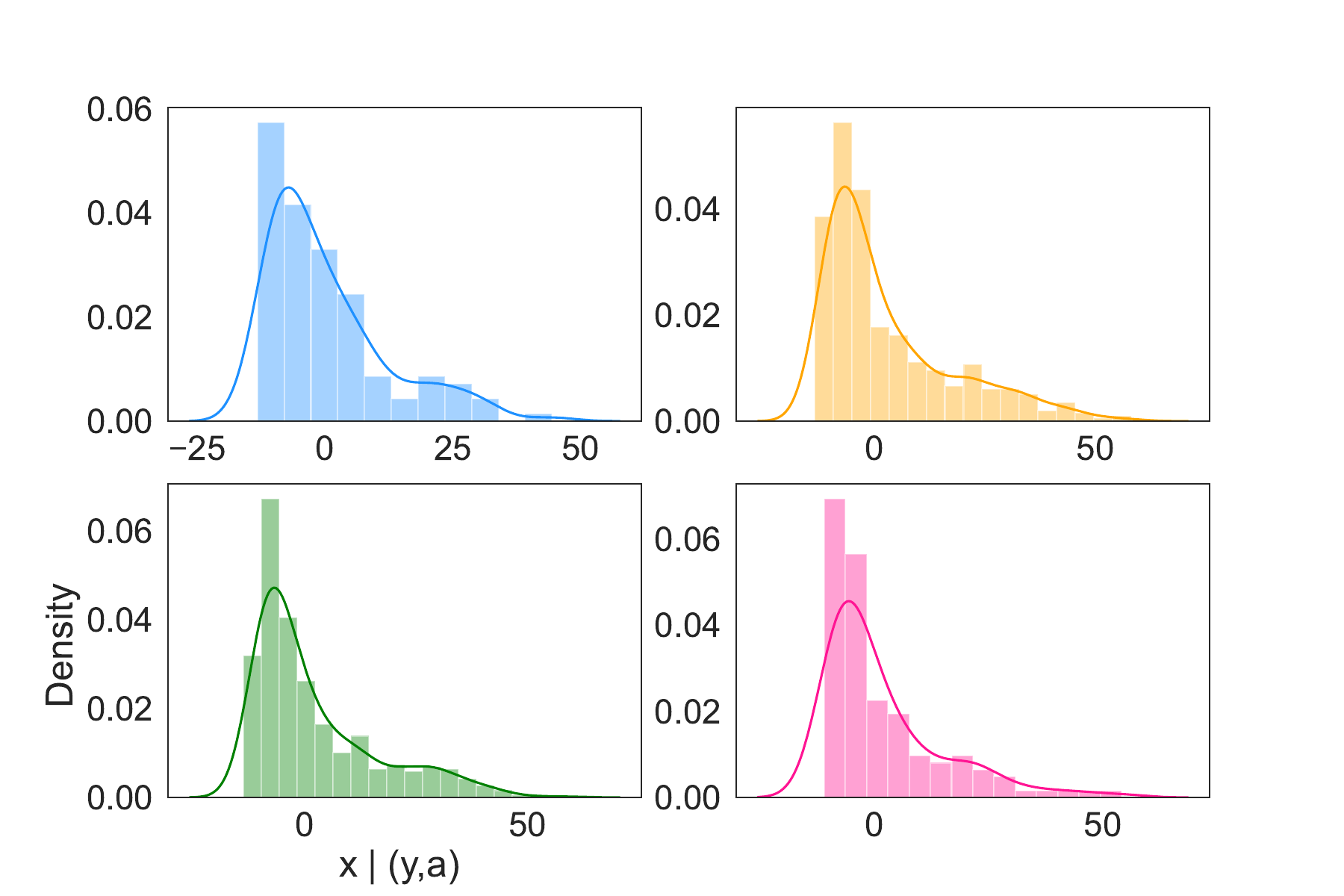}
          \caption{Waterbirds: Distribution of the top features.}
          \label{fig:tails_wb}
\end{figure}

\section{Discussion}

The vast majority of learning theory literature including works carried out in the context of multi group analysis have focussed on measuring concentrations of bounded functions \citep{rothblum2021multi,tosh2022simple,haghtalab2022demand}, and tail properties of distributions feature rarely. A handful of papers have looked at designing algorithms with performance guarantees for tasks carried out on heavy-tailed distributions. For example,~\cite{hsu2016loss} proposes a regression algorithm based on the median-of-means estimator that concentrates well under heavy-tailed distributions. In contrast, we analyze algorithms under distributions with different kinds of tail properties.
Other works on theory of imbalanced classification have focused on metrics other than accuracy~\cite{menon2013statistical, natarajan2017cost}, such as precision and recall~\cite{diochnos2021learning}. An example is \cite{narasimhan2015consistent}, which proposes a new Bayes Optimal classifier for different functions of the confusion matrix when there are imbalanced classes, together with consistency guarantees. However, the bounds in these works are coarser than ours, as they do not reflect changing behavior depending on the tail distributions of the classes.

Our work is also related to \citet{nagarajan2020understanding}, which analyzes max-margin classifiers to explain the failure of of ERM under group imbalance. Our work complements their findings. While the authors derive a lower bound on the weight associated with the spurious feature, they do not specify conditions on the distribution under which this bound is positive, which is crucial towards explaining when the model relies on spurious features. We fill this gap as we provide a characterization of the max-margin classifiers in terms of the tails of the distribution. To the best of our knowledge, this is a first characterization of the SVM classifiers as a function of the tails of the distribution. Therefore, we believe the proof techniques developed here can be of independent interest.   Looking forward, we believe that while data balancing is powerful it still requires access to the knowledge of spurious attributes. Therefore, it is important to formalize what is achievable in the absence of such knowledge.

\clearpage
\bibliographystyle{apalike}
\bibliography{main}

\begin{thebibliography}{}

\bibitem[Bennett and Bredensteiner, 2000]{BB00}
Bennett, K.~P. and Bredensteiner, E.~J. (2000).
\newblock Duality and geometry in svm classifiers.
\newblock In {\em ICML}, volume 2000, pages 57--64. Citeseer.

\bibitem[Buolamwini and Gebru, 2018]{buolamwini2018gender}
Buolamwini, J. and Gebru, T. (2018).
\newblock Gender shades: Intersectional accuracy disparities in commercial
  gender classification.
\newblock In {\em Conference on fairness, accountability and transparency},
  pages 77--91. PMLR.

\bibitem[De~Haan and Ferreira, 2007]{de2007extreme}
De~Haan, L. and Ferreira, A. (2007).
\newblock {\em Extreme value theory: an introduction}.
\newblock Springer Science \& Business Media.

\bibitem[De~Haan et~al., 2006]{extremalbook}
De~Haan, L., Ferreira, A., and Ferreira, A. (2006).
\newblock {\em Extreme value theory: an introduction}, volume~21.
\newblock Springer.

\bibitem[Diochnos and Trafalis, 2021]{diochnos2021learning}
Diochnos, D.~I. and Trafalis, T.~B. (2021).
\newblock Learning reliable rules under class imbalance.
\newblock In {\em Proceedings of the 2021 SIAM International Conference on Data
  Mining (SDM)}, pages 28--36. SIAM.

\bibitem[Feller, 1968]{feller68}
Feller, W. (1968).
\newblock {\em An Introduction to Probability Theory and Its Applications}.
\newblock Wiley.

\bibitem[Haghtalab et~al., 2022]{haghtalab2022demand}
Haghtalab, N., Jordan, M.~I., and Zhao, E. (2022).
\newblock On-demand sampling: Learning optimally from multiple distributions.
\newblock {\em arXiv preprint arXiv:2210.12529}.

\bibitem[Haussler, 1990]{haussler1990probably}
Haussler, D. (1990).
\newblock {\em Probably approximately correct learning}.
\newblock University of California.

\bibitem[Hsu and Sabato, 2016]{hsu2016loss}
Hsu, D. and Sabato, S. (2016).
\newblock Loss minimization and parameter estimation with heavy tails.
\newblock {\em The Journal of Machine Learning Research}, 17(1):543--582.

\bibitem[Idrissi et~al., 2021]{idrissi2021simple}
Idrissi, B.~Y., Arjovsky, M., Pezeshki, M., and Lopez-Paz, D. (2021).
\newblock Simple data balancing achieves competitive worst-group-accuracy.
\newblock {\em CLeaR}.

\bibitem[Idrissi et~al., 2022]{IdrissiSimple2022}
Idrissi, B.~Y., Arjovsky, M., Pezeshki, M., and Lopez-Paz, D. (2022).
\newblock Simple data balancing achieves competitive worst-group-accuracy.
\newblock {\em CLeaR}.

\bibitem[Johnson and Khoshgoftaar, 2019]{johnson2019survey}
Johnson, J.~M. and Khoshgoftaar, T.~M. (2019).
\newblock Survey on deep learning with class imbalance.
\newblock {\em Journal of Big Data}, 6(1):1--54.

\bibitem[Kirichenko et~al., 2022]{kirichenko2022last}
Kirichenko, P., Izmailov, P., and Wilson, A.~G. (2022).
\newblock Last layer re-training is sufficient for robustness to spurious
  correlations.
\newblock {\em arXiv preprint arXiv:2204.02937}.

\bibitem[Laurent and Massart, 2000]{laurent2000adaptive}
Laurent, B. and Massart, P. (2000).
\newblock Adaptive estimation of a quadratic functional by model selection.
\newblock {\em Annals of Statistics}, pages 1302--1338.

\bibitem[Liu et~al., 2015]{liu2015deep}
Liu, Z., Luo, P., Wang, X., and Tang, X. (2015).
\newblock Deep learning face attributes in the wild.
\newblock In {\em Proceedings of the IEEE international conference on computer
  vision}, pages 3730--3738.

\bibitem[Menon et~al., 2013]{menon2013statistical}
Menon, A., Narasimhan, H., Agarwal, S., and Chawla, S. (2013).
\newblock On the statistical consistency of algorithms for binary
  classification under class imbalance.
\newblock In {\em International Conference on Machine Learning}, pages
  603--611. PMLR.

\bibitem[Nagarajan et~al., 2020a]{nagarajan20}
Nagarajan, V., Andreassen, A., and Neyshabur, B. (2020a).
\newblock Understanding the failure modes of out-of-distribution
  generalization.
\newblock {\em arXiv}.

\bibitem[Nagarajan et~al., 2020b]{nagarajan2020understanding}
Nagarajan, V., Andreassen, A., and Neyshabur, B. (2020b).
\newblock Understanding the failure modes of out-of-distribution
  generalization.
\newblock {\em arXiv preprint arXiv:2010.15775}.

\bibitem[Narasimhan et~al., 2015]{narasimhan2015consistent}
Narasimhan, H., Ramaswamy, H., Saha, A., and Agarwal, S. (2015).
\newblock Consistent multiclass algorithms for complex performance measures.
\newblock In {\em International Conference on Machine Learning}, pages
  2398--2407. PMLR.

\bibitem[Natarajan et~al., 2017]{natarajan2017cost}
Natarajan, N., Dhillon, I.~S., Ravikumar, P., and Tewari, A. (2017).
\newblock Cost-sensitive learning with noisy labels.
\newblock {\em J. Mach. Learn. Res.}, 18(1):5666--5698.

\bibitem[Rojas et~al., 2022]{rojas2022the}
Rojas, W. A.~G., Diamos, S., Kini, K.~R., Kanter, D., Reddi, V.~J., and
  Coleman, C. (2022).
\newblock The dollar street dataset: Images representing the geographic and
  socioeconomic diversity of the world.
\newblock In {\em Thirty-sixth Conference on Neural Information Processing
  Systems Datasets and Benchmarks Track}.

\bibitem[Rothblum and Yona, 2021]{rothblum2021multi}
Rothblum, G.~N. and Yona, G. (2021).
\newblock Multi-group agnostic pac learnability.
\newblock In {\em International Conference on Machine Learning}, pages
  9107--9115. PMLR.

\bibitem[Sagawa et~al., 2019]{sagawa2019distributionally}
Sagawa, S., Koh, P.~W., Hashimoto, T.~B., and Liang, P. (2019).
\newblock Distributionally robust neural networks for group shifts: On the
  importance of regularization for worst-case generalization.
\newblock {\em arXiv preprint arXiv:1911.08731}.

\bibitem[Soudry et~al., 2018]{soudry}
Soudry, D., Hoffer, E., Nacson, M.~S., Gunasekar, S., and Srebro, N. (2018).
\newblock The implicit bias of gradient descent on separable data.
\newblock {\em The Journal of Machine Learning Research}, 19(1):2822--2878.

\bibitem[Steinwart and Christmann, 2008]{steinwart2008support}
Steinwart, I. and Christmann, A. (2008).
\newblock {\em Support vector machines}.
\newblock Springer Science \& Business Media.

\bibitem[Tosh and Hsu, 2022]{tosh2022simple}
Tosh, C.~J. and Hsu, D. (2022).
\newblock Simple and near-optimal algorithms for hidden stratification and
  multi-group learning.
\newblock In {\em International Conference on Machine Learning}, pages
  21633--21657. PMLR.

\bibitem[Vapnik, 1999]{vapnik1999nature}
Vapnik, V. (1999).
\newblock {\em The nature of statistical learning theory}.
\newblock Springer science \& business media.

\end{thebibliography}

\clearpage
\appendix
\onecolumn
\hrule
\begin{center}
  \textbf{\Large Why Throwing Away Data Improves Worst-Group Error?}
  \vskip 0.25cm
  {\large Appendices}
\end{center}
\hrule

\vskip 0.25 cm
We organize this section as follows. 
\begin{itemize}
\item In \Cref{sec:imb_class_proofs}, we derive the results for imbalanced classification.
\begin{itemize}
\item In \Cref{sec:oned_app}, we derive results for imbalanced classification in the one-dimensional setting. 
\item In \Cref{sec:imb_class_high_d_append}, we derive results for imbalanced classification in higher dimensions. 
\end{itemize}
\item In \Cref{sec:proof_groups}, we derive the results for classification with imbalanced groups. 
\begin{itemize}
\item In \Cref{sec:proof_groups_2d}, we derive the results for imbalanced groups in two-dimensional setting.
\item In \Cref{sec:proof_groups_high_d}, we derive the results for imbalanced groups in higher-dimensional setting.
\end{itemize}
\item In \Cref{sec:supp_exp}, we present the supplementary materials for the empirical findings. 
\end{itemize} 

\section{Proofs for imbalanced classes}
\label{sec:imb_class_proofs}

\subsection{One-dimensional case}
\label{sec:oned_app}

\ldgc*


\begin{proof}
Let $M_m$ denote the maximum of $m$ random variables drawn from $D(0)$. We observe that 
\[ \theta_{\text{ss}} = \frac{1}{2} (M_m - M'_m) ,\]
From the Gumbel concentration lemma (Lemma \ref{lem:1dgaussianmax}), with probability $\geq 1 - \delta$, $b_m + a_m \log \log (3/\delta) \leq M_m \leq b_m + a_m \log (3/\delta)$. Therefore,
\[ |\theta_{\text{ss}}| \leq \frac{a_m}{2} \log (3/\delta) \leq \lambda (U(p) - U(m)) ,\]
where the second step follows from the definition of $\lambda$.

Similarly, we can show that:
\[ \theta_{\text{erm}} = \frac{1}{2} (M_p - M_m), \]
where $M_p$ and $M_m$ are the maxima of $p$ and $m$ independent draws from $D(0)$. From the Fisher-Tippett-Gnedenko theorem, $M_n = b_n + a_n Z$ where $Z$ is an unit Gumbel random variable, and $b_n = U(n)$. Therefore,
\[ \theta_{\text{erm}} = \frac{1}{2} (U(p) - U(m) + a_p Z - a_m Z'),\]
where $Z$ and $Z'$ are independent Gumbel variables. This is at least:
\[ \theta_{\text{erm}} \geq \frac{1}{2} (U(p) - U(m) - \max\{a_m, a_p\} \log (3/\delta))  \geq  \frac{1}{2} (U(p) - U(m)) (1 - \lambda) ,\]
where the last step follows from the definition of $\lambda$. The first part of the lemma follows. 

For the second part of the lemma, let $\Phi(x) = 1 - F(x)$. Observe that
\begin{eqnarray*}
\wce(\theta_{\text{ss}}) & \leq &  \Phi(\mu - |\theta_{\text{ss}}|) \\
& \leq &  \Phi(U(p) + a_p \log (3/\delta) - \frac{a_m}{2} \log (3/\delta)) \\
& \leq & \Phi(U(p) ( 1 - \lambda))
\end{eqnarray*}
where the first step follows by definition of $\wce$, and the second step follows by plugging in the values of $\mu_n$ and an upper bound on $|\theta_{\text{ss}}|$ and using the fact that $\Phi(x)$ is a decreasing function of $x$.  The third step follows because $\frac{a_m}{2} \log (3/\delta) - a_p \log (3/\delta) \leq \lambda U(p)$, and because $\Phi(x)$ is a decreasing function of $x$. 

In contrast,
\begin{eqnarray*}
    \wce(\theta_{\text{erm}}) & \geq & \max(\Phi( \mu - \theta_{\text{erm}}), \Phi(\mu + \theta_{\text{erm}})) \\
    & \geq & \Phi(U(p) + a_n \log (3/\delta) - \frac{1}{2} (U(p) - U(m))(1 - \lambda))  \\
    & \geq & \Phi \left( \frac{U(p) ( 1 + 3\lambda) + U(m) ( 1 - 3\lambda)}{2} \right),
\end{eqnarray*}
where the first step follows from the fact that $\Phi$ is a decreasing function, the second step by plugging in the value of $\mu$, and the third step from the observation that $\mu_n \leq U(p) (1 + \lambda)$. 
The second part of the theorem follows.
\end{proof}

\begin{lemma} \label{lem:gumbeldiff}
	Let $Z_3$ and $Z_4$ be two independent standard Gumbel variables. Then with probability $\geq 1 - \frac{2}{1 + e^{\tau}}$, $|Z_3 - Z_4| \leq \tau$. 
\end{lemma}

\begin{proof}
Since $Z_3$ and $Z_4$ are independent standard Gumbel variables, $Z_3 - Z_4$ is distributed according to a logistic distribution with location parameter $0$ and scale parameter $1$. This means that $Z_3 - Z_4$ is symmetric about $0$, and also that for any $\tau$,
	\begin{equation}\label{eqn:log}
		\Pr(Z_3 - Z_4 \in [-\tau, \tau]) = 1 - \frac{2}{1 + e^{\tau}} 
	\end{equation}
The lemma follows.
\end{proof}




\begin{lemma}\label{lem:gausstailtight}~\cite{feller68}
	Let $X \sim N(0, 1)$. Then, for any $t$,
	\begin{eqnarray*}
		\frac{1}{\sqrt{2 \pi}} e^{-t^2/2} \Big{(} \frac{1}{t} - \frac{1}{t^3}\Big{)} \leq \Pr(X \geq t) \leq \frac{1}{\sqrt{2 \pi}} \frac{e^{-t^2/2}}{t} 
	\end{eqnarray*}
\end{lemma}

\gausclass*

\begin{proof} 

	Fix $\epsilon$, $\gamma$ and $\delta$. We first show that with the given value of $\mu_n$, the probability that the training samples are realizable is at least $1 - 2 \epsilon$. To see this, observe that from Lemma~\ref{lem:gausstailtight}, the probability a sample $z$ from the positive class has value $<0$ is at most:
	\[ \Pr(X \geq \sqrt{2 \log (n/\epsilon)} | X \sim N(0, 1)) \leq \frac{e^{-\mu_n^2/2}}{\sqrt{2 \pi}\mu_n}  \leq \epsilon/n \]
	An union bound over all $n$ samples establishes that the probability that any of the $n$ positive samples lie below zero is at most $\epsilon$. A similar argument can also be applies to the negative class to show that with probability $\geq 1 - 2 \epsilon$, the training data is linearly separable.

	To establish the first part of the theorem, we will separately look at $\theta_{\text{ss}}$ and $\theta_{\text{erm}}$. For $p$ large enough, we can write
	\begin{eqnarray*}
		\Pr\left(|\theta_{\text{ss}}| \geq \frac{1}{2} \cdot \frac{\log (1/\gamma)}{\sqrt{2 \log (\beta p)}}\right) \leq \Pr\left(|Z_3 - Z_4| \geq  \log(1/\gamma) \right) + \delta 
	\end{eqnarray*}
	From Lemma~\ref{lem:gumbeldiff}, this probability is at most $\delta + \frac{2}{1 + (1/\gamma)} \leq \delta + 2 \gamma$.

	Similarly, for $p$ large enough, we have that:
	\begin{eqnarray*}
		\theta_{\text{erm}} \rightarrow \frac{1}{2} (b_{\beta p} - b_p) + \frac{1}{2}a_{\beta p} (Z_1 - Z_2)  + \frac{1}{2}(a_p - a_{\beta p}) Z_2
	\end{eqnarray*}
	where $Z_1$ and $Z_2$ are standard Gumbel random variables. This means that for $p$ large enough, and for any threshold $\tau$, we have that:
	\begin{eqnarray} \label{eqn:gaussplusdelta}
		\Pr(\theta_{\text{erm}} \leq \tau) & \leq & \Pr(\frac{1}{2} (b_{\beta p} - b_p) + \frac{1}{2}a_{\beta p} (Z_1 - Z_2) + \frac{1}{2}(a_p - a_{\beta p}) Z_2 \leq \tau) + \delta
	\end{eqnarray}
	Now, observe that for Gaussians:
	\begin{eqnarray}\label{eqn:gaussd1}
		b_{\beta p} - b_p & = & \sqrt{2 \log \beta p} - \frac{ \log\log(\beta p) + \log 4 \pi}{\sqrt{2 \log \beta p}} \nonumber - \sqrt{2 \log p} + \frac{ \log \log p + \log 4 \pi}{\sqrt{2 \log p}}  \nonumber \\
		& \leq & \sqrt{2 \log \beta p} - \sqrt{2 \log p}  \nonumber\\
		& = & \sqrt{2 \log \beta p} (1 - \sqrt{ \frac{ \log p}{\log \beta p}}) \nonumber \\
		& = & \sqrt{2 \log \beta p} (1 - (1 + \frac{\log (1/\beta)}{\log \beta p})^{1/2}) \nonumber \\
		& \leq & \sqrt{2 \log \beta p}(- \frac{\log(1/\beta)}{3 \log \beta p}) \nonumber \\
		& \leq & - \frac{2 \log (1/\beta)}{3 \sqrt{ 2 \log \beta p}}
	\end{eqnarray}
	Here the first step follows as $\frac{ \log \log p + \log (4\pi)}{\sqrt{\log p}}$ is a decreasing function of $p$, and the second and third steps follow from algebra. The fourth step follows from the fact that $(1 + x)^{1/2} \geq 1 + x/3$ when $x \leq 3$; as $p^3 \geq \beta^4$, $\frac{\log(1/\beta)}{\log \beta p} \leq 3$. The final step then follows from algebra.
	Additionally, from Lemma~\ref{lem:gumbeldiff}, 
	\begin{eqnarray} \label{eqn:gaussd2}
		\Pr(a_{\beta p} (Z_1 - Z_2) \geq \frac{\log(1/\gamma)}{\sqrt{2\log \beta p}}) \leq \Pr(Z_1 - Z_2 \geq \log(1/\gamma)) \nonumber \leq \frac{1}{1 + 1/\gamma} \leq \gamma
	\end{eqnarray}
	Finally, observe that
	\begin{eqnarray}\label{eqn:gaussdeltaa}
		a_{\beta p} - a_n & = & \frac{1}{\sqrt{2 \log \beta p}} - \frac{1}{\sqrt{2 \log p}} \nonumber \\
		& = & \frac{1}{\sqrt{2 \log \beta p}}(1 - \sqrt{\frac{\log \beta p}{\log p}}) \nonumber\\
		& = & \frac{1}{\sqrt{2 \log \beta p}} (1 - (1 - \frac{\log(1/\beta)}{\log p})^{1/2})\nonumber \\
		& \leq & \frac{1}{\sqrt{2 \log \beta p}} \cdot \frac{\log(1/\beta)}{\log p}
	\end{eqnarray}
	where the last step follows because for $0< x < 1$, $\sqrt{1 - x} \geq 1 - x$. Therefore, 
	\begin{eqnarray} \label{eqn:gaussd3}
		\Pr((a_{\beta p}  - a_p) Z_2 \geq \frac{\log (1/\gamma)}{\sqrt{2\log \beta p}}) & \leq & \Pr(Z_2 \geq \log (1/\gamma) \cdot \frac{\log p}{\log (1/\beta)}) \nonumber \\
		& \leq & \Pr(Z_2 \geq \log (1/\gamma)) \nonumber \\
		& \leq & 1 - e^{-\gamma} \leq \gamma
	\end{eqnarray}
	where the first step follows from Equation~\ref{eqn:gaussdeltaa}, the second step from the fact that $\log (1/\beta) \leq \log n$, and the final step from the standard Gumbel cdf and the fact that $1 - e^{-x} \leq x$. The first part of the theorem follows from combining Equations~\ref{eqn:gaussd1},~\ref{eqn:gaussd2} and~\ref{eqn:gaussd3}.

	To prove the third part of the theorem, we use Lemma~\ref{lem:gausstailtight}. From the second part of the theorem and Lemma~\ref{lem:gausstailtight}, observe that 

	\[ \err(\theta_{\text{ss}}) \leq \Phi\Big{(}\mu - \frac{ \log(1/\gamma)}{2 \sqrt{2 \log \beta p}}\Big{)} \]
	From Lemma~\ref{lem:gausstailtight}, the right hand side is, in turn, at most:
	\begin{eqnarray*}
		& \leq & \frac{1}{\sqrt{2 \pi} (\mu -\frac{ \log(1/\gamma)}{2 \sqrt{2 \log \beta p}}) } \cdot \exp\Big{(}-\frac{1}{2}\left(\mu - \frac{ \log(1/\gamma)}{2 \sqrt{2 \log \beta p}}\right)^2\Big{)} \\
		& \leq & \frac{2}{\sqrt{2 \pi} \mu_n} \exp\Big{(}-\frac{1}{2} \mu^2 + \frac{ \mu \log(1/\gamma)}{2 \sqrt{2 \log \beta n}}\Big{)} \\
		& \leq & \frac{2 \epsilon}{n} \cdot \exp\Big{(} \frac{ \mu \log(1/\gamma)}{2 \sqrt{2 \log \beta p}}\Big{)}\\
	\end{eqnarray*}
	Here the first step follows because for $p$ large enough, $\mu - \frac{ \log(1/\gamma)}{2 \sqrt{2 \log \beta p}} \geq \frac{1}{2} \mu$; this is because $\mu$ is an increasing function of $p$. The second step follows because by design $\frac{1}{\sqrt{2 \pi} \mu} e^{-\mu^2/2} = \epsilon/p$ and the third step follows from simple algebra. 

We observe that $\mu \leq \sqrt{2 \log (p/\epsilon)}$ -- this is because $e^{-\sqrt{2 \log (p/\epsilon)}^2/2} / \sqrt{4 \pi \log (p/\epsilon)} < \epsilon/p$. This implies that:
	\begin{eqnarray*}
		\frac{ \mu_n}{\sqrt{2\log \beta p}} & \leq & \Big{(}1 + \frac{\log(1/\epsilon \beta)}{\log \beta p}\Big{)}^{1/2} \leq 2, 
	\end{eqnarray*}
	provided $\frac{1}{\epsilon} \leq \beta^2 p$. Therefore, 
	\[ \err(\theta_{\text{ss}}) \leq \frac{2 \epsilon}{\gamma p}  \]
	In contrast, from Lemma~\ref{lem:gausstailtight},
	\begin{eqnarray*}
		\err(\theta_{\text{erm}}) & \geq &  \Phi\Big{(}\mu - \frac{ \frac{2}{3} \log(1/\beta) - 2 \log(1/\gamma)}{ 2 \sqrt{2 \log \beta p}}\Big{)} \\
	& = & \Phi\Big{(}\mu - \frac{ \log(\gamma^2/\beta^{2/3})}{ 2 \sqrt{2 \log \beta p}}\Big{)}	
	\end{eqnarray*}
	From Lemma~\ref{lem:gausstailtight}, this is at least:
	\begin{eqnarray*}
		&& \frac{1}{\sqrt{2 \pi}} \Big{(} \frac{1}{(\mu - \frac{ \log(\gamma^2/\beta^{2/3})}{ 2 \sqrt{2 \log \beta p}})} - \frac{1}{(\mu - \frac{ \log(\gamma^2/\beta^{2/3})}{ 2 \sqrt{2 \log \beta p}})^3} \Big{)} \cdot \exp( -\frac{1}{2}\Big{(} (\mu - \frac{ \log(\gamma^2/\beta^{2/3})}{ 2 \sqrt{2 \log \beta p}})^2) \Big{)} 
	\end{eqnarray*}
	For $p$ large enough, $\frac{1}{(\mu - \frac{ \log(\gamma^2/\beta^{2/3})}{ 2 \sqrt{2 \log \beta p}})^3} \leq \frac{1}{2} \frac{1}{(\mu - \frac{ \log(\gamma^2/\beta^{2/3})}{ 2 \sqrt{2 \log \beta p}})}$; also $\frac{1}{(\mu_n - \frac{ \log(\gamma^2/\beta^{2/3})}{ 2 \sqrt{2 \log \beta p}})} \geq \frac{1}{\mu}$. This implies that the right hand side is at least:
\begin{eqnarray*}
	\frac{1}{\sqrt{2 \pi}} \cdot \frac{1}{2 \mu} \cdot \exp( -\frac{1}{2}\Big{(} (\mu - \frac{ \log(\gamma^2/\beta^{2/3})}{ 2 \sqrt{2 \log \beta p}})^2) \Big{)}
\end{eqnarray*}
Observe that:
\begin{eqnarray*}
	&& \frac{1}{2 \sqrt{2 \pi} \mu } \cdot \exp( -\frac{1}{2}\Big{(} (\mu - \frac{ \log(\gamma^2/\beta^{2/3})}{ 2 \sqrt{2 \log \beta p}})^2) \Big{)} \\
	& \geq & \frac{1}{2 \sqrt{2 \pi} \mu } \cdot \exp(-\frac{1}{2}\mu^2) \cdot \exp( \frac{ \mu \log(\gamma^2/\beta^{2/3})}{ 2 \sqrt{2 \log \beta p}} - (\frac{ \log(\gamma^2/\beta^{2/3})}{ 2 \sqrt{2 \log \beta p}})^2) \\
	& \geq & \frac{1}{2} \cdot \frac{\epsilon}{p} \cdot \exp(\frac{1}{2} \frac{\mu \log(\gamma^2/\beta^{2/3})}{ 2 \sqrt{2 \log \beta p}})
\end{eqnarray*}
where the first step follows since $\frac{1}{\sqrt{2 \pi} \mu} \cdot e^{-\mu^2/2} = \epsilon/p$ and the second step follows since for large enough $p$, $\mu/2 \geq \frac{ \log(\gamma^2/\beta^{2/3})}{ 2 \sqrt{2 \log \beta p}}$.  

Also: observe that for $p$ large enough $\frac{\mu}{\sqrt{2 \log \beta p}} \geq \frac{1}{2}$ -- since $\mu \geq \frac{1}{2}\sqrt{2 \log (p/\epsilon)}$. Putting these all together, the entire error is at least:
\[ \frac{\epsilon}{2p} \exp( \frac{1}{8} \log(\gamma^2/\beta^{2/3})) \geq \frac{\epsilon \gamma^{1/4}}{2 p \beta^{1/12}} \]
The theorem now follows.
\end{proof}

\weibullclass*
\begin{proof}

Let $M_m$ denote the maximum of $m$ random variables drawn from $D(0)$. We observe that 
\[ \theta_{\text{ss}} = \frac{1}{2} (M_m - M'_m) ,\]
which for the Weibull case converges to $\frac{a_m}{2} (W - W') = \frac{a_m}{2}(Z' - Z)$; here, $W$ and $W'$ are reverse Weibull random variables with parameters $\alpha$ and $1$, which makes $Z \sim \wb(\alpha, 1)$ and $Z' \sim \wb(\alpha, 1)$ are independent Weibull variables, and $a_m = x_F - U(m)$. Combining this with Lemma~\ref{lem:lbweibull}, and accounting for the distributional convergence, we get that for $m$ large enough,
\[ \Pr(|\theta_{\text{ss}}| \geq \frac{(x_F - U(m))}{2} (\ln 2)^{1/\alpha}) \geq \frac{1}{2^{2^{\alpha}}} - \delta\]

Similarly,
\[ \theta_{\text{erm}} = \frac{1}{2} (M_m - M_n) ,\]
which in the Weibull case converges to $\frac{a_m}{2}(Z' - Z)$; here $Z \sim \wb(\alpha, 1)$ and $Z' \sim \wb(\alpha, \frac{a_n}{a_m})$. Observe that as $U(\cdot)$ is an increasing function, $\frac{a_n}{a_m} = \frac{x_F - U(n)}{x_F - U(m)} \leq 1$. We can therefore apply Lemma~\ref{lem:lbweibull2} to conclude that for large enough $m$ and $n$,
\[ \Pr(|\theta_{\text{erm}}| \leq \frac{x_F - U(m)}{2} (\ln 2)^{1/\alpha}) \geq \frac{1}{4} - \delta \]
The theorem follows.
\end{proof}

\begin{lemma}\label{lem:lbweibull}
Let $Z \sim \wb(\alpha, 1)$ and $Z' \sim \wb(\alpha, 1)$ be independent Weibull variables. Then, 
\[ \Pr(|Z - Z'| \geq (\ln 2)^{1/\alpha} ) \geq \frac{1}{2^{2^{\alpha}}}\]
\end{lemma}

\begin{proof}
For any constants $a$ and $b$, 
\begin{eqnarray*}
    \Pr(|Z - Z'| \geq a) \geq 2 \Pr(Z \geq a + b, Z' \leq b) = 2 \Pr(Z \geq a + b) \Pr(Z' \leq b) = 2 e^{-(a + b)^{\alpha}} (1 - e^{-b^{\alpha}}),
\end{eqnarray*}
where the first step follows because $Z$ and $Z'$ come from the same distribution, the second step follows from the independence of $Z$ and $Z'$, and the third step from plugging in the CDF for the Weibull distribution. 
Next, we plug in $a = (\ln 2)^{1/k}$ and $b = (\ln 2)^{1/k}$ -- and bound the final quantity as:
\[ \geq 2 \cdot \frac{1}{2} \cdot \exp( - (2 (\ln 2)^{1/\alpha})^\alpha) \geq \frac{1}{2^{2^{\alpha}}},\]
which is a constant for constant $\alpha$.
\end{proof}

\begin{lemma}\label{lem:lbweibull2}
Let $\lambda < 1$ and let $Z \sim \wb(\alpha, 1)$ and $Z' \sim \wb(\alpha, \lambda)$ be independent Weibull variables. Then,
\[ \Pr(|Z - Z'| \leq \text{median}(Z)) \geq \frac{1}{4} \]
\end{lemma}

\begin{proof}
Since both $Z$ and $Z'$ are positive random variables, $\Pr(|Z - Z'| \leq a) \geq \Pr(Z \leq a, Z' \leq a) = \Pr(Z \leq a) \Pr(Z' \leq a)$. Now, $\lambda < 1$, we can establish a coupling between $Z$ and $Z'$ to show that for any $a$, $\Pr(Z' \leq a) \geq \Pr(Z = a)$. The lemma follows by plugging this in, and setting $a = \text{median}(Z)$. 
\end{proof}

Finally, the proof to Corollary~\ref{cor:uniform} follows by substituting the expression for $a_n$ and $b_n$ for the Uniform distribution in Theorem~\ref{thm:weibull1d}.
\subsection{Imbalanced Classes: Higher Dimensional Case}
\label{sec:imb_class_high_d_append}

Denote $\hat{v}$ as a unit vector in the direction of the vector $v$. In the lemmas below, we prove some facts that are necessary to prove the main Theorem \ref{thm: imb_class_highd}. $\hat{\mu}$ is a unit vector in the direction of $\mu$ (conditional mean of positive class) and $\hat{y}$ is a unit vector orthogonal to $\hat{\mu}$.

\begin{lemma}
\label{lemma: sph_symm_median}
    If $D(0)$ is spherically symmetric, then the median $\hat{y}^{\top}x $ conditioned on $\hat{\mu}^{\top} x $ is zero, i.e., $\text{Median}[\hat{y}^{\top}x | \hat{\mu}^{\top} x ]= 0$, where $\hat{\mu} \perp \hat{y}$. 
\end{lemma}
\begin{proof} 
To show the above lemma, we will first prove that the joint probability $\Pr(\hat{\mu}^{\top} x, \hat{y}^{\top}x ) = \Pr(\hat{\mu}^{\top} x, \hat{z}^{\top}x )$, where $\hat{y}$ and $\hat{z}$ are any two unit vectors perpendicular to $\hat{\mu}$. 
Consider two orthonormal basis $\mathcal{B}_1$ and $\mathcal{B}_2$ to express vectors $x\in \mathbb{R}^{d}$. We write the coordinates of $x$ in $\mathcal{B}_1$ as $\{\upsilon^{1}_1, \cdots, \upsilon^{1}_d\}$ and in $\mathcal{B}_2$ as $\{\upsilon^{2}_1, \cdots, \upsilon^{2}_d\}$. Since $x$ is drawn from a spherically symmetric distribution 
$\Pr(\upsilon^{1}_1, \cdots, \upsilon^{1}_d) =  \Pr(\upsilon^{2}_1, \cdots, \upsilon^{2}_d)$. As a result,  $\Pr(\upsilon^{1}_1, \upsilon^{1}_2) =  \Pr(\upsilon^{2}_1, \upsilon^{2}_2)$. Suppose the first two vectors in $\mathcal{B}_1$ are $\hat{\mu}$ and $\hat{y}$ and suppose the first two vectors in $\mathcal{B}_2$ are $\hat{\mu}$ and $\hat{z}$. Thus, $\upsilon_1^{1} = \hat{\mu}^{\top} x$, $\upsilon_2^{1} = \hat{y}^{\top} x$,  and  $\upsilon_1^{2} = \hat{\mu}^{\top} x$, $\upsilon_2^{2} = \hat{z}^{\top} x$. As a result, $\Pr(\hat{\mu}^{\top} x, \hat{y}^{\top}x ) = \Pr(\hat{\mu}^{\top} x, \hat{z}^{\top}x )$. This implies $$ \Pr( \hat{y}^{\top}x  \leq  a | \hat{\mu}^{\top} x) = \Pr(\hat{z}^{\top}x \leq a | \hat{\mu}^{\top}x) $$ 
Substitute $\hat{z} = -\hat{y}$ to get 
$$ \Pr( \hat{y}^{\top}x  \leq  a | \hat{\mu}^{\top} x) = \Pr(-\hat{y}^{\top}x \leq a | \hat{\mu}^{\top}x) $$ 
$$ \Pr( \hat{y}^{\top}x  \leq  a | \hat{\mu}^{\top} x) = \Pr(\hat{y}^{\top}x \geq -a | \hat{\mu}^{\top}x) $$ 
If $a=0$, then $\Pr( \hat{y}^{\top}x  \leq  0 | \hat{\mu}^{\top}x) = \Pr( \hat{y}^{\top}x  \geq  0| \hat{\mu}^{\top}x) $. This implies that the conditional median is zero. 
\end{proof}

\begin{lemma}
\label{lem:median_ord}
    If $D(0)$ is spherically symmetric, then the median $\hat{y}^{\top}x^{(i)} $ conditioned on $\hat{\mu}^{\top} x^{(i)} $ is zero, i.e., $\text{Median}[\hat{y}^{\top}x^{(i)} | \hat{\mu}^{\top} x^{(i)} ]= 0$, where $\hat{\mu} \perp \hat{y}$ and $x^{(i)}$ corresponds to the $x$ in $\{x_1, \cdots, x_n\}$ with $i^{th}$ largest projection on $\hat{\mu}$. 
\end{lemma}

\begin{proof}
For cleaner exposition, without loss of generality we will say $x_{j}$ takes the $j^{th}$ highest projection on $\hat{\mu}$ for all $j \in \{1, \cdots, n\}$.

We write the joint probability over $\hat{y}^{\top} x_{i}$, and all $\hat{\mu}^{\top} x_{j}$ as 
$$\Pr\bigg(\hat{y}^{\top} x_{i}, \hat{\mu}^{\top}x_{i}=u, \{\hat{\mu}^{\top}x_{j} \leq u, \forall j \leq i-1\},   \{\hat{\mu}^{\top}x_{k} \geq u, \forall k \geq i+1\}  \bigg)$$
\begin{equation}
=  \Pr(\hat{y}^{\top} x_{i}, \hat{\mu}^{\top}x_{i}=u) F_{\hat{\mu}}(u)^{i-1}(1- F_{\hat{\mu}}(u))^{n-i}
    \label{eqn1: joint_ord_med1}
\end{equation}
where $ F_{\hat{\mu}}$ is the CDF of $x$ projected on $\hat{\mu}$. 


The conditional probability simplifies as follows
\begin{equation}
      \Pr(\hat{y}^{\top} x^{(i)} | \hat{\mu}^{\top}x^{(i)} ) =  \frac{  \Pr(\hat{y}^{\top} x_{i}, \hat{\mu}^{\top}x_{i}=u) F_{\hat{\mu}}(u)^{i-1}(1- F_{\hat{\mu}}(u))^{n-i}}{ \Pr(\hat{\mu}^{\top}x_{i}=u) F_{\hat{\mu}}(u)^{i-1}(1- F_{\hat{\mu}}(u))^{n-i}} = \Pr(\hat{y}^{\top} x_{i} | \hat{\mu}^{\top}x_{i})
\end{equation}
The rest of the lemma now follows from the previous Lemma \ref{lemma: sph_symm_median}. 
\end{proof}

\begin{lemma}
\label{lem:orderinglemma}
Suppose $n$ iid samples $\{x_1, \cdots, x_n\}$ are sampled from a spherically symmetric $D(0)$.  Let $s = d \log (2N/\epsilon+1) + \log (1/\delta)$  where $N = \max_{i\in \{1, \cdots, n\}} \|x_i\|$. Then, with probability $\geq 1 - \delta$, for all directions $\hat{y}$ in $\mathbb{R}^d$ there exists an $i \in \{ 1, \ldots, s \}$ such that
\[ \hat{y}^{\top} x^{(i)} \geq -\epsilon \]
where $\{x^{(1)}, \cdots, x^{(n)}\}$ are in the decreasing order of their projection on $\hat{\mu}$, where $\hat{\mu} \perp \hat{y}$. 
\end{lemma}

\begin{proof}
We select an $\epsilon/N$ cover $C = \{ y_1, \ldots, y_M \}$ over the surface of the sphere; this is possible for $M = (\frac{2N}{\epsilon } + 1)^{d}$.\footnote{\url{https://people.eecs.berkeley.edu/~bartlett/courses/281b-sp08/19.pdf}} We will prove the lemma in two steps. First, we will show that with probability $\geq 1 - \delta$,  for all $y_j \in C$, there exists a $x_i$ in $x^{(1)}, \ldots, x^{(s)}$ for which $y_j^{\top} x_i \geq 0$. To show this, first let us fix a particular $y_j$, and consider $x^{(1)}, \ldots, x^{(s)}$. For any of these $x^{(i)}$'s, $\Pr(y_j^{\top} x^{(i)} \geq 0 | \hat{\mu}^{\top} x^{(i)}) = 1/2$ (from Lemma \ref{lem:median_ord}) and as a result $\Pr(y_j^{\top} x^{(i)} \geq 0) = 1/2$. This means that the probability that $y_j^{\top} x^{(i)} < 0$ for all $x^{(i)}$ in $x^{(1)}, \ldots, x^{(s)}$ is at most $1/2^s$. For $s = d \log (2N/\epsilon+1) + \log (1/\delta)$, this probability is at most $\delta / (\frac{2N}{\epsilon}+1)^d$. 
Now, let $\hat{y}$ be any vector on the surface of the sphere. Suppose $y_j$ is its closest vector in $C$, and $x_i$ is the corresponding $x$ such that $y_j^{\top} x_i \geq 0$. Since $C$ is an $\epsilon/N$-cover of the sphere, this means that:
\[ \hat{y}^{\top} x_i \geq y_j^{\top} x_i - \frac{\epsilon}{N} \|x_i\| \geq - \epsilon \]
The lemma follows. 
\end{proof}

Denote set of points in positive class as $B$ and the set of points in negative class as $A$. We simplify the optimal solution to SVM as follows. 

\begin{equation}
    \begin{split}
         w^{*} & = \argmax_{\|w\| = 1} \inf_{x \in B} w^{\top} x - \sup_{x \in A} w^{\top} x  \\ 
         & = \argmin_{\|w\| = 1} -\inf_{x \in B} w^{\top} x + \sup_{x \in A} w^{\top} x  \\ 
         & = \argmin_{\|w\| = 1} \sup_{x \in B} -w^{\top} x + \sup_{x \in A} w^{\top} x  \\ 
         & = \argmin_{\|w\| = 1} \sup_{x \in -B} w^{\top} x + \sup_{x \in A} w^{\top} x
    \end{split}
\end{equation}

Additionally,
\[ -b^* = \frac{1}{2} ( \sup_{x \in A} (\alpha \hat{\mu} + \beta \hat{y})^{\top})  x - \sup_{x \in -B} (\alpha \hat{\mu} + \beta \hat{y})^{\top}) x) \]

Define the set $A_{\mu} = \{ x+ \mu, \forall x \in A \}$, and the set $-B_{\mu} = \{ x+ \mu, \forall x \in -B \}$.

 With this, and some algebraic simplification the SVM optimization problem becomes:
\[  \argmin_{\alpha \in [-1,1], \hat{y}} \sup_{x \in A_{\mu}} (\alpha \hat{\mu} + \beta \hat{y})^{\top} (x - \mu) + \sup_{x \in -B_{\mu}} (\alpha \hat{\mu} +  \beta \hat{y})^{\top} (x - \mu) \]

\[\argmin_{\alpha \in [-1,1], \hat{y}} - 2 \alpha \|\mu\| + \sup_{x \in A_{\mu}} (\alpha \hat{\mu} + \beta \hat{y})^{\top} )x  + \sup_{x \in -B_{\mu}} (\alpha \hat{\mu} + \beta \hat{y})^{\top}) x \]
\begin{equation} \label{eqn:minform} w^* = \argmin_{\alpha, \hat{y}} - 2 \alpha \|\mu\| + \sup_{x \in A_{\mu}} (\alpha \hat{\mu} + \beta \hat{y})^{\top} )x  + \sup_{x \in -B_{\mu}} (\alpha \hat{\mu} + \beta \hat{y})^{\top}) x 
\end{equation}

Additionally,
\[ -b^* = \frac{1}{2} ( \sup_{x \in A_{\mu}} (\alpha \hat{\mu} + \beta \hat{y})^{\top})  x - \sup_{x \in -B_{\mu}} (\alpha \hat{\mu} + \beta \hat{y})^{\top}) x) \]

\imbclassd*


\begin{proof}

We start with expression for optimal SVM solution derived above 
\begin{equation} \label{eqn:minform} w^* = \argmin_{\alpha, \hat{y}} - 2 \alpha \|\mu\| + \sup_{x \in A_{\mu}} (\alpha \hat{\mu} + \beta \hat{y})^{\top} )x  + \sup_{x \in -B_{\mu}} (\alpha \hat{\mu} + \beta \hat{y})^{\top}) x 
\end{equation}
Additionally,
\[ -b^* = \frac{1}{2} ( \sup_{x \in A_{\mu}} (\alpha \hat{\mu} + \beta \hat{y})^{\top})  x - \sup_{x \in -B_{\mu}} (\alpha \hat{\mu} + \beta \hat{y})^{\top}) x) \]

Let us try to bound $-b^{*}$. From the concentration condition in Assumption \ref{assm: concentration_high_dim}, we know the first term above lies in $$\frac{1}{2}[X_{\max}(p,\delta,d)-c(p,\delta,d), X_{\max}(p,\delta,d)+C(p,\delta,d)]$$
The second term lies in $$\frac{1}{2}[X_{\max}(m,\delta,d)-c(m,\delta,d), X_{\max}(m,\delta,d)+C(m,\delta,d)]$$

As a result, the lower bound on $-b^{*}$ is 
\begin{equation}    
\frac{1}{2}\bigg(X_{\max}(p,\delta,d)-X_{\max}(m,\delta,d) - c(p,\delta,d) - C(m,\delta,d)\bigg)
\label{eqn: lb_b}
\end{equation}

The upper bound on $-b^{*}$ is 
\begin{equation}    
\frac{1}{2}\bigg(X_{\max}(p,\delta,d)-X_{\max}(m,\delta,d) + C(p,\delta,d) + c(m,\delta,d)\bigg)
\label{eqn: ub_b}
\end{equation}

Note $-b^{*}$ lies in the above interval with probability at least $1-2\delta$.

The expression for the error of a classifier $w^{\top}x + b$ is given as follows. We assume $\|w\|=1$. 

\begin{equation} 
\begin{split}
& \mathsf{Err}_{+} = \mathbb{P}(w^{\top}X + b\leq 0 | X \sim D(\mu)) \\  
\end{split}
\end{equation}

where $D(\mu)$ is the distribution of samples for the positive class centered at $\mu$. We assume $D$ is spherically symmetric about zero so we simplify the above expression as follows.

\begin{equation}
\begin{split}
&  \mathsf{Err}_{+} =\mathbb{P}(w^{\top}(\mu +\tilde{X}) + b\leq 0 | \tilde{X} \sim D(0)) \\  
\end{split}
\end{equation}

Since $\tilde{X}$ is sampled from $D(0)$ which is spherically symmetric, its projection on $w^T\tilde{X}$ will have a distribution that does not depend on the direction $w$. Let us denote $w^{\top}\tilde{X}= W$. The above expression becomes. Let us denote the CDF of $W$ as $F_W$.

\begin{equation}
\begin{split}
&  \mathsf{Err}_{+} =\mathbb{P}(W \leq -w^{\top}\mu -b ) = F_W(-w^{\top}\mu -b)
\end{split}
\end{equation}

 We now plug in expression for the max-margin solution.  In the analysis above, we showed that $-b^{*} \in [a_{\min}, a_{\max}]$. Therefore, the error for the positive class 
\begin{equation}
\begin{split}
& F_W(-w^{\top}\mu + a_{\min})  \leq   \mathsf{Err}_{+}   \leq F_W(-w^{\top}\mu + a_{\max}) \\ 
 & F_W(-\alpha \|\mu\| + a_{\min})  \leq   \mathsf{Err}_{+}   \leq F_W(-\alpha \|\mu\| + a_{\max}) \\
 & F_W(-\|\mu\| + \frac{1}{2}\big(X_{\max}(p,\delta,d) - X_{\max}(m,\delta,d) - c(p,\delta,d)-C(m,\delta,d)\big)) \leq \mathsf{Err}_{+} \leq \\
 & F_W(-\alpha\|\mu\| + \frac{1}{2}\big(X_{\max}(p,\delta,d) - X_{\max}(m,\delta,d) + C(p,\delta,d)+c(m,\delta,d)\big))
 \end{split}
\end{equation}
 In the upper bound, we invoked a condition that for the optimal $w = \alpha \hat{\mu} + \beta \hat{y}$, where $\beta = \sqrt{1-\alpha^2}$.  In the lower bound, we set $\alpha$ to one.  Consider the following two cases. 

 \begin{itemize}
 \item \textbf{Balanced class case:} $m=p$
 \begin{equation}
 F_W\bigg(-\|\mu\| - \frac{1}{2}c(p,\delta,d)-\frac{1}{2}C(m,\delta,d)\bigg) \leq \mathsf{Err}_{+} \leq  F_W\bigg(-\alpha \|\mu\| + \frac{1}{2}C(p,\delta,d)+\frac{1}{2}c(m,\delta,d)\bigg)
 \end{equation}
 \item \textbf{Imbalanced class case:} $p>>m$. In this case, the error at least grows as  $$F_W\big(-\|\mu\| + \frac{1}{2}\big(X_{\max}(p,\delta,d) - X_{\max}(m,\delta,d)\big)\big) $$
 \end{itemize}
 We need to show that the upper bound of the balanced case is better than the lower bound of the imbalanced case, which boils down to the following 
 \begin{equation}
 \begin{split}
&F_W\bigg(-\alpha\|\mu\| + \frac{1}{2}C(p,\delta,d)+\frac{1}{2}c(m,\delta,d)\bigg) \leq  \\ &F_W\bigg(-\|\mu\| + \frac{1}{2} \big( X_{\max}(p,\delta,d) - X_{\max}(m,\delta,d) - c(p,\delta,d)-C(m,\delta,d)\big)\bigg)       \\ 
\end{split}
 \end{equation}

  If $\alpha \geq  1 - \frac{X_{\max}(p,\delta,d) - X_{\max}(m,\delta,d) -\bar{c}(p,m, \delta,d)}{2\|\mu\|} $,  where $\bar{c}(p,m, \delta,d) = C(p,\delta,d) + c(p,\delta,d) + C(m,\delta,d)+ c(m,\delta,d) $ then the above inequality holds true.

 We now show that 
 $\alpha \geq  1 - \eta$, where  $ \eta = \frac{X_{\max}(p,\delta,d) - X_{\max}(m,\delta,d) -\bar{c}(p,m, \delta,d)}{2\|\mu\|} $.  Since $\|\mu\|>  X_{\max}(p,\delta,d)$, $\eta< \frac{1}{2}$.  To confirm that $1 - \frac{X_{\max}(p,\delta,d) - X_{\max}(m,\delta,d) -\bar{c}(p,m, \delta,d)}{2\|\mu\|} \leq 1$, we need to check that $X_{\max}(p,\delta,d) \geq X_{\max}(m,\delta,d) + \bar{c}(p,m, \delta,d) \geq 0$.  Observe that $\zeta(p, \delta, d, q)  > 4X_{\max}(m,\delta,d)$, which implies 
 $X_{\max}(p, \delta, d, q)  > 4X_{\max}(m,\delta,d)$.
 $X_{\max}(p,\delta,d) - \big(X_{\max}(m,\delta,d) + \bar{c}(p,m, \delta,d)\big)$, which is lower bounded $3X_{\max}(m,\delta,d)-\bar{c}(p,m, \delta,d)$. Note that the second term $\bar{c}(p,m, \delta,d)$ diminishes to zero for sufficiently large $m$ and $p$ while the first term is positive, which shows that $\alpha \leq 1$. 


Recall the SVM objective is 
\[- 2 \alpha \|\mu\| + \sup_{x \in A_{\mu}} (\alpha \hat{\mu} + \beta \hat{y})^{\top} x  + \sup_{x \in -B_{\mu}} (\alpha \hat{\mu} + \beta \hat{y})^{\top} x \]
where $x\in D(0)$.  For a fixed $\alpha$, let $\hat{y}(\alpha)$ denote the minimizer of the above. 

We compare the SVM objective when $\alpha=1$ to a lower bound on the optimal value achievable if $\alpha<1-\eta$. When $\alpha=1$ the objective becomes
\begin{equation}
-2\|\mu\| +  \sup_{x \in A_{\mu}} (\hat{\mu}^{\top} x) + \sup_{x \in -B_{\mu}} (\hat{\mu}^{\top} x)
\end{equation}

Recall $q = d \log(N/\epsilon + 1) + \log(1/\delta)$, where $N = \max_{i \in A} \|x_i\|$, where $\epsilon= \frac{1}{\log p}$. Similarly, define $\tilde{q} = d \log(\tilde{N}/\epsilon + 1) + \log(1/\delta)$, where $\tilde{N} = \max_{i \in B} \|x_i\|$. Consider the $q^{th}$ and $\tilde{q}^{th}$ highest value for $\hat{\mu}^{\top} x$ on set $A$ and set $B$ respectively. For a fixed $\alpha$, we obtain a lower bound for the SVM objective in terms of $q^{th}$ and $\tilde{q}^{th}$ highest values as follows. 

\[- 2 \alpha \|\mu\| + \alpha \hat{\mu}^{\top} x^{(i)}_{+}  + \alpha \hat{\mu}^{\top} x^{(j)}_{-}  + \beta\hat{y}(\alpha)^{\top} x^{(i)}_{+} + \beta \hat{y}(\alpha)^{\top} x^{(j)}_{-}  \]
where $x^{(j)}_{-}$ has one of the top $q$ projections on $\hat{\mu}$ among the positive samples, where $x^{(i)}_{+}$ has one of the top $\tilde{q}$ projections on $\hat{\mu}$ among the negative samples.  We use Lemma \ref{lem:orderinglemma} to arrive at a lower bound on the SVM objective. To use Lemma \ref{lem:orderinglemma}, we need $\hat{y}^{\top}x_{+}^{(i)}$ to have a median of zero conditional on $\hat{\mu}^{\top}x_{+}^{(i)}$.  We also need a similar condition for $\hat{y}^{\top}x_{-}^{(j)}$ conditional on $\hat{\mu}^{\top}x_{-}^{(j)}$. These conditions follow from  Lemma \ref{lem:median_ord}. 

With probability $1-2\delta$ the lower bound on the objective is
\[\alpha (- 2\|\mu\| + \hat{\mu}^{\top} x^{(i)}_{+}  + \hat{\mu}^{\top} x^{(j)}_{-})  -2\epsilon \]
 
We minimize this lower bound for $\alpha \in [-1, 1-\eta)$ and obtain the following 
\[(1-\eta)(- 2\|\mu\| + \hat{\mu}^{\top} x^{(i)}_{+}  + \hat{\mu}^{\top} x^{(j)}_{-})  -2\epsilon \]

where we use the following fact  $\|\mu\| >  X_{\max}(p,\delta,d) \geq X_{\max}(m,\delta, d) + \bar{c}(p,m,\delta,d)$. 
We will now show that the lower bound above has a very low probability to improve upon the objective value for $\alpha=1$. As a result, optimal $\alpha$ will have
to be more than $1-\eta$. Let us consider the event

\begin{equation}
\begin{split}
(1-\eta)(- 2\|\mu\| + \hat{\mu}^{\top} x^{(i)}_{+}  +  \hat{\mu}^{\top} x^{(j)}_{-})  -2\epsilon \leq  \\
-2\|\mu\| +  \sup_{x \in A_{\mu}} (\alpha \hat{\mu}^{\top} x) + \sup_{x \in -B_{\mu}} (\alpha \hat{\mu}^{\top} x)
\end{split}
\label{eqn: event_lb}
\end{equation}

After rearrangement we get 

\[ \eta\|\mu\|  + \frac{1}{2}(1-\eta) \bigg(\hat{\mu}^{\top} x^{(i)}_{+}  +  \hat{\mu}^{\top} x^{(j)}_{-})\bigg) \leq \frac{1}{2}\bigg(\sup_{x \in A_{\mu}} ( \hat{\mu}^{\top} x) + \sup_{x \in -B_{\mu}} ( \hat{\mu}^{\top} x) + 2\epsilon \bigg)\]


We substitute $\epsilon = 1/\log  p$ and use the expression for $\eta \|\mu\|$  to get
\begin{equation*}
    \begin{split}
        &\big(X_{\max}(p,\delta, d) - X_{\max}(m,\delta,d) -\bar{c}(p,m,\delta, d)\big) +  (1-\eta) \bigg(\hat{\mu}^{\top} x^{(i)}_{+}  +  \hat{\mu}^{\top} x^{(j)}_{-})\bigg)\leq  \\
        & \bigg(\sup_{x \in A_{\mu}} ( \hat{\mu}^{\top} x) + \sup_{x \in -B_{\mu}} ( \hat{\mu}^{\top} x) + \frac{2}{\log p} \bigg)
    \end{split}
\end{equation*}

After further rearrangement we get 

\begin{equation*}
    \begin{split}
        & (1-\eta) \bigg(\hat{\mu}^{\top} x^{(i)}_{+}  +  \hat{\mu}^{\top} x^{(j)}_{-})\bigg)\leq \\
        &  \bigg(\sup_{x \in A_{\mu}} ( \hat{\mu}^{\top} x) + \sup_{x \in -B_{\mu}} ( \hat{\mu}^{\top} x) -  X_{\max}(p,\delta, d) + X_{\max}(m,\delta,d) + \frac{2}{\log p} \bigg)  + \bar{c}(p,m,\delta, d)
    \end{split}
\end{equation*}

From the above we get

\begin{equation*}
\begin{split}
    & (1-\eta) \bigg(\hat{\mu}^{\top} x^{(i)}_{+}  +  \hat{\mu}^{\top} x^{(j)}_{-})\bigg)\leq  \\
    & \bigg(\sup_{x \in A_{\mu}} ( \hat{\mu}^{\top} x) + \sup_{x \in -B_{\mu}} ( \hat{\mu}^{\top} x) -  X_{\max}(p,\delta, d) + X_{\max}(m,\delta,d) + \frac{2}{\log p} \big)  + \bar{c}(p,m,\delta, d)
\end{split}
\end{equation*}


Since $\eta< \frac{1}{2}$  we can further simplify the LHS with a weaker lower bound

\[ \frac{1}{2} \bigg(\hat{\mu}^{\top} x^{(i)}_{+}  +  \hat{\mu}^{\top} x^{(j)}_{-})\bigg) \leq \bigg(\sup_{x \in A_{\mu}} ( \hat{\mu}^{\top} x) + \sup_{x \in -B_{\mu}} ( \hat{\mu}^{\top} x) -  X_{\max}(p,\delta, d) + X_{\max}(m,\delta,d) + \frac{2}{\log p} \bigg)  + \bar{c}(p,m,\delta, d) \]

We write an upper bound for RHS using the concentration condition. 

\[ \frac{1}{2} \bigg(\hat{\mu}^{\top} x^{(i)}_{+}  +  \hat{\mu}^{\top} x^{(j)}_{-})\bigg) \leq 2X_{\max}(m,\delta,d)  + \bar{c}(p,m,\delta, d) + C(m, \delta, d) + C(p, \delta, d) + \frac{2}{\log p} \]



Using the lower bound from the concentration condition in Assumption \ref{assm: concentration_qth_order_main}, we get the following lower bound

\[ \zeta(m,\delta, d, \tilde{q}) +  \zeta(p, \delta, d, q) \leq 4X_{\max}(m,\delta,d)  + 2\bar{c}(p,m,\delta, d) + 2C(m, \delta, d) + 2C(p, \delta, d) + \frac{4}{\log p} \] 

Since $\zeta(p, \delta, d, q)  > 4X_{\max}(m,\delta,d)  $ for a sufficiently large $m$ and $p$ we gather that $\zeta(p, \delta, d, q) + \zeta(m,\delta, d, \tilde{q}) \geq  4X_{\max}(m,\delta,d)  + 2\bar{c}(p,m,\delta, d) + 2C(m, \delta, d) + 2C(p, \delta, d) + \frac{4}{\log p}$. As a result, the above inequality does not hold. Therefore, the event in  \eqref{eqn: event_lb} occurs with a probability at most $2\delta$. As a result, we obtain that $\alpha \geq 1-\eta$ with a probability at least $1-4\delta$. We showed above that if $\alpha \geq 1-\eta$, then  with probability at least $1-4\delta$, worst class error improves under data balancing. The intersection of these two events occurs with a probability at least $1-8\delta$.

\end{proof}

\paragraph{Expression for $\zeta$}
In this section,  our goal is to derive a lower bound on the $q^{th}$ maximum projection of $\hat{\mu}$ across different data samples. We denote $V =\hat{\mu}^{\top}X$.

We first make some observations that we use subsequently. 
Consider the event $V^{(q)}>t$, where $V^{(q)}$ is $q^{th}$ highest value of $V$ among $p$ samples. Suppose the CDF of $V$ is $F_V$. Find a value $r$ such that $F_V(t)=1-r$. This denotes $r$ fraction of $V$ is greater than $t$. Define $U_i = I(V_i>t)$, where $I$ is the indicator function and $U_i$ is one when $V_i>r$ and zero otherwise. 
Consider the event 
$$\sum_{i=i}^{p}U_i>q$$
If the above event is true, then that implies there are at least $q$ values that are above $t$ and thus $V^{(q)}>t$. Also, if $V^{(q)}>t$, then there exist at least $q$ $U_i$'s that are one. Thus the above two events are equivalent. The expectation $\mathbb{E}[\sum_{i=1}^{p}U_i] = pr$. Let $q=\frac{pr}{2}$. 
We use Chernoff bound to arrive at the following bound
$$\Pr(\sum_{i=i}^{p}U_i<q) < e^{-\frac{pr}{8}}$$

We set $\frac{pr}{8} = d \log(\log p\max_{x_i\in A}\|x_i\|) + \log(1/\delta)$. 

Therefore, $r=8\frac{d \log p \log(\max_{x_i\in A}\|x_i\|/\epsilon) + \log(1/\delta)}{p}$. For the case of symmetric $d$ dimensional Gaussians centered at zero we get, $\frac{pr}{8} = d \log(\log p \sqrt{d\log p})) + \log(1/\delta) $. We simplify the bound $e^{-\frac{pr}{8}}$ as follows. 
$$e^{-\frac{pr}{8}} \leq  e^{-d \log(\log p \sqrt{d\log p}))} \leq \frac{1}{(\log p \sqrt{d \log p})^d}$$

For sufficiently large $p$, the probability falls below any $\delta$. 

We now derive a bound on $t$. 

Recall $F(t) = 1-r$, which simplifies for a Gaussian to $Q(t) = r$, where $Q$ is the $Q$ function. Since $Q(t) \leq e^{-t^2}$. We get $e^{-t^2} \geq r$, which implies $$t\leq \sqrt{2\log\frac{1}{r}} = \sqrt{2 \log \frac{p}{8(d \log(\log p \sqrt{d\log p}) + \log(1/\delta)) }} $$
Hence, we can use $\zeta(p, \delta, d) = \sqrt{2 \log \frac{p}{8(d \log(\log p \sqrt{d\log p}) + \log(1/\delta)) }} $ for symmetric Gaussian distributions.


\imbclasscor*

 

\begin{proof}
To prove this Corollary, we leverage Theorem \ref{thm: imb_class_highd} and its proof.  In \eqref{eqn: lb_b} and \eqref{eqn: ub_b} we derive the upper and the lower bounds for the bias. For a $d$ dimensional spherically symmteric Gaussian, the expressions for concentration condition are derived in Lemma \ref{lem:1dgaussianmax_new}. For a sufficiently large $m$, $p$, the bias is centered at $$ \sqrt{2\log \frac{p}{\delta}} - \sqrt{2\log \frac{m}{\delta}}$$ Observe that $X_{\max}(p, \delta, d) \approx \sqrt{2\log \frac{p}{\delta}}$ and $X_{\max}(m, \delta, d) \approx \sqrt{2\log \frac{m}{\delta}}$. If we subsample, then we are in the case, where  $X_{\max}(p, \delta, d) = X_{\max}(m, \delta, d)$ and as a result the bias term is centered at zero. 
If $m$ grows as $\log p$, then the upper bound on  $X_{\max}(m, \delta, d)$ is $\sqrt{2\log \log p}$. As a result, the condition that  $\zeta(p, \delta, d,q) > 4X_{\max}(m, \delta, d)$ is satisfied for sufficiently large $p$. Finally, if $\|\mu\|$ for the mean of the Gaussian is more than $\sqrt{2\log(p/\delta)} - \sqrt{2\log(m/\delta)}$,  then it follows from the previous theorem that subsampling improves the worst group error.   
For a $2$ dimensional symmetric uniform, the expressions for the concentration condition are derived in Lemma \ref{lem:1dunif_dir}. For a sufficiently large $m$, $p$, the bias term is centered at zero with the interval given as 
$$ \bigg[-\frac{\varrho}{p^{2/3}}-\frac{1}{p}, \frac{1}{m} + \frac{\varrho}{m^{2/3}}\bigg]$$
where $\varrho$ is a constant whose expression can be obtained from Lemma \ref{lem:1dunif_dir}.
\end{proof}

\section{Proofs for imbalanced groups}
\label{sec:proof_groups}
\subsection{Two-dimensional case}
\label{sec:proof_groups_2d}

\begin{lemma}\label{lem:1dgaussianmax}
Let $x_1, \ldots, x_n$ be $n$ i.i.d unit Gaussians and let $X_{\max} = \max(x_1, \ldots, x_n)$. Then for $n$ large enough, with probability $\geq 1 - 3\delta$, we have that: 
\[ X_{\max} \leq b_n + a_n \log (1/\delta), \quad X_{\max} \geq b_n - a_n \log \log (1/\delta) \]
where $a_n=\frac{1}{\sqrt{2\log(n)}}$ and $b_n=\sqrt{2\log(n)}- \frac{\log\log n + \log(4\pi)}{\sqrt{2\log n}}$ are the constants in the Fisher-Tippett-Gnedenko theorem when applied to Gaussians.
\end{lemma}

\begin{proof} From the Fisher-Tippett-Gnedenko theorem, when $n$ is large enough, we have that $ X_{\max} \xrightarrow{d} a_n Z + b_n$, where $\xrightarrow{d}$ stands for convergence in distribution, and $Z$ is a standard Gumbel random variable. If $n$ is sufficiently large, then for any $t \in \mathbb{R}$,  $|\Pr(X_{\max}  \leq t)  - \Pr(a_n Z + b_n \leq t)| \leq \delta/2$. 

For a standard Gumbel variable $Z$, we have that:
\[ \Pr(Z \leq \log (1/\delta)) = \exp(-\exp(-\log(1/\delta))) = \exp(-\delta) \geq 1-\delta \]
As a result, 
\[ \Pr(X_{\max} \leq b_n + a_n \log (1/\delta)  \geq 1-\frac{3\delta}{2} \]
Additionally, we have:
\[ \Pr(Z \geq - \log \log (1/\delta)) = 1-\exp(-\exp(\log\log(1/\delta))) = 1-\delta \]
As a result, 
\[ \Pr(X_{\max} \geq b_n-a_n\log\log (1/\delta))  \geq 1-\frac{3\delta}{2} \]

Finally, if we take a union bound on the complement of the above two events and then the lemma follows.  
\end{proof}

\begin{lemma}\label{lem:gaussiannorm}
Let $x_1, \ldots, x_n$ be vectors in $\mathbb{R}^2$ drawn i.i.d from $N(0, I_2)$. Then, with probability $\geq 1 - \delta$, 
\[ \max_i \|x_i\| \leq  \sqrt{2 \log (n/\delta)} \]
\end{lemma}

\begin{proof} $\|x_i\| $  follows a Rayleigh distribution and we use this observation to arrive at the above result. 
\begin{equation}
    \begin{split}
        & \Pr\bigg(\max_i \|x_i\| \leq  \sqrt{2 \log (n/\delta)}\bigg)  = 1- \Pr\bigg(\max_i \|x_i\| \geq  \sqrt{2 \log (n/\delta)}\bigg) \geq  \\  
        & 1 - n \Pr\bigg(\|x_i\|\geq  \sqrt{2\log(n/\delta)}\bigg) =  1- n e^{-\log(n/\delta)} = 1-\delta
    \end{split}
\end{equation}
\end{proof}


\begin{lemma}\label{lem:1dgaussianmax_dir}
Let $x_1, \ldots, x_n$ be $n$ i.i.d unit Gaussians with covariance $I_2$.  Then for $n$ large enough, with probability $\geq 1 -  \delta$, we have that for all directions $v\in \mathbb{R}^2$: 
\[  \max_{i\in\{1,\cdots,n\}}\{v^{\top}x_i\} \leq  b_n + a_n +a_n \log \bigg(\frac{6\sqrt{2\log (2n/\delta)}}{a_n\delta}\bigg),\] 
\[\max_{i\in\{1,\cdots,n\}}\{v^{\top}x_i\}  \geq b_n - a_n- a_n \log\log \bigg(\frac{6\sqrt{2\log (2n/\delta)}}{a_n\delta}\bigg) \]
where $a_n$ and $b_n$ are the constants in the Fisher-Tippett-Gnedenko theorem when applied to Gaussians.
\end{lemma}
\begin{proof} 
Suppose $f(v) = \max_{i} v^{\top} x_i$ where $v$ is a unit vector in $\mathbb{R}^2$. Then, 
\[ f(v) - f(u) = \max_i v^{\top} x_i - \max_i u^{\top} x_i \leq \max_i (v - u)^{\top} x_i \leq \| v - u \| \cdot \max_i \|x_i\| \]
where the first step follows from definition, the second step from subtracting a smaller quantity, and the last step from the Cauchy-Schwartz inequality. 

From Lemma~\ref{lem:gaussiannorm}, with probability $\geq 1 - \delta/2$, $\max_i \|x_i\| \leq  \sqrt{2\log (2n/\delta)}$, which gives us:
\[ f(v) - f(u) \leq  \sqrt{2\log (2n/\delta)} \cdot \| v - u \| \]

Now, we can build an $\epsilon$-cover $C(\epsilon)$ over unit vectors on the circle so that successive vectors $v_i$ and $v_{i+1}$ have the property that $\|v_i - v_{i+1}\| \leq \epsilon$. The size of such an $\epsilon$-cover is $N(\epsilon) = 1/\epsilon$; additionally, for any unit vector $v$ in $\mathbb{R}^2$, there exists some $v_i$ in the cover such that 
\[ f(v_i) -  \sqrt{2\log (2n/\delta)} \epsilon \leq f(v) \leq f(v_i) +  \sqrt{2\log (2n/\delta)} \epsilon \]

$f(v_i)$ is the maximum over $n$ i.i.d. standard Gaussians $N(0,1)$. From Lemma~\ref{lem:1dgaussianmax}, we know 

\[ b_n - a_n \log\log (1/\delta)\leq f(v_i) \leq b_n + a_n \log (1/\delta)\]

Now we can apply Lemma~\ref{lem:1dgaussianmax} with $\delta = \frac{\delta}{6N(\epsilon)}$ plus an union bound over the cover $C(\epsilon)$ to get that for all $v_i$ in the cover, 
\[ b_n - a_n \log\log (6/(\epsilon\delta))\leq f(v_i) \leq b_n + a_n \log (6/(\epsilon\delta))\]

For all directions $v\in \mathbb{R}^{2}$

\[ b_n - a_n \log\log (6/(\epsilon\delta)) -  \sqrt{2\log (2n/\delta)} \epsilon \leq f(v) \leq b_n + a_n \log (6/(\epsilon\delta)) +  \sqrt{2\log (2n/\delta)} \epsilon \]
Plugging in $\epsilon = \frac{a_n}{ \sqrt{2\log (2n/\delta)}}$ in the above expression we get.

\[ b_n - a_n- a_n \log\log \bigg(\frac{6\sqrt{2\log (2n/\delta)}}{a_n\delta}\bigg)  \leq f(v) \leq b_n + a_n +a_n \log \bigg(\frac{6\sqrt{2\log (2n/\delta)}}{a_n\delta}\bigg)    \]

\end{proof}
\begin{lemma}
	Consider the density: $f(t) = \frac{2}{\pi}\sqrt{1 - t^2}$ for $t \in [0, 1]$ and $f(t) = 0$ otherwise. Let $F$ be the corresponding CDF and let $U(t) = F^{-1}(1 - 1/t)$. Then, the following facts hold:
	\begin{enumerate}
		\item \[ \lim_{h \rightarrow 0} \frac{1 - F(1 - xh)}{1 - F(1 - h)} = x^{3/2} \]
		\item $1 - U(n) \geq \left( \frac{3 \pi}{4 \sqrt{2} n} \right)^{2/3}$. 
	\end{enumerate}
\label{lem:boundsuniform}
\end{lemma}

\begin{proof} The first part follows by integration by substitution and Taylor expansion of $\theta - \frac{\sin 2 \theta}{2}$ around $\theta = 0$. 

To see the first part, observe that:
	\[ 1 - F(1 - h) = \int_{1 - h}^{1} \frac{2}{\pi} \cdot \sqrt{1 - t^2} dt \]
We now calculate this integral by substitution. Let $t = \cos \theta$, then $dt = -\sin \theta d\theta$, and the limits of the integral become $\cos^{-1}(1 - h)$ to $0$. The integral becomes:
	\begin{eqnarray*}
		\int_{0}^{\cos^{-1}(1 - h)} \frac{2}{\pi} \cdot \sin^2 \theta d\theta =  \int_{0}^{\cos^{-1}(1 - h)} \frac{1}{\pi} \cdot \left(1 - \cos 2 \theta\right) d\theta = \frac{1}{\pi} \cdot \left( \theta - \frac{\sin 2 \theta}{2} \right) \Big{|}_{0}^{\cos^{-1}(1 - h)} 
	\end{eqnarray*}
A Taylor series expansion of $\sin 2 \theta$ shows that $\sin 2 \theta = 2 \theta - \frac{8 \theta^3}{3!} + o(\theta^3)$; therefore 
	\[ \theta - \frac{\sin 2 \theta}{2} = \frac{2 \theta^3}{3} + o(\theta^3), \]
	which brings the result of the integral to $\frac{4 (\cos^{-1}(1 - h))^3}{ 3 \pi} + o((\cos^{-1}(1 - h))^3)$. Observe through a Taylor expansion that 
	\[ \cos^{-1}(1 - h) = \sin^{-1}(\sqrt{h(2 - h)}) = \sqrt{h(2 - h)} - o(h(2 - h)),\]
and hence 
	\begin{equation} 
 \label{eqn:asympotic} 
 \lim_{h \rightarrow 0} \frac{1}{\pi} \cdot \left( \theta - \frac{\sin 2 \theta}{2} \right) \Big{|}_{0}^{\cos^{-1}(1 - h)} = \frac{1}{\pi} \cdot \frac{2 (2h)^{3/2}}{3}, 
 \end{equation}
from which the first part of the lemma follows. 
	For the second part, we observe that from the definition of $U(n)$, we have that $1 - U(n) = h$, where:
	\[ \int_{1-h}^{1} \frac{2}{\pi}\cdot \sqrt{1 - t^2} dt = \frac{1}{n} \]
	From \eqref{eqn:asympotic}, observe that for small enough $h$ (which corresponds to large enough $n$), the left hand side is at most $\frac{4 \sqrt{2}}{3 \pi} h^{3/2}$. This implies that $h = 1 - U(n) \geq \left(\frac{3 \pi}{4 \sqrt{2} n} \right)^{2/3}$ and the lemma follows.  
 \end{proof}
\begin{lemma}
\label{lem:1dunif}
	Consider the density: $f(t) = \frac{2}{\pi}\sqrt{1 - t^2}$ for $t \in [0, 1]$ and $f(t) = 0$ otherwise. Let $x_1, \ldots, x_n$ be $n$ drawn i.i.d from $f$ and let $X_{\max} = \max(x_1, \ldots, x_n)$. Then for $n$ large enough, with probability $\geq 1 - \delta$, we have that: 
\[ X_{\max} \leq 1, \quad X_{\max} \geq 1 - \left( \frac{3 \pi\log (2/\delta)}{4 \sqrt{2} n} \right)^{2/3} \]
 \end{lemma}
\begin{proof} 
	Observe that for this distribution, $x_F = 1$. From this, and the first part of Lemma~\ref{lem:boundsuniform}, it follows that this distribution is of the Weibull type with $\alpha = 3/2$. From the Fisher-Tippett-Gnedenko Theorem, this means that the maximum of $n$ points converges to $a_n Z + b_n$ in distribution, where $a_n = 1 - U(n)$, $b_n = 1$, and $Z$ is a reverse Weibull distributed variable with $\alpha = 3/2$. Setting $X_{\max}(n, \delta) = 1$, we get that $C(n, \delta) = 0$. 

	To calculate $c(n, \delta)$, we observe that from the second part of Lemma~\ref{lem:boundsuniform}, $a_n \geq \left( \frac{3 \pi}{4 \sqrt{2} n} \right)^{2/3}$. Additionally, if $Z$ is a reverse Weibull variable with parameter $\alpha = 3/2$, then,
	\[ \Pr(Z \leq - (\log (2/\delta)^{2/3})) = \exp(-(\log (2/\delta))^{2/3})^{3/2} = \exp(-(\log (2/\delta)) = \delta/2 \]
	Therefore, $\Pr\Big(a_n Z + b_n \leq 1 - \left( \frac{3 \pi\log (2/\delta)}{4 \sqrt{2} n} \right)^{2/3}\Big) \leq \delta/2$. We get another $\delta/2$ from the distributional convergence of the maximum of $n$ random variables to the limit for large enough $n$. 
\end{proof}

 \begin{lemma}
 \label{lem:1dunif_dir}
 Let $x_1, \ldots, x_n$ be $n$ drawn i.i.d from symmetric uniform distribution centered at zero. Then for $n$ large enough, with probability $\geq 1 -  \delta$, we have that for all directions $v\in \mathbb{R}^2$: 
\[  \max_{i\in\{1,\cdots,n\}}\{v^{\top}x_i\} \leq 1\]
\[\max_{i\in\{1,\cdots,n\}}\{v^{\top}x_i\}  \geq 1 - \left( \frac{3 \pi\log (2n/\delta)}{4 \sqrt{2} n} \right)^{2/3}-\frac{1}{n}\]
 \end{lemma}

 \begin{proof}
 Suppose $f(v) = \max_{i} v^{\top} x_i$ where $v$ is a unit vector in $\mathbb{R}^2$. Then, 
\[ f(v) - f(u) = \max_i v^{\top} x_i - \max_i u^{\top} x_i \leq \max_i (v - u)^{\top} x_i \leq \| v - u \| \cdot \max_i \|x_i\| \]
where the first step follows from definition, the second step from subtracting a smaller quantity, and the last step from the Cauchy-Schwartz inequality. 

Note that  $\max_i \|x_i\| \leq  1$, which gives us:
\[ f(v) - f(u) \leq   \cdot \| v - u \| \]

Now, we can build an $\epsilon$-cover $C(\epsilon)$ over unit vectors on the circle so that successive vectors $v_i$ and $v_{i+1}$ have the property that $\|v_i - v_{i+1}\| \leq \epsilon$. The size of such an $\epsilon$-cover is $N(\epsilon) = 1/\epsilon$; additionally, for any unit vector $v$ in $\mathbb{R}^2$, there exists some $v_i$ in the cover such that 
\[ f(v_i) -  \epsilon \leq f(v) \leq f(v_i) +   \epsilon \]

Observe that $f(v_i)$ is a maximum over $n$ i.i.d. random variables drawn from a distribution $f(t) = \frac{2}{\pi}\sqrt{1 - t^2}$ for $t \in [0, 1]$ and $f(t) = 0$ otherwise. 
Now we can apply Lemma~\ref{lem:1dunif} with $\delta = \frac{\delta}{N(\epsilon)}$ plus an union bound over the cover $C(\epsilon)$ to get that for all $v_i$ in the cover, 
\[1 - \left( \frac{3 \pi\log (2N(\epsilon)/\delta)}{4 \sqrt{2} n} \right)^{2/3}\leq f(v_i) \leq 1\]

For all directions $v\in \mathbb{R}^{2}$
\[1 - \left( \frac{3 \pi\log (2N(\epsilon)/\delta)}{4 \sqrt{2} n} \right)^{2/3} -\epsilon \leq f(v) \leq 1\]

Plugging in $\epsilon = \frac{1}{n}$ in the above expression we get.
\[1 - \left( \frac{3 \pi\log (2n/\delta)}{4 \sqrt{2} n} \right)^{2/3} -\frac{1}{n}\leq f(v) \leq 1\]
 \end{proof}

\begin{lemma}[Approximate Maximization Lemma - I]
Let $F(\alpha) = f(\alpha) + g(\alpha)$ where $g(\alpha) = \alpha u + \sqrt{1 - \alpha^2} v$, $u, v > 0$, and $f(\alpha)$ is an arbitrary function of $\alpha$ that lies in the interval $[-L, U]$. Let $\alpha_F$ be the value of $\alpha$ that maximizes $F(\alpha)$, and let $\alpha_g = \frac{u}{\sqrt{u^2 + v^2}}$ be the value of $\alpha$ that maximizes $g(\alpha)$. 

	Then, the angle between $(\alpha_F, \sqrt{1 - \alpha_F^2})$ and $(\alpha_g, \sqrt{1 - \alpha_g^2})$ is at most $\cos^{-1}\left( 1 - \frac{L + U}{\sqrt{u^2 + v^2}}\right)$. Additionally, the maximum value of $F(\alpha)$ is at least $\sqrt{u^2 + v^2} - L$. 
\label{lem:approxmax1}
\end{lemma}

\begin{proof}
	For convenience, we can do a quick change of variables -- we let $\alpha = \cos \theta$. Then $g(\theta) = u \cos \theta + v \sin \theta$, and is maximized at $\theta_g = \cos^{-1}\left(\frac{u}{\sqrt{u^2 + v^2}}\right)$. This means we can re-write $g$ as follows:
\begin{eqnarray*}
	g(\theta) & = & \sqrt{u^2 + v^2} \cdot (\cos \theta_g \cos \theta + \sin \theta_g \sin \theta) \\
	& = & \sqrt{u^2 + v^2} \cdot \cos(\theta_g - \theta) 
\end{eqnarray*}

Similarly, we can do a change of variables on $F$ and $f$ as well. Suppose the value of $\theta$ that maximizes $F$ is $\theta_F$. Then we have that:
	\[ f(\theta_g) + \sqrt{u^2 + v^2} \leq f(\theta_F) + \sqrt{u^2 + v^2} \cos (\theta_g - \theta_F) \]
Since $f(\theta_g) \geq -L$ and $f(\theta_F) \leq U$, this gives us:
	\[ -L + \sqrt{u^2 + v^2} \leq U + \sqrt{u^2 + v^2} \cos (\theta_g - \theta_F) \]
The lemma follows from simple algebra.
\end{proof}

\begin{lemma}[Approximate Maximixation Lemma - II]
	Let $F(\alpha) = f(\alpha) + g(\alpha)$ where $g(\alpha) = \alpha u - \sqrt{1 - \alpha^2} v$, $u, v > 0$, and $f(\alpha)$ is an arbitrary function of $\alpha$ that lies in the interval $[-L, U]$. Let $\alpha_F$ be the value of $\alpha$ that maximizes $F(\alpha)$, and let $\alpha_g = 1$ be the value of $\alpha$ that maximizes $g(\alpha)$. Then, $\alpha_F \geq 1 - \frac{U + L}{u + v}$.
\label{lem:approxmax2}
\end{lemma}

\begin{proof}
To show the lemma, we observe that since $f(\alpha) \in [-L, U]$,
	\[ -L + u \leq U + \alpha_F u - \sqrt{1 - \alpha_F^2} v \]
	which implies $u (1 - \alpha_F) + v \sqrt{1 - \alpha_F^2} \leq L + U$. This will hold when $1 - \alpha_F \leq \frac{U + L}{u + v}$. The lemma follows. 
\end{proof}

Denote set of points in positive class as $B$ and the set of points in negative class as $A$. We simplify the optimal solution to SVM as follows. 

\begin{equation}
    \begin{split}
         w & = \argmax_{\|w\| = 1} \inf_{x \in B} w^{\top} x - \sup_{x \in A} w^{\top} x  \\ 
         & = \argmin_{\|w\| = 1} -\inf_{x \in B} w^{\top} x + \sup_{x \in A} w^{\top} x  \\ 
         & = \argmin_{\|w\| = 1} \sup_{x \in B} -w^{\top} x + \sup_{x \in A} w^{\top} x  \\ 
         & = \argmin_{\|w\| = 1} \sup_{x \in -B} w^{\top} x + \sup_{x \in A} w^{\top} x
    \end{split}
\end{equation}

We write the classifier as $w = \alpha \hat{\mu} + \sigma \beta \hat{\psi}$, where $\hat{\mu}$ is a unit vector in the direction $\mu$, $\hat{\psi}$ is a unit vector in the direction $\psi$, $\alpha \in [-1,1]$, $\beta =  \sqrt{1 - \alpha^2}$ and $\sigma$ is either $+1$ or $-1$.  

Define the set $A_{\mu} = \{ x+ \mu, \forall x \in A \}$, and the set $-B_{\mu} = \{ x+ \mu, \forall x \in -B \}$.

 With this, and some algebraic simplification the SVM optimization problem becomes:
\[ \alpha^* = \argmin_{\alpha \in [-1, 1], \sigma \in \{-1, 1\}} \sup_{x \in A_{\mu}} (\alpha \hat{\mu} + \sigma \beta \hat{\psi})^{\top} (x - \mu) + \sup_{x \in -B_{\mu}} (\alpha \hat{\mu} + \sigma \beta \hat{\psi})^{\top} (x - \mu) \]
\weibulgroup*


\begin{proof}

We write $w^* = \alpha^* \hat{\mu} + \sigma \beta^* \hat{\psi}$. Recall that $\alpha^*$ is a solution to:
\[ \alpha^* = \argmin_{\alpha \in [-1, 1], \sigma \in \{-1, 1\}} \sup_{x \in A_{\mu}} (\alpha \hat{\mu} + \sigma \beta \hat{\psi})^{\top} (x - \mu) + \sup_{x \in -B_{\mu}} (\alpha \hat{\mu} + \sigma \beta \hat{\psi})^{\top} (x - \mu) \]
where $\beta = \sqrt{1 - \alpha^2}$. We next consider a further split of the positive class into the majority and minority groups -- $A^M_{\mu}$ and $A^m_{\mu}$. 

\[ \sup_{x \in A_{\mu}} v^{\top} x = \max \left( \sup_{x \in A^M_{\mu}} v^{\top} x, \sup_{x \in A^m_{\mu}} v^{\top} x \right), \]

Define sets $A_{\mu, \psi}^{M} = -\psi + A^M_{\mu}$  and $A_{\mu, \psi}^{m} = \psi + A^m_{\mu}$. We can write 
\begin{eqnarray*}
	&&  \sup_{x \in A_{\mu}} (\alpha \hat{\mu} + \sigma \beta \hat{\psi})^{\top} (x - \mu) \\
	& = & \max \left(  \sup_{x \in A^M_{\mu}} (\alpha \hat{\mu} + \sigma \beta \hat{\psi})^{\top} (x - \mu), \sup_{x \in A^m_{\mu}} (\alpha \hat{\mu} + \sigma \beta \hat{\psi})^{\top} (x - \mu) \right) \\
& = & \max \left( \sup_{x \in A^M_{\mu, \psi}} (\alpha \hat{\mu} + \sigma \beta \hat{\psi})^{\top} x - \alpha \|\mu\| + \sigma \beta \|\psi\|, \sup_{x \in A^m_{\mu, \psi}} (\alpha \hat{\mu} + \sigma \beta \hat{\psi})^{\top} x - \alpha \|\mu\| - \sigma \beta \|\psi\| \right),
\end{eqnarray*}
A similar expression will hold for $B_{\mu}$. Define sets $B_{\mu, \psi}^{M} = \psi + B^M_{\mu}$  and $B_{\mu, \psi}^{m} = \psi + B^m_{\mu}$. 
We can write 
\[\sup_{x \in -B_{\mu}} (\alpha \hat{\mu} + \sigma \beta \hat{\psi})^{\top} (x - \mu) \]
\[\max \left( \sup_{x \in -B^M_{\mu, \psi}} (\alpha \hat{\mu} + \sigma \beta \hat{\psi})^{\top} x - \alpha \|\mu\| + \sigma \beta \|\psi\|, \sup_{x \in -B^m_{\mu, \psi}} (\alpha \hat{\mu} + \sigma \beta \hat{\psi})^{\top} x - \alpha \|\mu\| - \sigma \beta \|\psi\| \right)\]
Define 
\[ f_1(\alpha)  = \sup_{x \in A^M_{\mu, \psi}} (\alpha \hat{\mu} + \sigma \beta \hat{\psi})^{\top} x, \quad f_2(\alpha)  = \sup_{x \in A^m_{\mu, \psi}} (\alpha \hat{\mu} + \sigma \beta \hat{\psi})^{\top} x \]
\[ f_3(\alpha)  = \sup_{x \in -B^M_{\mu, \psi}} (\alpha \hat{\mu} + \sigma \beta \hat{\psi})^{\top} x, \quad f_4(\alpha)  = \sup_{x \in -B^m_{\mu, \psi}} (\alpha \hat{\mu} + \sigma \beta \hat{\psi})^{\top} x \]

We split up the SVM objective, and begin with two cases:

\paragraph{Case 1: $\sigma = 1$.}  Here, the SVM objective becomes: 

\begin{equation*}
    \begin{split}
      &  F(\alpha) = \min_{\alpha}\bigg\{\max( f_1(\alpha) - \alpha\|\mu\| + \beta \|\psi\|, f_2(\alpha) - \alpha \|\mu\| - \beta \|\psi\|) + \\
      &   + \max( f_3(\alpha) - \alpha \|\mu\| + \beta \|\psi\|, f_4(\alpha) - \alpha \|\mu\| - \beta\| \psi\|) \bigg\}
    \end{split}
\end{equation*}

Recall the concentration condition holds for $D(0)$. Since the size of the majority group and minority group both $\geq n_0$, we obtain 
\begin{equation}
f_1(\alpha), f_3(\alpha) \in  [X_{\max}(p,\delta,2) - c(p, \delta, 2),   X_{\max}(p,\delta,2) + C(p, \delta,2)], 
\end{equation}
and also:
\begin{equation}
f_2(\alpha), f_4(\alpha) \in [X_{\max}(m, \delta, 2) - c(m, \delta, 2),   X_{\max}(m, \delta, 2) + C(m, \delta, 2)]
\end{equation}

We now look at two possible cases for $\alpha$ to determine what the inside maximum will look like. The first case is for large $\beta$ -- where $$\beta \geq \frac{C(m, \delta, 2) - c(p, \delta, 2) - (X_{max}(p,\delta,2) - X_{\max}(m, \delta, 2))}{2\|\psi\|}$$ and the objective simplifies to. 


\[ F(\alpha) = \min_{\alpha} f_1(\alpha) + f_3(\alpha) - 2 \alpha \|\mu\| + 2 \beta \|\psi\|, \]

	The solution to this, from Lemma~\ref{lem:approxmax2} is $\alpha^* \geq 1 - \frac{C(p,\delta,2) + c(p,\delta,2)}{\|\mu\| + \|\psi\|}$, which converges to the relevant solution as $p \rightarrow \infty$.
	The other case is where $\beta$ is small -- namely, $\beta \leq \frac{C(m, \delta, 2) - c(p,\delta,2)}{2\|\psi\|}$ (here we use the fact that $X_{\max}(p,\delta,2)\geq X_{\max}(m, \delta, 2)$. Since $\alpha^2 + \beta^2 = 1$, this implies that $\alpha$ is now close to $1$ -- specifically, $\alpha \geq \sqrt{1 - \left(\frac{C(m, \delta, 2) - c(p,\delta,2)}{2\|\psi\|})\right)^2} \rightarrow 1$ as $p \rightarrow \infty$ by the conditions of the theorem. This means that in both cases, the inner maximum is achieved when $\alpha$ is close to $1$. 

\paragraph{Case 2: $\sigma = -1$.} In this case, the SVM objective becomes:
\begin{equation*}
    \begin{split}
       &  F(\alpha) =   \min_{\alpha} \bigg\{ \max( f_1(\alpha) - \alpha \|\mu\| - \beta \|\psi\|, f_2(\alpha) - \alpha \|\mu\| + \beta \|\psi\|) +  \\
       &\max( f_3(\alpha) - \alpha \|\mu\| - \beta \|\psi\|, f_4(\alpha) - \alpha \|\mu\| + \beta \|\psi\|) \bigg\} 
    \end{split}
\end{equation*}

	We again do a case by case analysis. We say that $\beta$ is large if $$\beta \geq \beta_{\mathsf{th}} = \frac{X_{\max}(p,\delta,2) - X_{\max}(m, \delta, 2) + C(p,\delta,2) + c(m, \delta, 2)}{2\|\psi\|}$$  In this case, the SVM objective becomes:
	\[ F(\alpha) = f_2(\alpha) + f_4(\alpha) - 2 \|\mu\| \alpha + 2 \beta \|\psi\| \]
Since $\beta\geq \beta_{\mathsf{th}}$, $-\sqrt{1-\beta_{\mathsf{th}}^2}\leq \alpha \leq \sqrt{1-\beta_{\mathsf{th}}^2}$.
From Lemma~\ref{lem:approxmax2_n}, the solution to this is $\alpha \geq \sqrt{1-\beta_{\mathsf{th}}^2} - \frac{C(m, \delta, 2) + c(m, \delta, 2)}{\|\mu\| + \|\psi\|} \rightarrow 1$ as $m \rightarrow \infty$.
	The other case is for small $\beta$, where $$\beta \leq \frac{X_{\max}(p,\delta,2) - X_{\max}(m, \delta, 2) + C(p,\delta,2) + c(m, \delta, 2)}{2\|\psi\|}$$ 
 
 Since $\alpha^2 + \beta^2 = 1$, here $\alpha$ by definition satisfies $$\alpha \geq \sqrt{1 -  \left(\frac{X_{\max}(p,\delta,2) - X_{\max}(m, \delta, 2) + C(p,\delta,2) + c(m, \delta, 2)}{4\|\psi\|}\right)^2} \rightarrow 1$$ 
 
 As $p, m \rightarrow \infty$ from the condition in the theorem $\alpha \rightarrow 1$.  This means that in all four cases, the optimum is achieved when $\alpha$ is close to $1$. The theorem follows.
\end{proof}

\spugumbgroup* 



\begin{proof} 
Recall $w^* = \alpha^* \hat{\mu} + \sigma \beta^* \hat{\psi}$, where $\alpha^*$ is a solution to:
\[ \alpha^* = \argmin_{\alpha \in [-1, 1], \sigma \in \{-1, 1\}} \sup_{x \in A_{\mu}} (\alpha \hat{\mu} + \sigma \beta \hat{\psi})^{\top} (x - \mu) + \sup_{x \in -B_{\mu}} (\alpha \hat{\mu} + \sigma \beta \hat{\psi})^{\top} (x - \mu) \]
where $\beta = \sqrt{1 - \alpha^2}$.  We follow a similar strategy as the previous proof of Theorem \ref{thm: invariant_weibull} and look at two cases -- $\sigma = 1$ and $-1$.

\paragraph{Case 1: $\sigma = 1$.}  Here, the SVM objective becomes: 
\begin{equation*}
    \begin{split}
       &  F(\alpha) = \min_{\alpha} \bigg\{ \max( f_1(\alpha) - \alpha \|\mu\| + \beta \|\psi\|, f_2(\alpha) - \alpha \|\mu\| - \beta\|\psi\|) +  \\ 
        &  + \max( f_3(\alpha) - \alpha \|\mu\| + \beta \|\psi\|, f_4(\alpha) - \alpha \|\mu\| - \beta\|\psi\|)\bigg\}
    \end{split}
\end{equation*}

	where $f_1(\alpha) =  \sup_{x \in A^M_{\mu, \psi}} (\alpha \hat{\mu} + \sigma \beta \hat{\psi})^{\top} x$, $f_2(\alpha) =   \sup_{x \in A^m_{\mu, \psi}} (\alpha \hat{\mu} + \sigma \beta \hat{\psi})^{\top} x$, $f_3(\alpha) = \sup_{x \in -B^M_{\mu, \psi}} (\alpha \hat{\mu} + \sigma \beta \hat{\psi})^{\top} x$, $f_2(\alpha) =   \sup_{x \in -B^m_{\mu, \psi}} (\alpha \hat{\mu} + \sigma \beta \hat{\psi})^{\top} x$. From conditions on the majority and the minority class, and the Concentration Condition, with probability $1-4\delta$
\begin{equation}
f_1(\alpha), f_3(\alpha) \in  [X_{\max}(p,\delta,2) - c(p,\delta,2),   X_{\max}(p,\delta,2) + C(p,\delta,2)], 
\end{equation}
and also:
\begin{equation}
f_2(\alpha), f_4(\alpha) \in [X_{\max}(m, \delta, 2) - c(m, \delta, 2),   X_{\max}(m, \delta, 2) + C(m, \delta, 2)]
\end{equation}

	Observe that from the conditions of the theorem, the first terms will dominate for all values of $\alpha$, and hence the SVM objective will become:
	\[ F(\alpha) = \min_{\alpha} f_1(\alpha) + f_3(\alpha) - 2 \alpha \|\mu \|+ 2 \beta \|\psi\|, \]
	From Lemma~\ref{lem:approxmax2}, the optimal solution $\alpha^* \geq 1 - \frac{C(p,\delta,2) + c(p,\delta,2)}{ \|\mu\| + \|\psi\| }$, with a lower bound on the optimal value $2 X_{\max}(p,\delta,2) - 2 c(p,\delta,2) - 2 \|\mu\|$.

\paragraph{Case 2: $\sigma = -1$.} Here, the SVM objective becomes:
\begin{equation*}
    \begin{split}
      &  F(\alpha) = \min_{\alpha} \bigg\{ \max( f_1(\alpha) - \alpha \|\mu\| - \beta \|\psi\|, f_2(\alpha) - \alpha \|\mu\| + \beta\|\psi\|) + \\
      & \max( f_3(\alpha) - \alpha \|\mu\| - \beta\|\psi\|, f_4(\alpha) - \alpha \|\mu\| + \beta \|\psi\|) \bigg\}
    \end{split}
\end{equation*}

	This time, from the conditions of the theorem, the first terms will dominate the maximum for all values of $\alpha$, and hence the objective will become:
	\[ F(\alpha) = \min_{\alpha} f_1(\alpha) + f_3(\alpha) - 2 \alpha \|\mu\| - 2 \beta \|\psi\| \]

	From Lemma~\ref{lem:approxmax1}, the optimal solution vector $(\alpha, \sqrt{1 - \alpha^2})$ will be close to the spurious solution vector $( \frac{
 |\mu\|}{\sqrt{\|\mu\|^2 + \|\psi\|^2}}, \frac{\|a\|}{\sqrt{\|\mu\|^2 + \|\psi\|^2}})$, with the angle being at most $\cos^{-1}(1 - \frac{C(p,\delta,2) + c(p,\delta,2)}{\sqrt{\|\mu\|^2 + \|\psi\|^2}})$. The optimal solution value will be at most $2 X_{\max}(p,\delta,2) + 2 C(p,\delta,2) - 2 \sqrt{\|\mu\|^2 + \|\psi\|^2}$. From the conditions of the theorem, this value is lower than the lower bound on the optimal solution for $\sigma = 1$, and hence the optimal SVM solution will be achieved at this value. Thus the result follows, from the additional condition that $C(p,\delta,2) + c(p,\delta,2) \rightarrow 0$ as $p \rightarrow \infty$.  
 The comparison of the worst group erros is carried out in Lemma \ref{lem:wge_comp_appendix}.



\end{proof}

\paragraph{Illustrating Theorem \ref{thm: spurious_gumbel} using Gaussians.}

\begin{itemize}
\item For sufficiently large $p$, $\big(X_{\max}(p,\delta,2)-X_{\max}(p^{\tau},\delta,2)\big)^2$ gets arbitarily close  to  $(\sqrt{2\log \frac{p}{\delta})} - \sqrt{2\tau \log \frac{p}{\delta}})^2$, which when simplified gives 
$$(\sqrt{2\log \frac{p}{\delta}} - \sqrt{2\tau \log \frac{p}{\delta}})^2 = 2\log \frac{p}{\delta}(1+\tau-2\sqrt{\tau}) $$
For sufficiently large $p$, $(2\|\psi\| + c(p,\delta,2) + C(p^{\tau},\delta,2))^2$ gets arbitrarily close to $ \log (\frac{p}{\delta})^{\kappa}$. Now if $\kappa < 2(1+\tau-2\sqrt{\tau})$ the condition $X_{\max}(p,\delta,2)-X_{\max}(p^{\tau},\delta,2)\big) \geq 2\|\psi\| + c(p,\delta,2) + C(p^{\tau},\delta,2) $ is satisfied. 
\item The objective value $F(\alpha)$ when $\alpha=1$ is at most $2 X_{\max}(p,\delta,2) + 2 C(p,\delta,2) - 2 \|\mu\|$. This expression simplifies to $ 2\big((\sqrt{2}-\sqrt{3})\log p +  C(p,\delta,2)\big)$. For a sufficiently large $p$, the objective is negative. This implies that the data is perfectly separable in the invariant feature. 
\item We also need to check $\sqrt{\|\mu\|^2 + \|\psi\|^2} - \|\mu\| >  C(p,\delta,2) + c(p,\delta,2)$. The expression in the LHS simplifies $\sqrt{\|\mu\|^2 + \|\psi\|^2} - \|\mu\| = \sqrt{\log \frac{p}{\delta}}(\sqrt{3+\frac{\kappa}{4}} - \sqrt{3})$. The expression in the LHS is an increasing function of $p$ and grows to infinity and the RHS decreases to zero. For sufficiently large $p$, the condition has to be satisfied.  Finally, the ratio spurious feature to invariant feature is $\sqrt{\kappa/12}$.
\end{itemize}

\begin{lemma}[Data balancing helps improve worst group error under heavy tails] \label{lem:wge_comp_appendix}
Consider the same set of assumptions as in Theorem \ref{thm: spurious_gumbel}. With probability at least $ 1 - 12 \delta$, $\wge(\theta^{*}_{\text{ss}})<\wge(\theta^{*}_{\text{erm}})$.
\end{lemma}

\begin{proof}
We start with analyzing the worst group error for the standard SVM solution, i.e., without any data balancing. Recall 
 \[ -b^* = \frac{1}{2} ( \sup_{x \in A} (\alpha \hat{\mu} + \sigma \beta \hat{\psi})^{\top})  x - \sup_{x \in -B} (\alpha \hat{\mu} +\sigma \beta \hat{\psi})^{\top}) x) \]
Let us try to bound $-b^{*}$. From the concentration condition and the fact that $p\geq n_0$, with probability at least $1-\delta$, the first term above $\sup_{x \in A} (\alpha \hat{\mu} + \sigma\beta \hat{\psi})^{\top}x)$ lies in $$[X_{\max}(p+m,\delta,2)-c(p+m,\delta,2), X_{\max}(p+m,\delta)+C(p+m,\delta,2)]$$
The second term also lies in $$[X_{\max}(p+m,\delta,2)-c(p+m,\delta,2), X_{\max}(p+m,\delta,2)+C(p+m,\delta,2)]$$
As a result, with probability $1-2\delta$, $-b^{*}$ is in
$$\bigg[\frac{- c(p+m, \delta,2) - C(p+m, \delta,2))}{2}, \frac{c(p+m,\delta,2) + C(p+m,\delta,2)}{2}\bigg]$$
We denote $a_{\min} = \frac{- c(p+m, \delta,2) - C(p+m, \delta,2))}{2}$ and $a_{\max} = \frac{ c(p+m, \delta,2) + C(p+m, \delta,2))}{2}$ 

Consider a classifier $w^{\top} x +b $. We write the error for different groups. $\mathsf{Err}_{y,a}$ is the error for the group $g=(y,a)$. 

\begin{equation} 
\begin{split}
 \mathsf{Err}_{1,-1} &= \mathbb{P}(w^{\top}X + b\leq 0 | X \sim D(\mu-\psi)) \\  
  \mathsf{Err}_{1,-1} &=\mathbb{P}(w^{\top}(\mu -\psi + \tilde{X}) - b\leq 0 | \tilde{X} \sim D(0)) \\  
& =\mathbb{P}(w^{\top}X \leq w^{\top} (\psi-\mu)  - b| \tilde{X} \sim D(0))\\
& =F_W( w^{\top} (\psi-\mu)  - b) 
\end{split}
\end{equation}
Denote $w^{\top}X=W$, $F_W$ is the CDF of W. Also, observe that since $\tilde{X}$ is spherically symmetric, the distribution $w^{\top}X$ is the same as distribution of another $w^{',\top}\tilde{X}$, where $\|w\|=\|w^{'}\|=1$. We now plug in the value of $-b^{*}$ for the max-margin classifier to arrive at the bounds for 
the error for each of the groups.  We write 
\begin{equation}
\begin{split}
& F_W( w^{\top} (\psi-\mu)  + a_{\min})  \leq   \mathsf{Err}_{1,-1} \leq F_W( w^{\top} (\psi-\mu)  + a_{\max})  \\
& F_W( -\alpha \|\mu\| +\sigma \beta \|\psi\|  + a_{\min}) \leq \mathsf{Err}_{1,-1} \leq F_W( -\alpha \|\mu\| +\sigma \beta \|\psi\|  + a_{\max})
 \end{split}
\end{equation}

Similarly, we write
\begin{equation}
\begin{split}
& F_W( -\alpha \|\mu\| -\sigma \beta \|\psi\|  + a_{\min}) \leq \mathsf{Err}_{1,1} \leq F_W( -\alpha \|\mu\| -\sigma \beta \|\psi\|  + a_{\max})
 \end{split}
\end{equation}

Observe that as $p$ grows, $a_{\max}$ and $a_{\min}$ converge to zero (from Assumption \ref{assm: concentration interval}). Also as $p$ grows, from Theorem \ref{thm: spurious_gumbel}, we know that the optimal $\alpha$ approaches $\frac{\|\mu\|}{\sqrt{\|\mu\|^2 + \|\psi\|^2}}$ and $\sigma=-1$. 

As  a result, we can say that with probability at least $1-6\delta$, $\mathsf{Err}_{1,-1}$ and $\mathsf{Err}_{1,-1}$ approach the following quantities.

$$\mathsf{Err}_{1,-1} \rightarrow F_W\bigg( (-\|\mu\|^2 - \|\psi\|^2)\frac{1}{\sqrt{\|\mu\|^2 + \|\psi\|^2}} \bigg) $$

$$\mathsf{Err}_{1,1} \rightarrow F_W\bigg( (-\|\mu\|^2 + \|\psi\|^2)\frac{1}{\sqrt{\|\mu\|^2 + \|\psi\|^2}} \bigg) $$

We now turn our attention to the optimal SVM solution achieved after balancing the data. In this case, we throw the data out so all groups have same size $m$. 
In this case, the optimal $b^{*}$ lies in the interval 
$a_{\min} = \frac{- c(2m, \delta, 2) - C(2m, \delta, 2))}{2}$ and $a_{\max} = \frac{ c(2m, \delta, 2) + C(2m, \delta,2))}{2}$ 

Observe that as $m$ grows, $a_{\max}$ and $a_{\min}$ converge to zero (from Assumption \ref{assm: concentration interval}). Also as $m$ grows, from Theorem \ref{thm: invariant_weibull}, we know that the optimal $\alpha$ approaches $1$. We denote the error for a group $y,a$ under balancing as $\mathsf{Err}_{y,a}^{\mathsf{bal}}$.

As  a result, we can say that with probability at least $1-6\delta$, $\mathsf{Err}_{1,-1}^{\mathsf{bal}}$ and $\mathsf{Err}_{1,-1}^{\mathsf{bal}}$ approach the following quantities.

$$\mathsf{Err}_{1,-1}^{\mathsf{bal}} \rightarrow F_W(-\|\mu\|) $$

$$\mathsf{Err}_{1,1}^{\mathsf{bal}} \rightarrow F_W(-\|\mu\|) $$

We compare the error achieved by the two approaches. With probability $1-12\delta$ (We need to account for the joint probability that for imbalanced case the optimal solution is the spurious one and under the balanced case the optimal solution is the invariant one. From union bound it follows that at least one of them does not occur with probability at most $12\delta$).

$$\mathsf{Err}_{1,1}^{\mathsf{bal}} \rightarrow F_W(-\|\mu\|) $$ and 
$$\mathsf{Err}_{1,1} \rightarrow F_W(-\|\mu\|) $$

We want to show 
\begin{equation}
\begin{split}
  & \mathsf{Err}_{1,1}^{\mathsf{bal}}< \mathsf{Err}_{1,1} \\
\end{split}
\end{equation}
To show the above, is equivalent to showing 

\[(-\|\mu\|^2 + \|\psi\|^2)\frac{1}{\sqrt{\|\mu\|^2 + \|\psi\|^2}} > -\|\mu\|\]
Suppose $\|\psi\|\geq \|\mu\|$, then the LHS is non negative and RHS is non positive. Thus the claim is true in that case. 

Suppose $\|\psi\| < \|\mu\|$, then both the LHS and RHS are negative. As a result, we want to show that 
\begin{equation}
\begin{split}
 &  \frac{ (-\|\mu\|^2 + \|\psi\|^2)^2}{\|\mu\|^2 + \|\psi\|^2} < \|\mu\|^2 
\end{split}
\end{equation}
Further simplification yields
\begin{equation}
\begin{split}
 &  \frac{ (-\|\mu\|^2 + \|\psi\|^2)^2}{\|\mu\|^2 + \|\psi\|^2} < \|\mu\|^2 \iff \|\psi\|^2(\|\psi\|^2 - 3\|\mu\|^2) <0 
\end{split}
\end{equation}

Since  $\|\psi\| < \|\mu\|$, the above condition is satisfied. 
\end{proof}

\subsection{Higher-dimensional case}
\label{sec:proof_groups_high_d}
\begin{lemma}\label{lem:gaussiannorm_new}
Let $x_1, \ldots, x_n$ be vectors in $\mathbb{R}^d$ drawn i.i.d from $N(0, I_d)$. Then, with probability $\geq 1 - \delta$, 
\[ \max_i \|x_i\| \leq \sqrt{d + 2\sqrt{d \log (n/\delta)} + 2  \log (n/\delta)} \]
\end{lemma}
\begin{proof}
Observe that 
\begin{equation}
    \begin{split}
    &    \|x_i\| \leq \sqrt{d + 2\sqrt{d \log (n/\delta)} + 2  \log (n/\delta)}, \forall i \in \{1, \cdots, n\} \\ 
    &    \|x_i\|^2 \leq d + 2\sqrt{d \log (n/\delta)} + 2  \log (n/\delta), \forall i \in \{1, \cdots, n\}
    \end{split}
\end{equation}
We bound the probability of the above 
\begin{equation}
    \begin{split}
        & \mathbb{P}(\max_i \|x_i\|^2 \leq d + 2\sqrt{d \log (n/\delta)} + 2  \log (n/\delta) ) = \\ & 1- \mathbb{P}(\max_i \|x_i\|^2 \geq   d + 2\sqrt{d \log (n/\delta)} + 2  \log (n/\delta)) \geq 1 - n \mathbb{P}(\|x_i\|^2    \geq \\
        &  d + 2\sqrt{d \log (n/\delta)} + 2  \log (n/\delta)) = 1- n e^{-\log(n/\delta)} = 1-\delta
    \end{split}
\end{equation}
For the last step in the above, we leverage the fact that $\|x_i\|^2 $ follows the Chi-square $\mathcal{X}^{2}(d)$ distribution and use the tail bound in Lemma 1 from \cite{laurent2000adaptive}.
\end{proof}

Define $Q(n,\delta,d) = \sqrt{d + 2\sqrt{d \log (n/\delta)} + 2  \log (n/\delta)} $

\begin{lemma}\label{lem:1dgaussianmax_new}
Let $x_1, \ldots, x_n$ be $n$ i.i.d unit Gaussians with covariance $I_d$.  Then for $n$ large enough, with probability $\geq 1 -  \delta$, we have that for all directions $v\in \mathbb{R}^d$: 
\[  \max_{i\in\{1,\cdots,n\}}\{v^{\top}x_i\} \leq  b_n + a_n +a_n d\log \bigg(\frac{\frac{12 Q(n,\delta/2, d)}{a_n}+6}{\delta}\bigg) ,\] 
\[\max_{i\in\{1,\cdots,n\}}\{v^{\top}x_i\}  \geq b_n - a_n- a_n \log\bigg(d\log \bigg(\frac{\frac{12 Q(n,\delta/2, d)}{a_n}+6}{\delta}\bigg)\bigg)  \]
where $a_n$ and $b_n$ are the constants in the Fisher-Tippett-Gnedenko theorem when applied to Gaussians.
\end{lemma}
\begin{proof} 
Suppose $f(v) = \max_{i} v^{\top} x_i$ where $v$ is a unit vector in $\mathbb{R}^d$. Then, 
\[ f(v) - f(u) = \max_i v^{\top} x_i - \max_i u^{\top} x_i \leq \max_i (v - u)^{\top} x_i \leq \| v - u \| \cdot \max_i \|x_i\| \]
where the first step follows from definition, the second step from subtracting a smaller quantity, and the last step from the Cauchy-Schwartz inequality.

From Lemma~\ref{lem:gaussiannorm_new}, with probability $\geq 1 - \delta/2$, $\max_i \|x_i\| \leq  Q(n,\delta/2,d) $, which gives us:
\[ f(v) - f(u) \leq  Q(n,\delta/2,d) \cdot \| v - u \| \]

Now, we can build an $\epsilon$-cover $C(\epsilon)$ over unit vectors on the circle so that successive vectors $v_i$ and $v_{i+1}$ have the property that $\|v_i - v_{i+1}\| \leq \epsilon$. The size of such an $\epsilon$-cover is $N(\epsilon) \leq  (\frac{2}{\epsilon}+1)^d$ \footnote{\url{https://www.stat.berkeley.edu/~bartlett/courses/2013spring-stat210b/notes/12notes.pdf}}; additionally, for any unit vector $v$ in $\mathbb{R}^d$, there exists some $v_i$ in the cover such that 
\[ f(v_i) -  Q(n,\delta/2,d) \epsilon \leq f(v) \leq f(v_i) +  Q(n,\delta/2,d)\epsilon \]

$f(v_i)$ is the maximum over $n$ i.i.d. standard Gaussians $N(0,1)$. From Lemma~\ref{lem:1dgaussianmax}, we know 

\[ b_n - a_n \log\log (1/\delta)\leq f(v_i) \leq b_n + a_n \log (1/\delta)\]

Now we can apply Lemma~\ref{lem:1dgaussianmax} with $\delta = \frac{\delta}{6N(\epsilon)}$ plus an union bound over the cover $C(\epsilon)$ to get that for all $v_i$ in the cover, 
\[ b_n - a_n \log\log (6N(\epsilon)/(\delta))\leq f(v_i) \leq b_n + a_n \log (6N(\epsilon)/(\delta))\]

\[ b_n - a_n \log\bigg(d\log \bigg(\frac{(\frac{12}{\epsilon}+6)}{\delta}\bigg)\bigg)\leq f(v_i) \leq b_n + a_n d\log \bigg(\frac{(\frac{12}{\epsilon}+6)}{\delta}\bigg)\]

For all directions $v\in \mathbb{R}^{d}$

\[ b_n - a_n \log\bigg(d\log \bigg(\frac{(\frac{12}{\epsilon}+6)}{\delta}\bigg)\bigg) -  Q(n,\delta/2,d)\epsilon \leq f(v) \leq b_n + a_n d\log \bigg(\frac{(\frac{12}{\epsilon}+6)}{\delta}\bigg)+  Q(n,\delta/2,d)\epsilon\]
Plugging in $\epsilon = \frac{a_n}{ Q(n,\delta/2, d)}$ in the above expression we get.

\[ b_n - a_n- a_n \log\bigg(d\log \bigg(\frac{\frac{12 Q(n,\delta/2, d)}{a_n}+6}{\delta}\bigg)\bigg)  \leq f(v) \leq b_n + a_n +a_n d\log \bigg(\frac{\frac{12 Q(n,\delta/2, d)}{a_n}+6}{\delta}\bigg)  \]

\end{proof}

\begin{lemma}
	Consider the density: $f(t) = \frac{3}{2}(1 - t^2)$ for $t \in [0, 1]$ and $f(t) = 0$ otherwise. Let $F$ be the corresponding CDF and let $U(t) = F^{-1}(1 - 1/t)$. Then, the following facts hold:
	\begin{enumerate}
		\item \[ \lim_{h \rightarrow 0} \frac{1 - F(1 - xh)}{1 - F(1 - h)} = x \]
		\item $1 - U(n) \geq \frac{2}{3n}$. 
	\end{enumerate}
\label{lem:boundsuniform_new}
\end{lemma}

\begin{proof}


To see the first part, observe that:
	\[ 1 - F(1 - h) = \int_{1 - h}^{1} \frac{3}{2}(1 - t^2) = \frac{3}{2} (h - \frac{h^3-3h^2 + 3h}{3}  \]
\[ \lim_{h \rightarrow 0} \frac{1 - F(1 - xh)}{1 - F(1 - h)} = \frac{xh - \frac{(xh)^3-3(xh)^2 + 3xh}{3} }{h - \frac{h^3-3h^2 + 3h}{3} } = x \]

For the second part, we observe that from the definition of $U(n)$, we have that $1 - U(n) = h$, where:
	\[ \int_{1-h}^{1} \frac{3}{2}(1 - t^2) = \frac{1}{n} \]
	Observe that for small enough $h$ (which corresponds to large enough $n$), the left hand side is at most $\frac{3}{2}h$. This implies that $h = 1 - U(n) \geq \frac{2}{3n}$ and the lemma follows.  
 \end{proof}
 \begin{lemma}
\label{lem:1dunif_new}
	Consider the density: $f(t) = \frac{3}{2}(1 - t^2)$ for $t \in [0, 1]$ and $f(t) = 0$ otherwise. Let $x_1, \ldots, x_n$ be $n$ drawn i.i.d from $f$ and let $X_{\max} = \max(x_1, \ldots, x_n)$. Then for $n$ large enough, with probability $\geq 1 - \delta$, we have that: 
\[ X_{\max} \leq 1, \quad X_{\max} \geq 1 - \left( \frac{2\log (2/\delta)}{3n} \right) \]
 \end{lemma}
\begin{proof} 
	Observe that for this distribution, $x_F = 1$. From this, and the first part of Lemma~\ref{lem:boundsuniform}, it follows that this distribution is of the Weibull type with $\alpha = 1$. From the Fisher-Tippett-Gnedenko Theorem, this means that the maximum of $n$ points converges to $a_n Z + b_n$ in distribution, where $a_n = 1 - U(n)$, $b_n = 1$, and $Z$ is a reverse Weibull distributed variable with $\alpha = 1$. Setting $X_{\max}(n, \delta) = 1$, we get that $C(n, \delta) = 0$. 

	To calculate $c(n, \delta)$, we observe that from the second part of Lemma~\ref{lem:boundsuniform_new}, $a_n \geq \frac{2}{3n}$. Additionally, if $Z$ is a reverse Weibull variable with parameter $\alpha = 1$, then,
	\[ \Pr(Z \leq - (\log (2/\delta))) = \exp(-(\log (2/\delta))) =\delta/2 \]
	Therefore, $\Pr\Big(a_n Z + b_n \leq 1 - \left( \frac{2\log (2/\delta)}{3n} \right)\Big) \leq \delta/2$. We get another $\delta/2$ from the distributional convergence of the maximum of $n$ random variables to the limit for large enough $n$. 
\end{proof}

 \begin{lemma}
 Let $x_1, \ldots, x_n$ be $n$ drawn i.i.d from symmetric uniform distribution centered at zero in $\mathbb{R}^3$. Then for $n$ large enough, with probability $\geq 1 -  \delta$, we have that for all directions $v\in \mathbb{R}^2$: 
\[  \max_{i\in\{1,\cdots,n\}}\{v^{\top}x_i\} \leq 1\]
\[\max_{i\in\{1,\cdots,n\}}\{v^{\top}x_i\}  \geq 1 - \frac{6\log (4n+2)/\delta)}{3n} -\frac{1}{n}\]
 \end{lemma}

 \begin{proof}
 Suppose $f(v) = \max_{i} v^{\top} x_i$ where $v$ is a unit vector in $\mathbb{R}^2$. Then, 
\[ f(v) - f(u) = \max_i v^{\top} x_i - \max_i u^{\top} x_i \leq \max_i (v - u)^{\top} x_i \leq \| v - u \| \cdot \max_i \|x_i\| \]
where the first step follows from definition, the second step from subtracting a smaller quantity, and the last step from the Cauchy-Schwartz inequality. 

Note that  $\max_i \|x_i\| \leq  1$, which gives us:
\[ f(v) - f(u) \leq   \cdot \| v - u \| \]

Now, we can build an $\epsilon$-cover $C(\epsilon)$ over unit vectors on the circle so that successive vectors $v_i$ and $v_{i+1}$ have the property that $\|v_i - v_{i+1}\| \leq \epsilon$. The size of such an $\epsilon$-cover is $N(\epsilon) = 1/\epsilon$; additionally, for any unit vector $v$ in $\mathbb{R}^2$, there exists some $v_i$ in the cover such that 
\[ f(v_i) -  \epsilon \leq f(v) \leq f(v_i) +   \epsilon \]

Observe that $f(v_i)$ is a maximum over $n$ i.i.d. random variables drawn from a distribution $f(t) = \frac{3}{2}(1 - t^2)$ for $t \in [0, 1]$ and $f(t) = 0$ otherwise. 
Now we can apply Lemma~\ref{lem:1dunif_new} with $\delta = \frac{\delta}{N(\epsilon)}$ plus an union bound over the cover $C(\epsilon)$ to get that for all $v_i$ in the cover, 
\[1 -\frac{2\log (2N(\epsilon)/\delta)}{3n}\leq f(v_i) \leq 1\]

For all directions $v\in \mathbb{R}^{3}$
\[1 - \frac{2\log (2N(\epsilon)/\delta)}{3n} -\epsilon \leq f(v) \leq 1\]

Plugging in $\epsilon = \frac{1}{n}$ in the above expression we get.
\[1 - \frac{6\log (4n+2)/\delta)}{3n} -\frac{1}{n}\leq f(v) \leq 1\]
 \end{proof}

\begin{lemma}[Approximate Maximization Lemma - I]
Let $F(\alpha) = f(\alpha) + g(\alpha)$ where $g(\alpha) = \alpha u + \sqrt{\eta^2 - \alpha^2} v$, $u, v > 0$, and $f(\alpha)$ is an arbitrary function of $\alpha$ that lies in the interval $[-L, U]$. Let $\alpha_F$ be the value of $\alpha$ that maximizes $F(\alpha)$, and let $\alpha_g = \eta\frac{u}{\sqrt{u^2 + v^2}}$ be the value of $\alpha$ that maximizes $g(\alpha)$. Then, the angle between $(\alpha_F, \sqrt{\eta^2 - \alpha_F^2})$ and $(\alpha_g, \sqrt{1 - \alpha_g^2})$ is at most $\cos^{-1}\left( 1 - \frac{L + U}{\sqrt{u^2 + v^2}}\right)$. Additionally, the maximum value of $F(\alpha)$ is at least $\eta\sqrt{u^2 + v^2} - L$. 
\label{lem:approxmax1_n}
\end{lemma}
\begin{proof}
	For convenience, we can do a quick change of variables -- we let $\alpha = \eta\cos \theta$. Then $g(\theta) = \eta (u \cos \theta + v \sin \theta)$, and is maximized at $\theta_g = \cos^{-1}\left(\frac{u}{\sqrt{u^2 + v^2}}\right)$. This means we can re-write $g$ as follows:
\begin{eqnarray*}
	g(\theta) & = &  \eta \sqrt{u^2 + v^2} \cdot (\cos \theta_g \cos \theta + \sin \theta_g \sin \theta) \\
	& = & \eta\sqrt{u^2 + v^2} \cdot \cos(\theta_g - \theta) 
\end{eqnarray*}
Similarly, we can do a change of variables on $F$ and $f$ as well. Suppose the value of $\theta$ that maximizes $F$ is $\theta_F$. Then we have that:
	\[ f(\theta_g) + \eta\sqrt{u^2 + v^2} \leq f(\theta_F) + \eta \sqrt{u^2 + v^2} \cos (\theta_g - \theta_F) \]
Since $f(\theta_g) \geq -L$ and $f(\theta_F) \leq U$, this gives us:
	\[ -L + \eta \sqrt{u^2 + v^2} \leq U + \eta \sqrt{u^2 + v^2} \cos (\theta_g - \theta_F) \]
The lemma follows from simple algebra.
\end{proof}

\begin{lemma}[Approximate Maximixation Lemma - II]
	Let $F(\alpha) = f(\alpha) + g(\alpha)$ where $g(\alpha) = \alpha u - \sqrt{\eta^2 - \alpha^2} v$, $u, v > 0$, and $f(\alpha)$ is an arbitrary function of $\alpha$ that lies in the interval $[-L, U]$. Let $\alpha_F$ be the value of $\alpha \in [-\eta, \eta]$ that maximizes $F(\alpha)$, and let $\alpha_g = \eta$ be the value of $\alpha$ that maximizes $g(\alpha)$. Then, $\alpha_F \geq \eta - \frac{U + L}{u + v}$.
\label{lem:approxmax2_n}
\end{lemma}
\begin{proof}
To show the lemma, we observe that since $f(\alpha) \in [-L, U]$,
	\[ -L + \eta u \leq U + \alpha_F u - \sqrt{\eta^2 - \alpha_F^2} v \]
	which implies $u (\eta - \alpha_F) + v \sqrt{\eta^2 - \alpha_F^2} \leq L + U$. This will hold when $\eta - \alpha_F \leq \frac{U + L}{u + v}$. The lemma follows. 
\end{proof}

Recall that the SVM solution is stated as 
\[ w^{*} = \argmax_{\|w\| = 1} \inf_{x \in B} w^{\top} x - \sup_{x \in A} w^{\top} x   \]
We rewrite $w = \alpha \hat{\mu} +  \beta \hat{\psi} + \gamma^{\top}\hat{\Gamma}$, where $\hat{\mu}$, $\hat{\psi}$ denote unit vectors along $\mu$, $\psi$ respectively. $\hat{\Gamma} \in \mathbb{R}^{d\times d-2}$ is a matrix of $d-2$ vectors that span the subspace orthogonal to the subspace spanned $\mu$ and $\psi$. 
We introduce an additional parameter $\eta \in [0,1]$ and define a set as follows $\mathcal{S}_{\eta} = \{(\alpha, \beta, \gamma),\; \alpha^2+\beta^2 = \eta^2, \|\gamma\|^2= 1-\eta^2\}$.  Note that $\mathcal{S} = \cup_{n\in [0,1]} \mathcal{S}_{\eta}$ is the set of all vectors of norm $1$. We divide the standard SVM optimization into an optimization over the set $\mathcal{S}_{\eta}$ and then choosing the best $\eta$. 
\begin{equation}
    \begin{split}
        \alpha^*(\eta) &= \argmin_{\alpha \in [-\eta, \eta], \sigma \in \{-1, 1\}, \|\gamma\||=\sqrt{1-\eta^2}} \sup_{x \in A_{\mu}} \Big\{(\alpha \hat{\mu} +  \sigma \beta \hat{\psi} + \gamma^{\top}\hat{\Gamma})^{\top}  (x - \mu)\Big\}\\ 
        &  +   \sup_{x \in -B_{\mu}} \Big\{(\alpha \hat{\mu} + \sigma \beta \hat{\psi} + \gamma^{\top}\hat{\Gamma}) ^{\top} (x - \mu) \Big\},
    \end{split}
\end{equation}
where $\beta = \sqrt{1-\alpha^2}$. We compare the SVM objective for all $\eta$ and pick the set of $\alpha^*(\eta)$ that lead to the optimal value. 

\gumbelgroupd*
\begin{proof} 
We fix an $\eta$ and  write the optimal solution for the $\eta$ as 
\begin{equation}
     \begin{split}
        \alpha^*(\eta) &= \argmin_{\alpha \in [-\eta, \eta], \sigma \in \{-1, 1\}, \|\gamma\||=\sqrt{1-\eta^2}} \sup_{x \in A_{\mu}} \Big\{(\alpha \hat{\mu} +  \sigma \beta \hat{\psi} + \gamma^{\top}\hat{\Gamma})^{\top}  (x - \mu)\Big\}\\ 
        &  +   \sup_{x \in -B_{\mu}} \Big\{(\alpha \hat{\mu} + \sigma \beta \hat{\psi} + \gamma^{\top}\hat{\Gamma}) ^{\top} (x - \mu) \Big\},
    \end{split}
\end{equation}
where $\beta = \sqrt{1 - \alpha^2}$. We next consider a further split of the positive class into the majority and minority groups -- $A^M_{\mu}$ and $A^m_{\mu}$.  This means that:  
\[ \sup_{x \in A_{\mu}} \boldsymbol{v}^{\top} x = \max \left( \sup_{x \in A^M_{\mu}} \boldsymbol{v}^{\top} x, \sup_{x \in A^m_{\mu}} \boldsymbol{v}^{\top} x \right), \]
and hence, we can write:
\begin{equation*}
\begin{split}
	&  \sup_{x \in A_{\mu}} ((\alpha \hat{\mu} + \sigma \beta \hat{\psi} + \gamma^{\top}\hat{\Gamma})^{\top}  (x - \mu)\\
	& =  \max \left(  \sup_{x \in A^M_{\mu}} (\alpha \hat{\mu} + \sigma \beta \hat{\psi} + \gamma^{\top}\hat{\Gamma})^{\top}  (x - \mu),  \sup_{x \in A^m_{\mu}} ((\alpha \hat{\mu} + \sigma \beta \hat{\psi} + \gamma^{\top}\hat{\Gamma})^{\top}  (x - \mu) \right) \\
= & \max \Bigg( \sup_{x \in A^M_{\mu, \psi}} (\alpha \hat{\mu} + \sigma \beta \hat{\psi} + \gamma^{\top}\hat{\Gamma})^{\top}  x - \alpha \|\mu\| + \sigma \beta \|\psi\|, \\ 
&  \sup_{x \in A^m_{\mu, \psi}} (\alpha \hat{\mu} + \sigma \beta \hat{\psi} + \gamma^{\top}\hat{\Gamma})^{\top}  x - \alpha \|\mu\| - \sigma \beta \|\psi\| \Bigg),
\end{split}
\end{equation*}
and a similar expression will hold for $B_{\mu}$. 
We look at two cases -- $\sigma = 1$ and $-1$.
\paragraph{Case 1: $\sigma = 1$.}  Here, the SVM objective becomes: 
\begin{equation*}
\begin{split}
&    \min_{\alpha, \gamma}\bigg\{ \max( f_1(\alpha, \gamma) - \alpha \|\mu\| + \beta \|\psi\|, f_2(\alpha, \gamma) - \alpha \|\mu\| - \beta \|\psi\|) +  \\ 
 & \max( f_3(\alpha, \gamma) - \alpha \|\mu\| + \beta \|\psi\|, f_4(\alpha, \gamma) - \alpha \|\mu\| - \beta\|\psi\|) \bigg\}
\end{split}
\end{equation*}
	where $f_1(\alpha, \gamma) =  \sup_{x \in A^M_{\mu, \psi}} (\alpha \hat{\mu} + \sigma \beta \hat{\psi}  + \gamma^{\top}\hat{\Gamma})^{\top} x$, $f_2(\alpha, \gamma) =   \sup_{x \in A^m_{\mu, \psi}} (\alpha \hat{\mu} + \sigma \beta \hat{\psi}+ + \gamma^{\top}\hat{\Gamma})^{\top} x$, $f_3(\alpha, \gamma) = \sup_{x \in -B^M_{\mu, \psi}} (\alpha \hat{\mu} + \sigma \beta \hat{\psi} + \gamma^{\top}\hat{\Gamma})^{\top} x$, $f_4(\alpha, \gamma) =   \sup_{x \in -B^m_{\mu, \psi}} (\alpha \hat{\mu} + \sigma \beta \hat{\psi} + \gamma^{\top}\hat{\Gamma})^{\top} x$. From conditions on the majority and the minority class, and the Concentration Condition with probability $\geq 1-4\delta$, 
\begin{equation}
f_1(\alpha, \gamma), f_3(\alpha, \gamma) \in  [X_{\max}(p, \delta, d) - c(p, \delta,d),   X_{\max}(p, \delta, d) + C(p, \delta, d)], 
\end{equation}
and also:
\begin{equation} 
f_2(\alpha, \gamma), f_4(\alpha, \gamma) \in [X_{\max}(m, \delta, d) - c(m, \delta, d),   X_{\max}(m, \delta, d) + C(m, \delta, d)]
\end{equation}
	Observe that from the conditions of the theorem, the first terms will dominate for all values of $\alpha$, and hence the SVM objective will become:
	\begin{equation}
	    \begin{split}& \min_{\alpha \in [-\eta, \eta], \|\gamma\|=\sqrt{1-\eta^2}} f_1(\alpha, \gamma) + f_3(\alpha, \gamma) - 2 \alpha \|\mu\| + 2 \beta \|\psi\|=  \\
	      &\min_{\alpha \in [-\eta, \eta]} h(\alpha, \eta) - 2 \alpha \|\mu\| + 2 \beta \|\psi\|
	    \end{split}
	\end{equation}
	where $h(\alpha,\eta) = \min_{\|\gamma\|=\sqrt{1-\eta^2}} f_1(\alpha, \gamma) + f_3(\alpha, \gamma)$. 
	Observe that  $h(\alpha, \eta) \in [2X_{\max}(p, \delta, d) - 2c(p, \delta,d),   2X_{\max}(p, \delta, d) + 2C(p, \delta, d)]$.
	From Lemma~\ref{lem:approxmax2_n}, the optimal solution $\alpha^* \geq \eta - \frac{C(p, \delta, d) + c(p, \delta,d)}{ \|\mu\| + \|\psi\| }$, with an optimal value is lower bounded by $2 X_{\max}(p, \delta, d) - 2 c(p, \delta,d) - 2 \eta\|\mu\|$. 
\paragraph{Case 2: $\sigma = -1$.} Here, the SVM objective becomes: 
\begin{equation*}
    \begin{split}
&        \min_{\alpha\in [-\eta, \eta], \|\gamma\|=\sqrt{1-\eta^2}} \bigg\{ \max( f_1(\alpha, \gamma) - \alpha \|\mu\| - \beta \|\psi\|, f_2(\alpha, \gamma) - \alpha \mu + \beta \psi ) +  \\ 
&         \max( f_3(\alpha,\gamma) - \alpha \|\mu\| - \beta \|\psi\|, f_4(\alpha, \gamma) - \alpha \|\mu\| + \beta \|\psi\|) \bigg\}
    \end{split}
\end{equation*}

	This time, from the conditions of the theorem, the first terms will dominate the maximum for all values of $\alpha$, and hence the objective will become:
	\begin{equation}
	\begin{split}
	  & \min_{\alpha\in [-\eta,\eta], \|\gamma\|=\sqrt{1-\eta^2}} f_1(\alpha,\gamma) + f_3(\alpha,\gamma) - 2 \alpha \|\mu\| - 2 \beta\|\psi\| =\\
	  & \min_{\alpha\in [-\eta,\eta]} h(\alpha, \eta) - 2 \alpha \|\mu\| - 2 \beta \|\psi\|
	 \end{split}
	\end{equation}
Observe that  $h(\alpha, \eta) \in [2X_{\max}(p, \delta, d) - 2c(p, \delta,d),   2X_{\max}(p, \delta, d) + 2C(p, \delta, d)]$. From Lemma~\ref{lem:approxmax1_n}, the optimal solution  $(\alpha, \sqrt{1 - \alpha^2})$ will be close to the spurious solution vector $( \eta \frac{\|\mu\|}{\sqrt{\|\mu\|^2 + \|\psi\|^2}}, \eta\frac{\|\psi\|}{\sqrt{\|\mu\|^2 + \|\psi\|^2}})$, with the angle being at most $\cos^{-1}(1 - \frac{C(p, \delta,d) + c(p, \delta, d)}{\sqrt{\|\mu\|^2 + \|\psi\|^2}})$. The optimal solution value will be at most $2 X_{\max}(p, \delta, d) + 2 C(p, \delta, d) - 2 \eta \sqrt{\|\mu\|^2 + \|\psi\|^2}$. From the conditions of the theorem, this value is lower than the lower bound on the optimal solution for $\sigma = 1$, and hence the optimal SVM solution for a fixed $\eta$ will be achieved at this value. 
We now compare the lower bound for the optimal $\eta$ with the upper bound of the optimal value at $\eta=1$. 
\begin{equation}
\begin{split}
 & 2 X_{\max}(p, \delta, d) - 2 C (p, \delta, d) - 2 \eta \sqrt{\|\mu\|^2 + \|\psi\|^2} \leq  2 X_{\max}(p, \delta, d) + 2 C(p, \delta, d) - 2 \sqrt{\|\mu\|^2 + \|\psi\|^2}  \\ 
 & \eta \geq 1 - \frac{C(p,\delta,d)}{\sqrt{\|\mu\|^2+\|\psi\|^2}}
\end{split}
\end{equation}
The result follows from the additional condition that $C(p, \delta,d) + c(p, \delta, d) \rightarrow 0$ as $p \rightarrow \infty$.  The comparison of worst group errors is carried out in Lemma \ref{lem:wge_comparison_high_dim}
\end{proof}
\begin{lemma}[Data balancing helps improve worst group error under heavy tails]\label{lem:wge_comparison_high_dim}
Consider the same set of assumptions as Theorem \ref{thm: spurious_d}. With probability at least $ 1 - 12 \delta$, $\wge(\theta^{*}_{\text{ss}})<\wge(\theta^{*}_{\text{erm}})$.
\end{lemma}

\begin{proof}
We start with analyzing the worst group error for the standard SVM solution, i.e., without any data balancing. Recall 
 \[ -b^* = \frac{1}{2} ( \sup_{x \in A} (\alpha \hat{\mu} + \sigma \beta \hat{\psi} + \gamma^{\top}\hat{\Gamma})^{\top}  x - \sup_{x \in -B} (\alpha \hat{\mu} + \sigma \beta \hat{\psi} + \gamma^{\top}\hat{\Gamma}) x) \]
Let us try to bound $-b^{*}$. From the concentration condition and the fact that $p\geq n_0$, with probability at least $1-\delta$, the first term above $\sup_{x \in A} (\alpha \hat{\mu} + \sigma\beta \hat{\psi} + \gamma^{\top}\hat{\Gamma})^{\top}x)$ lies in $$[X_{\max}(p+m,\delta,d)-c(p+m,\delta,d), X_{\max}(p+m,\delta,d)+C(p+m,\delta,d)]$$
The second term also lies in $$[X_{\max}(p+m,\delta,d)-c(p+m,\delta,d), X_{\max}(p+m,\delta,d)+C(p+m,\delta,d)]$$
As a result, with probability $1-2\delta$, $-b^{*}$ is in
$$\bigg[\frac{- c(p+m, \delta,d) - C(p+m, \delta,d))}{2}, \frac{c(p+m,\delta,d) + C(p+m,\delta,d)}{2}\bigg]$$
We denote $a_{\min} = \frac{- c(p+m, \delta, d) - C(p+m, \delta, d))}{2}$ and $a_{\max} = \frac{ c(p+m, \delta, d) + C(p+m, \delta, d))}{2}$ 

Consider a classifier $w^{\top} x +b $. We write the error for different groups. $\mathsf{Err}_{y,a}$ is the error for the group $g=(y,a)$. 

\begin{equation} 
\begin{split}
 \mathsf{Err}_{1,-1} &= \mathbb{P}(w^{\top}X + b\leq 0 | X \sim D(\mu-\psi)) \\  
  \mathsf{Err}_{1,-1} &=\mathbb{P}(w^{\top}(\mu -\psi + \tilde{X}) + b\leq 0 | \tilde{X} \sim D(0)) \\  
& =\mathbb{P}(w^{\top}X \leq w^{\top} (\psi-\mu)  - b| \tilde{X} \sim D(0))\\
& =F_W( w^{\top} (\psi-\mu)  + b) 
\end{split}
\end{equation}
Denote $w^{\top}X=W$, $F_W$ is the CDF of W. Also, observe that since $X$ is spherically symmetric, the distribution $w^{\top}X$ is the same as distribution of another $w^{',\top}X$, where $\|w\|=\|w^{'}\|=1$. We now plug in the value of $b^{*}$ for the max-margin classifier to arrive at the bounds for 
the error for each of the groups.  We write 
\begin{equation}
\begin{split}
& F_W( w^{\top} (\psi-\mu)  + a_{\min})  \leq   \mathsf{Err}_{1,-1} \leq F_W( w^{\top} (\psi-\mu)  + a_{\max})  \\
& F_W( -\alpha \|\mu\| +\sigma \beta \|\psi\|  + a_{\min}) \leq \mathsf{Err}_{1,-1} \leq F_W( -\alpha \|\mu\| +\sigma \beta \|\psi\|  + a_{\max})
 \end{split}
\end{equation}

Similarly, we write
\begin{equation}
\begin{split}
& F_W( -\alpha \|\mu\| -\sigma \beta \|\psi\|  + a_{\min}) \leq \mathsf{Err}_{1,1} \leq F_W( -\alpha \|\mu\| -\sigma \beta \|\psi\|  + a_{\max})
 \end{split}
\end{equation}

Observe that as $p$ grows, $a_{\max}$ and $a_{\min}$ converge to zero (from Assumption \ref{assm: concentration interval}). Also as $p$ grows, from Theorem \ref{thm: spurious_d}, we know that the optimal $\alpha$ approaches $\frac{\|\mu\|}{\sqrt{\|\mu\|^2 + \|\psi\|^2}}$ and $\sigma=-1$ and $\gamma=0$. 

As  a result, we can say that with probability at least $1-6\delta$, $\mathsf{Err}_{1,-1}$ and $\mathsf{Err}_{1,-1}$ approach the following quantities.

$$\mathsf{Err}_{1,-1} \rightarrow F_W\bigg( (-\|\mu\|^2 - \|\psi\|^2)\frac{1}{\sqrt{\|\mu\|^2 + \|\psi\|^2}} \bigg) $$

$$\mathsf{Err}_{1,1} \rightarrow F_W\bigg( (-\|\mu\|^2 + \|\psi\|^2)\frac{1}{\sqrt{\|\mu\|^2 + \|\psi\|^2}} \bigg) $$

We now turn our attention to the optimal SVM solution achieved after balancing the data. In this case, we throw the data out so all groups have same size $m$. 
In this case, the optimal $b^{*}$ lies in the interval 
$a_{\min} = \frac{- c(2m, \delta, d) - C(2m, \delta, d))}{2}$ and $a_{\max} = \frac{ c(2m, \delta, d) + C(2m, \delta, d))}{2}$ 

Observe that as $m$ grows, $a_{\max}$ and $a_{\min}$ converge to zero (from Assumption \ref{assm: concentration interval}). Also as $m$ grows, from Theorem \ref{thm: invariant_d}, we know that the optimal $\alpha$ approaches $1$. We denote the error for a group $y,a$ under balancing as $\mathsf{Err}_{y,a}^{\mathsf{bal}}$.

As  a result, we can say that with probability at least $1-6\delta$, $\mathsf{Err}_{1,-1}^{\mathsf{bal}}$ and $\mathsf{Err}_{1,-1}^{\mathsf{bal}}$ approach the following quantities.

$$\mathsf{Err}_{1,-1}^{\mathsf{bal}} \rightarrow F_W(-\|\mu\|) $$

$$\mathsf{Err}_{1,1}^{\mathsf{bal}} \rightarrow F_W(-\|\mu\|) $$

We compare the error achieved by the two approaches. With probability $1-12\delta$ (We need to account for the joint probability that for imbalanced case the optimal solution is the spurious one and under the balanced case the optimal solution is the invariant one. From union bound it follows that at least one of them does not occur with probability at most $12\delta$).

$$\mathsf{Err}_{1,1}^{\mathsf{bal}} \rightarrow F_W(-\|\mu\|) $$ and 
$$\mathsf{Err}_{1,1} \rightarrow F_W(-\|\mu\|) $$

We want to show 
\begin{equation}
\begin{split}
  & \mathsf{Err}_{1,1}^{\mathsf{bal}}< \mathsf{Err}_{1,1} \\
\end{split}
\end{equation}
To show the above, is equivalent to showing 

\[(-\|\mu\|^2 + \|\psi\|^2)\frac{1}{\sqrt{\|\mu\|^2 + \|\psi\|^2}} > -\|\mu\|\]
Suppose $\|\psi\|\geq \|\mu\|$, then the LHS is non negative and RHS is negative. Thus the claim is true in that case. 

Suppose $\|\psi\| < \|\mu\|$, then both the LHS and RHS are negative. As a result, we want to show that 
\begin{equation}
\begin{split}
 &  \frac{ (-\|\mu\|^2 + \|\psi\|^2)^2}{\|\mu\|^2 + \|\psi\|^2} < \|\mu\|^2 
\end{split}
\end{equation}
Further simplification yields
\begin{equation}
\begin{split}
 &  \frac{ (-\|\mu\|^2 + \|\psi\|^2)^2}{\|\mu\|^2 + \|\psi\|^2} < \|\mu\|^2 \iff \|\psi\|^2(\|\psi\|^2 - 3\|\mu\|^2) <0 
\end{split}
\end{equation}

Since  $\|\psi\| < \|\mu\|$, the above condition is satisfied. 
\end{proof}

\weibullgroupd*
\begin{proof}
We fix an $\eta$ and  write the optimal solution for the $\eta$ as 
\begin{equation}
    \begin{split}
      &  \alpha^*(\eta) = \\ &\argmin_{\alpha \in [-\eta, \eta], \sigma \in \{-1, 1\}, \|\gamma\||=\sqrt{1-\eta^2}} \sup_{x \in A_{\mu}} \Big\{(\alpha \hat{\mu} + \sigma \beta \hat{\psi} + \gamma^{\top}\hat{\Gamma})^{\top}  (x - \mu) + \\ & \sup_{x \in -B_{\mu}} (\alpha \hat{\mu} + \sigma \beta \hat{\psi} + \gamma^{\top}\hat{\Gamma})^{\top}  (x - \mu) \Big\}
    \end{split}
\end{equation}
where $\beta = \sqrt{1 - \alpha^2}$. We next consider a further split of the positive class into the majority and minority groups -- $A^M_{\mu}$ and $A^m_{\mu}$.  This means that:  
\[ \sup_{x \in A_{\mu}} \boldsymbol{v}^{\top} x = \max \left( \sup_{x \in A^M_{\mu}} \boldsymbol{v}^{\top} x, \sup_{x \in A^m_{\mu}} \boldsymbol{v}^{\top} x \right), \]
and hence, we can write:
\begin{equation*}
\begin{split}
    &  \sup_{x \in A_{\mu}} ((\alpha \hat{\mu} + \sigma \beta \hat{\psi} + \gamma^{\top}\hat{\Gamma})^{\top}  (x - \mu)\\
    = & \max \left(  \sup_{x \in A^M_{\mu}} ((\alpha \hat{\mu} + \sigma \beta \hat{\psi} + \gamma^{\top}\hat{\Gamma})^{\top}  (x - \mu),  \sup_{x \in A^m_{\mu}} ((\alpha \hat{\mu} + \sigma \beta \hat{\psi} + \gamma^{\top}\hat{\Gamma})^{\top}  (x - \mu) \right) \\
    & \max \Bigg( \sup_{x \in A^M_{\mu, \psi}} ((\alpha \hat{\mu} + \sigma \beta \hat{\psi} + \gamma^{\top}\hat{\Gamma})^{\top}  x - \alpha \|\mu\| + \\ 
    & \sigma \beta \|\psi\|, \sup_{x \in A^m_{\mu, \psi}} ((\alpha \hat{\mu} + \sigma \beta \hat{\psi} + \gamma^{\top}\hat{\Gamma})^{\top}  x - \alpha \|\mu\| - \sigma \beta \|\psi\| \Bigg)
\end{split}
\end{equation*}
and a similar expression will hold for $B_{\mu}$. 
 We look at two cases -- $\sigma = 1$ and $-1$.
\paragraph{Case 1: $\sigma = 1$.}  Here, the SVM objective becomes: 
\begin{equation*}
\begin{split}
&    \min_{\alpha, \gamma} \bigg\{ \max( f_1(\alpha, \gamma) - \alpha \|\mu\| + \beta \|\psi\|, f_2(\alpha, \gamma) - \alpha \|\mu\| - \beta \|\psi\|) + \\ 
&    \max( f_3(\alpha, \gamma) - \alpha \|\mu\| + \beta\|\psi\|, f_4(\alpha, \gamma) - \alpha \|\mu\| - \beta\|\psi\|)\bigg\}
\end{split}    
\end{equation*}
	where $f_1(\alpha, \gamma) =  \sup_{x \in A^M_{\mu, \psi}} (\alpha \hat{\mu} + \sigma \beta \hat{\psi}  + \gamma^{\top}\hat{\Gamma})^{\top} x$, $f_2(\alpha, \gamma) =   \sup_{x \in A^m_{\mu, \psi}} (\alpha \hat{\mu} + \sigma \beta \hat{\psi}+ + \gamma^{\top}\hat{\Gamma})^{\top} x$, $f_3(\alpha, \gamma) = \sup_{x \in -B^M_{\mu, \psi}} (\alpha \hat{\mu} + \sigma \beta \hat{\psi} + \gamma^{\top}\hat{\Gamma})^{\top} x$, $f_4(\alpha, \gamma) =   \sup_{x \in -B^m_{\mu, \psi}} (\alpha \hat{\mu} + \sigma \beta \hat{\psi} + \gamma^{\top}\hat{\Gamma})^{\top} x$. From conditions on the majority and the minority class, and the Concentration Condition with probability $\geq 1-4\delta$, 
\begin{equation}
f_1(\alpha, \gamma), f_3(\alpha, \gamma) \in  [X_{\max}(p, \delta, d) - c(p, \delta,d),   X_{\max}(p, \delta, d) + C(p, \delta, d)], 
\end{equation}
and also:
\begin{equation}
f_2(\alpha, \gamma), f_4(\alpha, \gamma) \in [X_{\max}(m, \delta, d) - c(m, \delta, d),   X_{\max}(m, \delta, d) + C(m, \delta, d)]
\end{equation}

We now look at two possible cases for $\alpha$ to determine what the inside maximum will look like. The first case is for large $\beta$ -- where $$\beta \geq \frac{C(m, \delta, d) - c(n, \delta,d) - (X_{max}(p, \delta, d) - X_{\max}(m, \delta,d))}{2\|\psi\|}$$ and the objective simplifies to. 


\[ F(\alpha) = \min_{\alpha} f_1(\alpha, \gamma) + f_3(\alpha, \gamma) - 2 \alpha \|\mu\| + 2 \beta \|\psi\|, \]

\begin{equation}
	    \begin{split}& \min_{\alpha \in [-\eta, \eta], \|\gamma\|=\sqrt{1-\eta^2}} f_1(\alpha, \gamma) + f_3(\alpha, \gamma) - 2 \alpha \|\mu\| + 2 \beta \|\psi\|=  \\
	      &\min_{\alpha \in [-\eta, \eta]} h(\alpha, \eta) - 2 \alpha \|\mu\| + 2 \beta \|\psi\|
	    \end{split}
	\end{equation}
	where $h(\alpha,\eta) = \min_{\|\gamma\|=\sqrt{1-\eta^2}} f_1(\alpha, \gamma) + f_3(\alpha, \gamma)$. 
	Observe that  $h(\alpha, \eta) \in [2X_{\max}(p, \delta, d) - 2c(p, \delta,d),   2X_{\max}(p, \delta, d) + 2C(p, \delta, d)]$.
	From Lemma~\ref{lem:approxmax2_n}, the optimal solution $\alpha^* \geq \eta - \frac{C(p, \delta, d) + c(p, \delta,d)}{ \|\mu\| + \|\psi\| }$, with an optimal value is lower bounded by $2 X_{\max}(p, \delta, d) - 2 c(p, \delta,d) - 2 \eta\|\mu\|$. 
 
	The other case is where $\beta$ is small -- namely, $\beta \leq \frac{C(m, \delta, d) - c(p, \delta,d)}{2\|\psi\|}$ (here we use the fact that $X_{\max}(p,\delta,d)\geq X_{\max}(m,\delta,d)$. Since $\alpha^2 + \beta^2 = \eta^2$, this implies that $\alpha$ is now close to $1$ -- specifically, $\alpha \geq \sqrt{\eta^2 - \left(\frac{C(m, \delta,d) - c(p, \delta,d)}{2\|\psi\|})\right)^2} \rightarrow 1$ as $p \rightarrow \infty$ by the conditions of the theorem. This means that in both cases, the inner maximum is achieved when $\alpha$ is close to $\eta$. 

\paragraph{Case 2: $\sigma = -1$.} In this case, the SVM objective becomes:

\begin{equation*}
    \begin{split}
       &  F(\alpha) = \min_{\alpha, \gamma} \bigg\{ \max( f_1(\alpha,\gamma) - \alpha \|\mu\| - \beta \|\psi\|, f_2(\alpha, \gamma) - \alpha \|\mu\| + \beta \|\psi\|) + \\ 
       &  \max( f_3(\alpha, \gamma) - \alpha \|\mu\| - \beta \|\psi\|, f_4(\alpha, \gamma) - \alpha \|\mu\| + \beta \|\psi\|)\bigg\}
    \end{split}
\end{equation*}


We again do a case by case analysis. We say that $\beta$ is large if $$\beta \geq \beta_{\mathsf{th}}= \frac{X_{\max}(p, \delta, d) - X_{\max}(m, \delta, d) + C(p, \delta, d) + c(m, \delta, d)}{2\|\psi\|}$$  In this case, the SVM objective becomes:
	\[ F(\alpha) = f_2(\alpha, \gamma) + f_4(\alpha, \gamma) - 2 \|\mu\| \alpha + 2 \beta \|\psi\| \]
 As a result of the above, $- \sqrt{1-\beta_{\mathsf{th}}^2}\leq \alpha \leq \sqrt{1-\beta_{\mathsf{th}}^2}$.  We divide the analysis into two cases. 
 Case 1. $\eta<\sqrt{1-\beta_{\mathsf{th}}^2}$ From Lemma~\ref{lem:approxmax2_n}, the solution to this is $\alpha \geq \eta - \frac{C(m, \delta, d) + c(m, \delta, d)}{\|\mu\| + \|\psi\|} \rightarrow \eta$ as $m \rightarrow \infty$.
 Case 2. $\eta\geq \sqrt{1-\beta_{\mathsf{th}}^2}$ From Lemma~\ref{lem:approxmax2_n}, the solution to this is $\alpha \geq \sqrt{1-\beta_{\mathsf{th}}^2} - \frac{C(m, \delta, d) + c(m, \delta, d)}{\|\mu\| + \|\psi\|} \rightarrow 1$ as $m \rightarrow \infty$.

 Now we are left with analyzing the setting when $\beta$ is small, where $$\beta \leq \frac{X_{\max}(p, \delta, d) - X_{\max}(m, \delta , d) + C(p, \delta, d) + c(m, \delta, d)}{2\|\psi\|}$$ 
 
 Since $\alpha^2 + \beta^2 = \eta^2$, here $\alpha$ by definition satisfies $$\alpha \geq \sqrt{\eta^2 -  \left(\frac{X_{\max}(p, \delta,d) - X_{\max}(m, \delta,d) + C(p, \delta,d) + c(m, \delta,d)}{4\|\psi\|}\right)^2} \rightarrow \eta$$ 
  This means that in all four cases, the inner maximum is achieved when $\alpha$ is close to $\eta$. We now compare the lower bound for optimal value achieved by $\eta$ with the upper bound on the objective for $\alpha=\sqrt{1-\beta_{\mathsf{th}}^2}$ to show that $\eta$ approaches $1$ as $m\rightarrow\infty$. 

For a fixed $\eta$, $\alpha$ takes value arbitrarily close to $\eta$. A lower bound on the SVM objective when $\alpha\in [-\eta,\eta]$ and $\eta \leq \sqrt{1-\beta_{\mathsf{th}}^2}$ is 
\begin{equation*}
    \begin{split}
&        2X_{\max}(m,\delta, d)-2c(m,\delta,d) -2\eta\|\mu\| - \beta\|\psi\| =  \\ 
&        2X_{\max}(m,\delta, d)-2c(m,\delta,d) -2\eta\|\mu\| - \sqrt{2\eta \frac{C(m, \delta, d) + c(m, \delta, d)}{\|\mu\| + \|\psi\|} }\|\psi\|\leq \\
& 2X_{\max}(m,\delta, d)-2c(m,\delta,d) -2\eta\|\mu\| - \sqrt{2C(m, \delta, d) + c(m, \delta, d)\|\psi\|}
    \end{split}
\end{equation*}

  When $\alpha=\sqrt{1-\beta_{\mathsf{th}}^2}$ the SVM objective can be at most 
    $$2X_{\max}(m,\delta, d)+2C(m,\delta,d) -2\sqrt{1-\beta_{\mathsf{th}}^2}\|\mu\| + 2\beta_{\mathsf{th}}\|\psi\|$$

Comparing the above to the lower bound on the SVM objective we get
    \begin{equation}
    \begin{split}
&   2X_{\max}(m,\delta, d)-2c(m,\delta,d) -2\eta\|\mu\| - \sqrt{2 C(m, \delta, d) + c(m, \delta, d)\|\psi\|}< \\
&2X_{\max}(m,\delta, d)+2C(m,\delta,d) -2\sqrt{1-\beta_{\mathsf{th}}^2}\|\mu\| + 2\beta_{\mathsf{th}}\|\psi\|  \\ 
 &  \eta \geq  \sqrt{1-\beta_{\mathsf{th}}^2} - \frac{C(m,\delta,d) + c(m,\delta,d) +   \beta_{\mathsf{th}}\|\psi\|+  \frac{1}{2}\sqrt{2 C(m, \delta, d) + c(m, \delta, d)\|\psi\|}}{\|\mu\|}
    \end{split}
    \end{equation}
    Owing to the conditions in the theorem, as $m\rightarrow \infty$, $\eta\rightarrow 1$. 
\end{proof}

\begin{theorem}\label{thm:wge}
The ideal invariant classifier $\hat{\mu}$ achieves the minimum worst group error.
\end{theorem}
\begin{proof}
    We write down the error expressions for the four groups as follows. We consider a general classifier $w^{\top}x +b$, where $\|w\|=1$. 
\begin{equation}
\begin{split}
&      \mathsf{Err}_{1,-1} = \mathbb{P}(w^{\top}X + b\leq 0 | X \sim D(\mu-\psi)) \\ 
&      \mathsf{Err}_{1,-1} = F_W(w^{\top} (\psi-\mu)  - b)
\end{split}
\end{equation}

\begin{equation}
\begin{split}
&      \mathsf{Err}_{1,1} = \mathbb{P}(w^{\top}X + b\leq 0 | X \sim D(\mu+\psi)) \\ 
&      \mathsf{Err}_{1,1} = F_W(w^{\top} (-\psi-\mu)  - b)
\end{split}
\end{equation}

\begin{equation}
\begin{split}
&      \mathsf{Err}_{-1,1} = \mathbb{P}(w^{\top}X + b\geq 0 | X \sim D(-\mu+\psi)) \\ 
& =\mathbb{P}(w^{\top}(\tilde{X} -\mu + \psi) + b\geq 0 | \tilde{X} \sim D(0)) \\ 
& =\mathbb{P}(w^{\top}\tilde{X} \geq w^{\top}(\mu - \psi) - b | \tilde{X} \sim D(0)) \\ 
& =F_W( w^{\top}( \psi -\mu) + b ) \\ 
\end{split}
\end{equation}
In the above simplification, we exploit the fact that $\tilde{X}$ is symmetric and as a result the distribution of $\tilde{X}$ is same as $-\tilde{X}$.
\begin{equation}
\begin{split}
&      \mathsf{Err}_{-1,-1} = \mathbb{P}(w^{\top}X + b\geq 0 | X \sim D(-\mu-\psi)) \\  
& =F_W( w^{\top}( -\psi -\mu) + b ) \\ 
\end{split}
\end{equation}
\begin{itemize}
\item Case 1. $w^{\top}\psi \geq 0$, $b\geq 0$. In this case, observe that $ \mathsf{Err}_{-1,1}$ achieves the worst group error. Observe that $F_W$ is monotonic in $b$ so $b=0$ is optimal. Now we want to minimize $F_W( w^{\top}( \psi -\mu)$ subject to $\|w\|=1$ and $w^{\top}\psi \geq 0$. The first term takes smallest value when $w^{\top}( \psi)=0$ and second term takes smallest value when $-w^{\top}\mu =-\|\mu\|$. If $w= \hat{\mu}$, then both constraints are simultaneously satisfied as $\mu \perp \psi$. The error achieved as a result is $F_W(-\|\mu\|)$
\item Case 2. $w^{\top}\psi \leq 0$, $b\geq 0$. In this case, observe that $ \mathsf{Err}_{-1,-1}$ achieves the worst group error. Observe that $F_W$ is monotonic in $b$ so $b=0$ is optimal. Now we want to minimize $F_W( w^{\top}( -\psi -\mu)$ subject to $\|w\|=1$ and $w^{\top}\psi \leq 0$. The first term takes smallest value when $w^{\top}( \psi)=0$ and second term takes smallest value when $-w^{\top}\mu =-\|\mu\|$. If $w= \hat{\mu}$, then both constraints are simultaneously satisfied as $\mu \perp \psi$. The error achieved as a result is $F_W(-\|\mu\|)$
\item Case 3. $w^{\top}\psi \geq 0$, $b\leq 0$. In this case, observe that $ \mathsf{Err}_{1,-1}$ achieves the worst group error. Observe that $F_W$ is monotonic in $b$ so $b=0$ is optimal. Now we want to minimize $F_W( w^{\top}( \psi -\mu)$ subject to $\|w\|=1$ and $w^{\top}\psi \geq 0$. The first term takes smallest value when $w^{\top}( \psi)=0$ and second term takes smallest value when $-w^{\top}\mu =-\|\mu\|$. If $w= \hat{\mu}$, then both constraints are simultaneously satisfied as $\mu \perp \psi$. The error achieved as a result is $F_W(-\|\mu\|)$
\item Case 4. $w^{\top}\psi \leq 0$, $b\leq 0$. In this case, observe that $ \mathsf{Err}_{1,1}$ achieves the worst group error. Observe that $F_W$ is monotonic in $b$ so $b=0$ is optimal. Now we want to minimize $F_W( w^{\top}( -\psi -\mu)$ subject to $\|w\|=1$ and $w^{\top}\psi \leq 0$. The first term takes smallest value when $w^{\top}( \psi)=0$ and second term takes smallest value when $-w^{\top}\mu =-\|\mu\|$. If $w= \hat{\mu}$, then both constraints are simultaneously satisfied as $\mu \perp \psi$. The error achieved as a result is $F_W(-\|\mu\|)$
\end{itemize}
Therefore, $F_{W}(-\|\mu\|)$ is the lowest value for the error and is achieved by  $w= \hat{\mu}$. In fact, if the cdf of $F_W$ is strictly increasing, then $w= \hat{\mu}$ is the unique optimal solution.
\end{proof}

\section{Supplementary Materials for Empirical Findings}
\label{sec:supp_exp}

\subsection{Training details for the experiments}
The training procedure consists of two steps. We use the training strategy very similar to that in \citep{kirichenko2022last}. We process CelebA and Waterbirds dataset using the procedure used in \citep{idrissi2021simple}. We train in Pytorch using the same environment from \citep{idrissi2021simple} provided at \url{https://github.com/facebookresearch/BalancingGroups}. 

We first explain training of ERM and SS.  
\begin{itemize}
\item \textbf{Feature Learning} We take a pretrained ResNet-50 and fine tune a fresh linear layer on the target data (Waterbirds or CelebA). We use Adam optimizer with a learning rate of $10^{-4}$ and a weight decay of $10^{-3}$ and train for $10$ epochs with a batch size of $128$.  
\item \textbf{Linear Layer Learning} In this step, we train a fresh linear layer. The only difference between ERM and SS is that SS is trained on a balanced dataset obtained by subsampling. We use Adam optimizer with a learning rate of $10^{-2}$ and train for $100$ epochs with a batch size of $128$. 
\end{itemize}

In ERM-PCA and SS-PCA the first step is exactly the same. Before the second step of linear layer learning, we carry out PCA on the representations input to the last linear layer and retain the first four components as they explain $99$ percent of variance in the data. After this we carry out the second step with same parameters as above.

\subsection{Supplementary figures}

\begin{figure}[!h]
\centering
\includegraphics[width=3.5in]{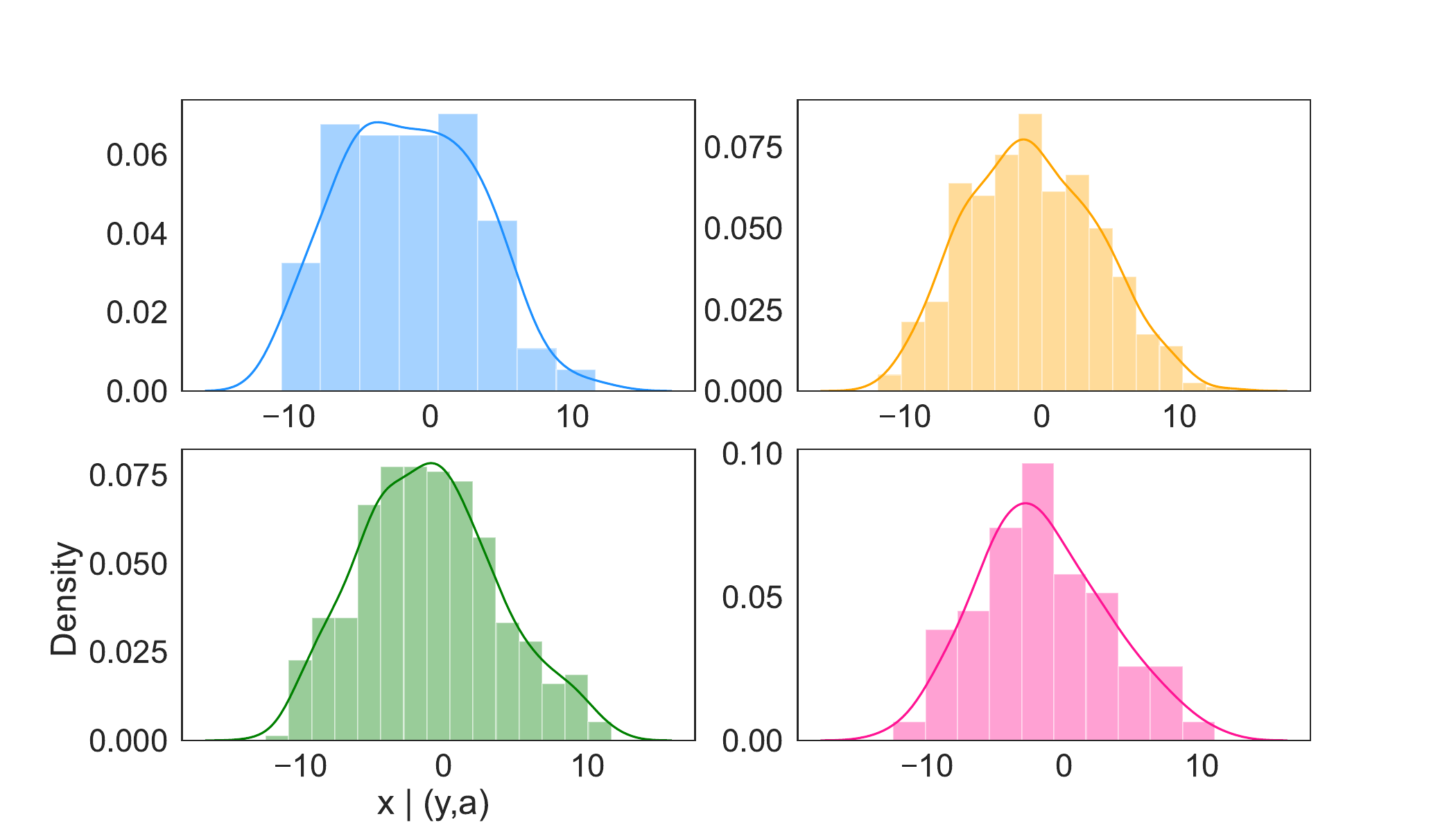}
          \caption{Waterbirds: Distribution of the second highest PCA feature.}
          \label{fig:tails_wb2}
\end{figure}

\begin{figure}[!h]
\centering
\includegraphics[width=3.5in]{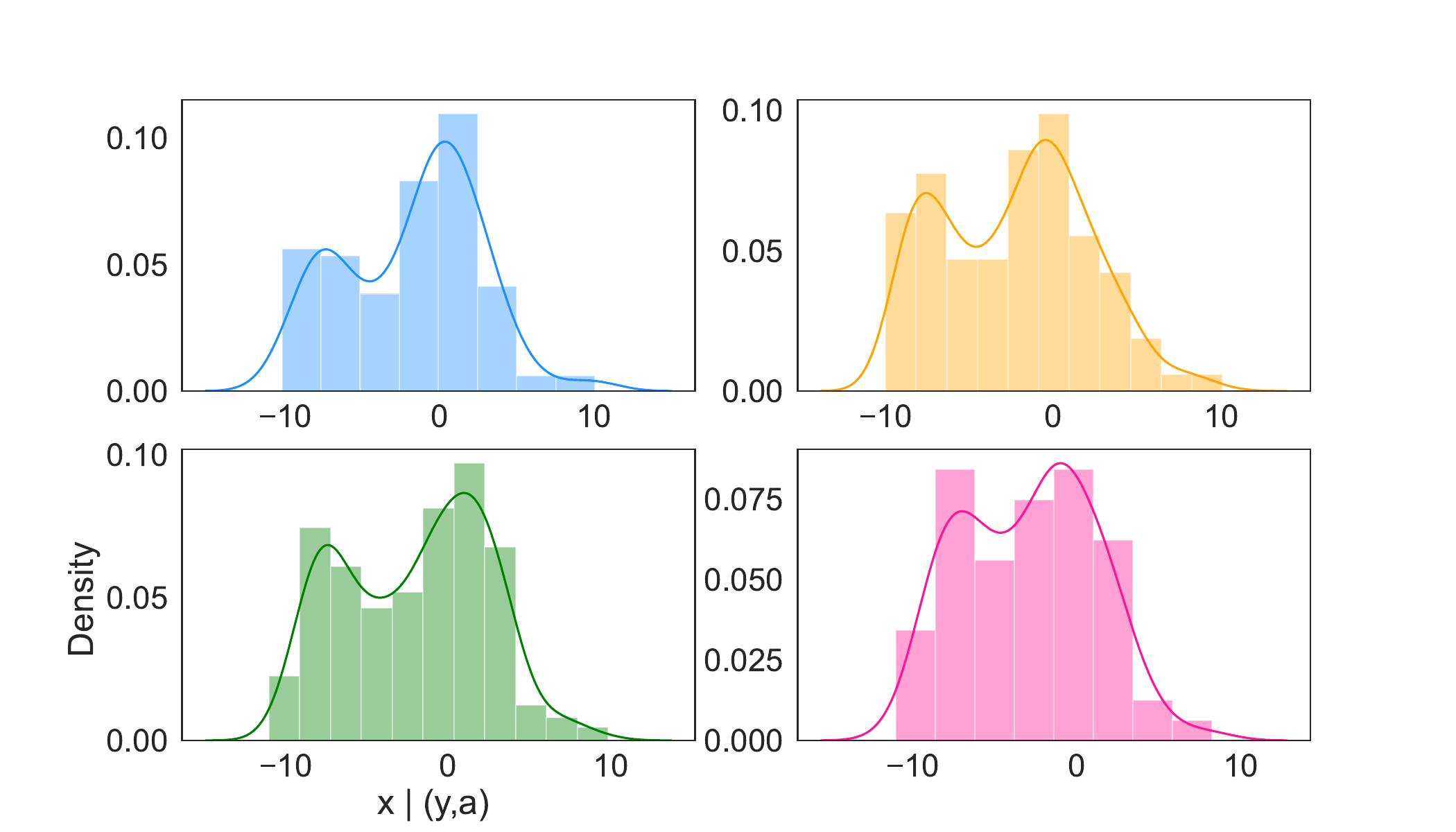}
          \caption{Waterbirds: Distribution of the third highest PCA feature.}
          \label{fig:tails_wb3}
\end{figure}

\begin{figure}[!h]
\centering
\includegraphics[width=3.5in]{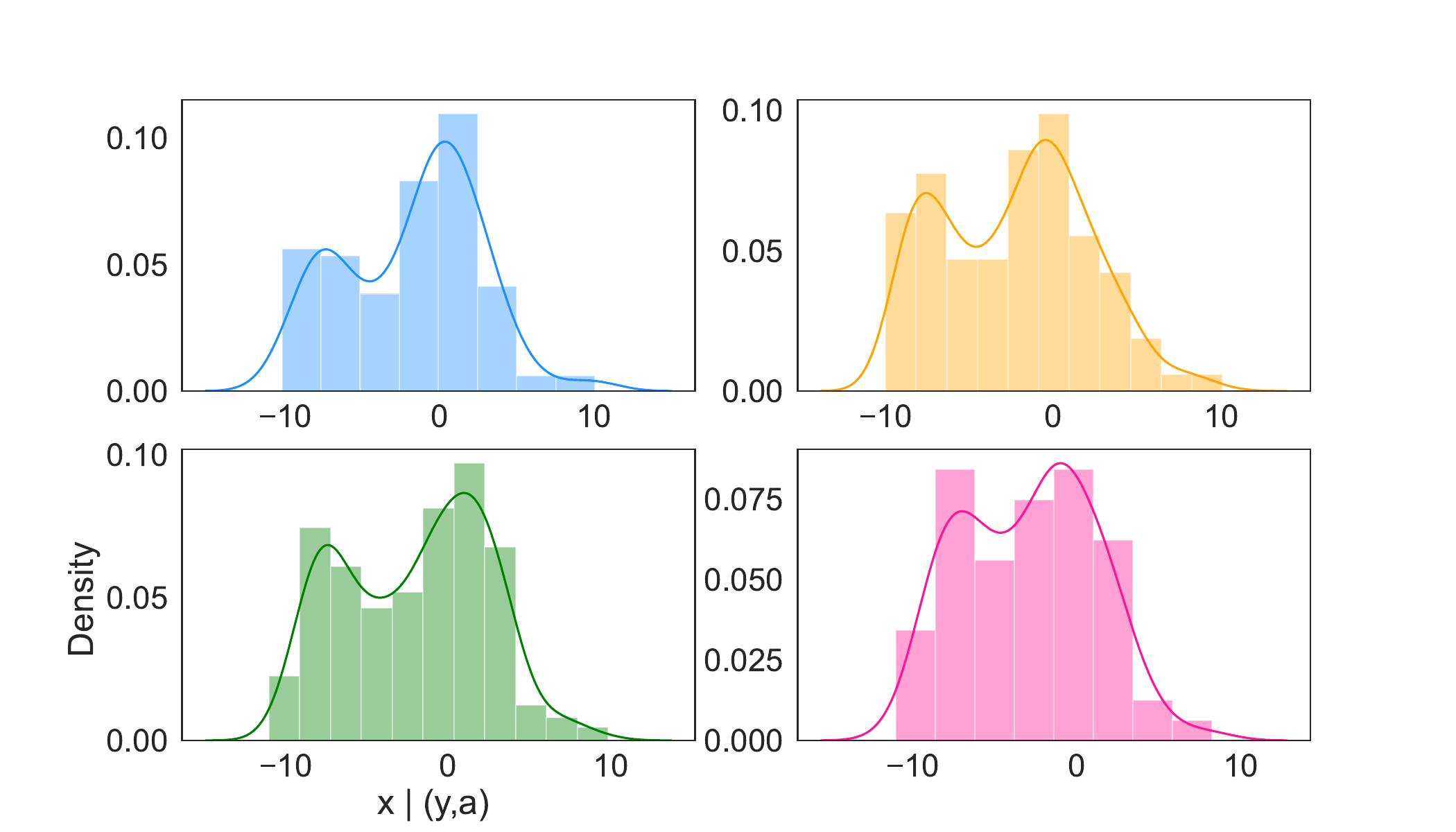}
          \caption{Waterbirds: Distribution of the fourth highest PCA feature}
          \label{fig:tails_wb4}
\end{figure}

\begin{figure}[!h]
\centering
\includegraphics[width=3.5in]{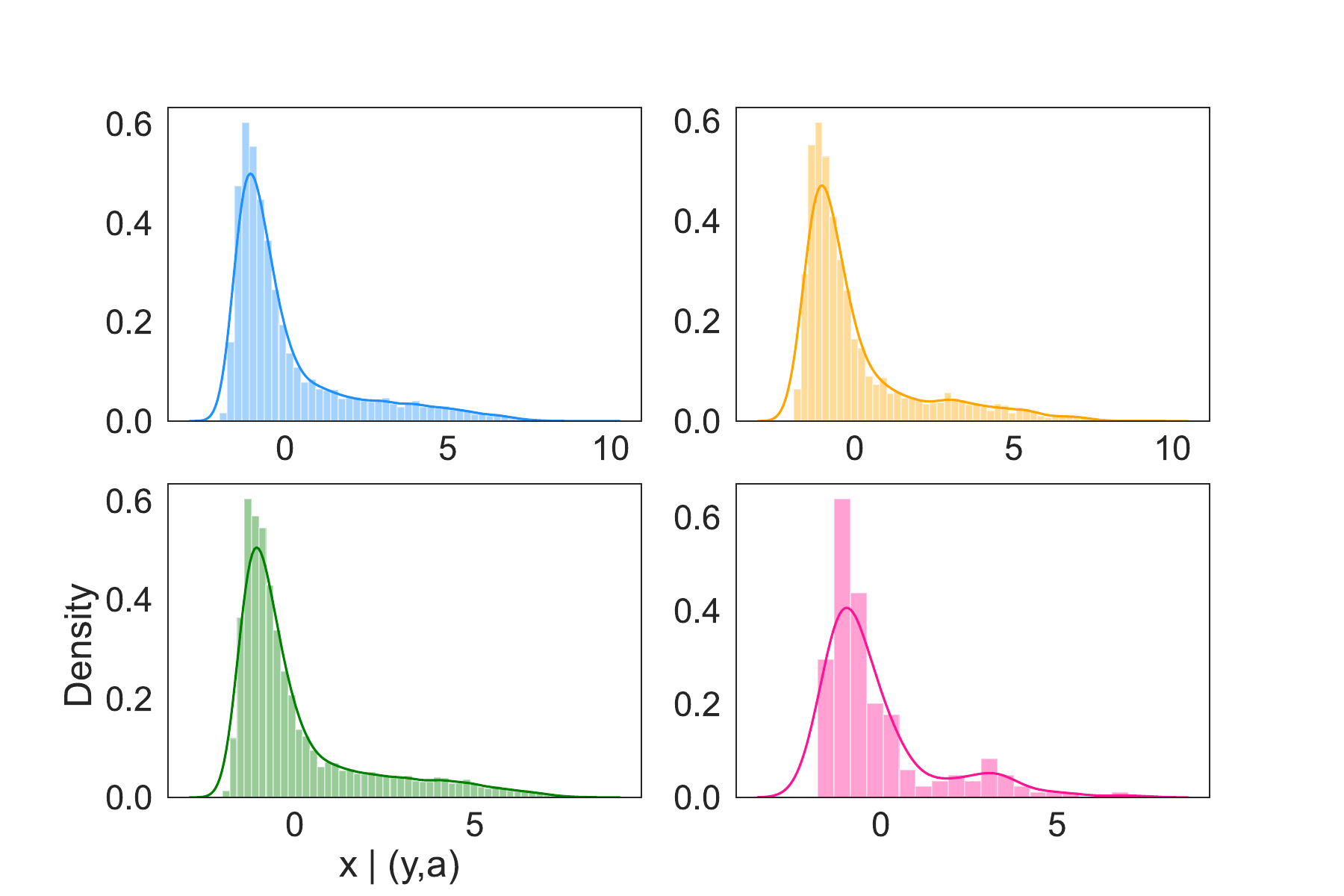}
          \caption{CelebA: Distribution of the highest PCA feature.}
          \label{fig:tails_cb1}
\end{figure}

\begin{figure}[!h]
\centering
\includegraphics[width=3.5in]{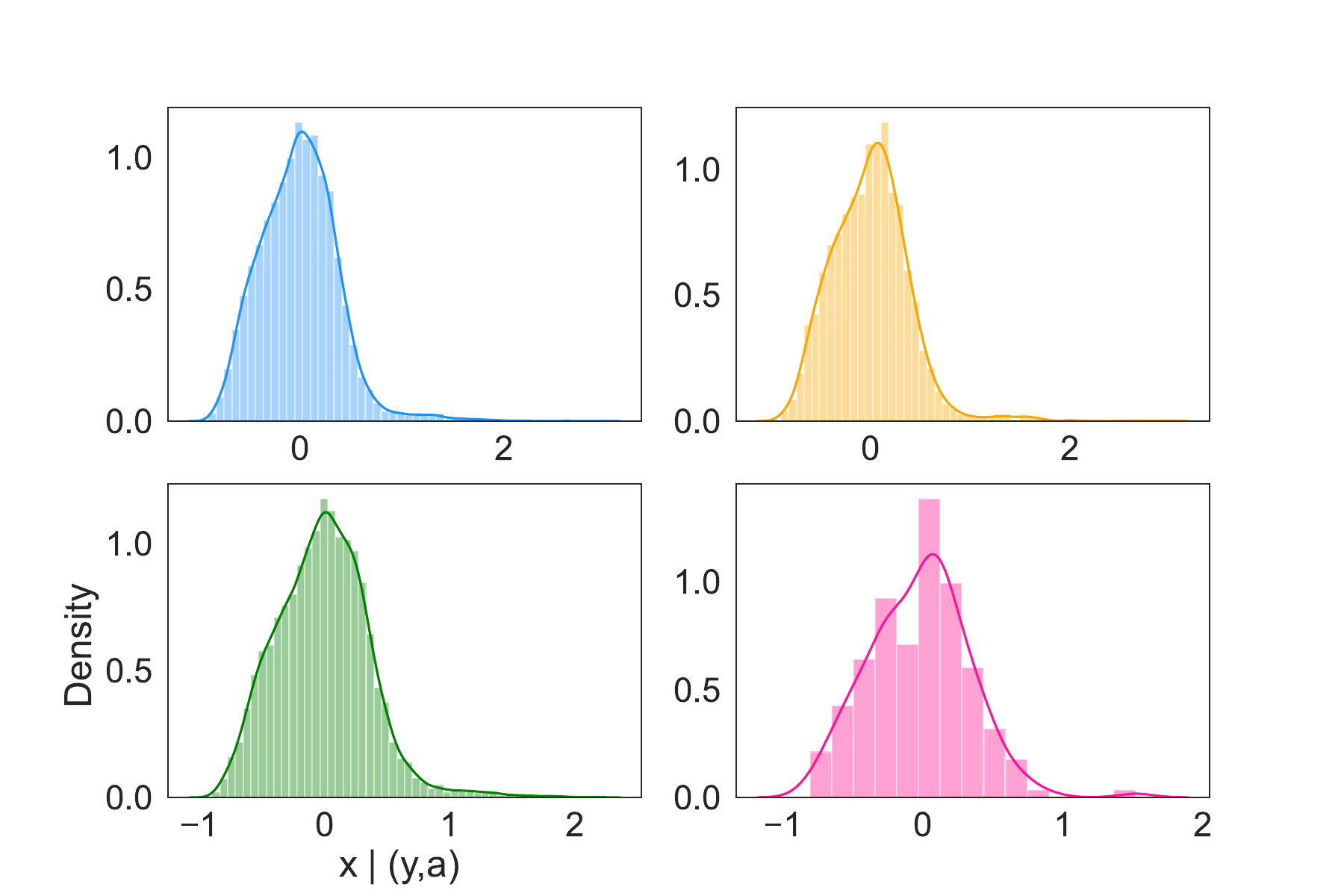}
          \caption{CelebA: Distribution of the second highest PCA feature.}
          \label{fig:tails_cb2}
\end{figure}

\begin{figure}[!h]
\centering
\includegraphics[width=3.5in]{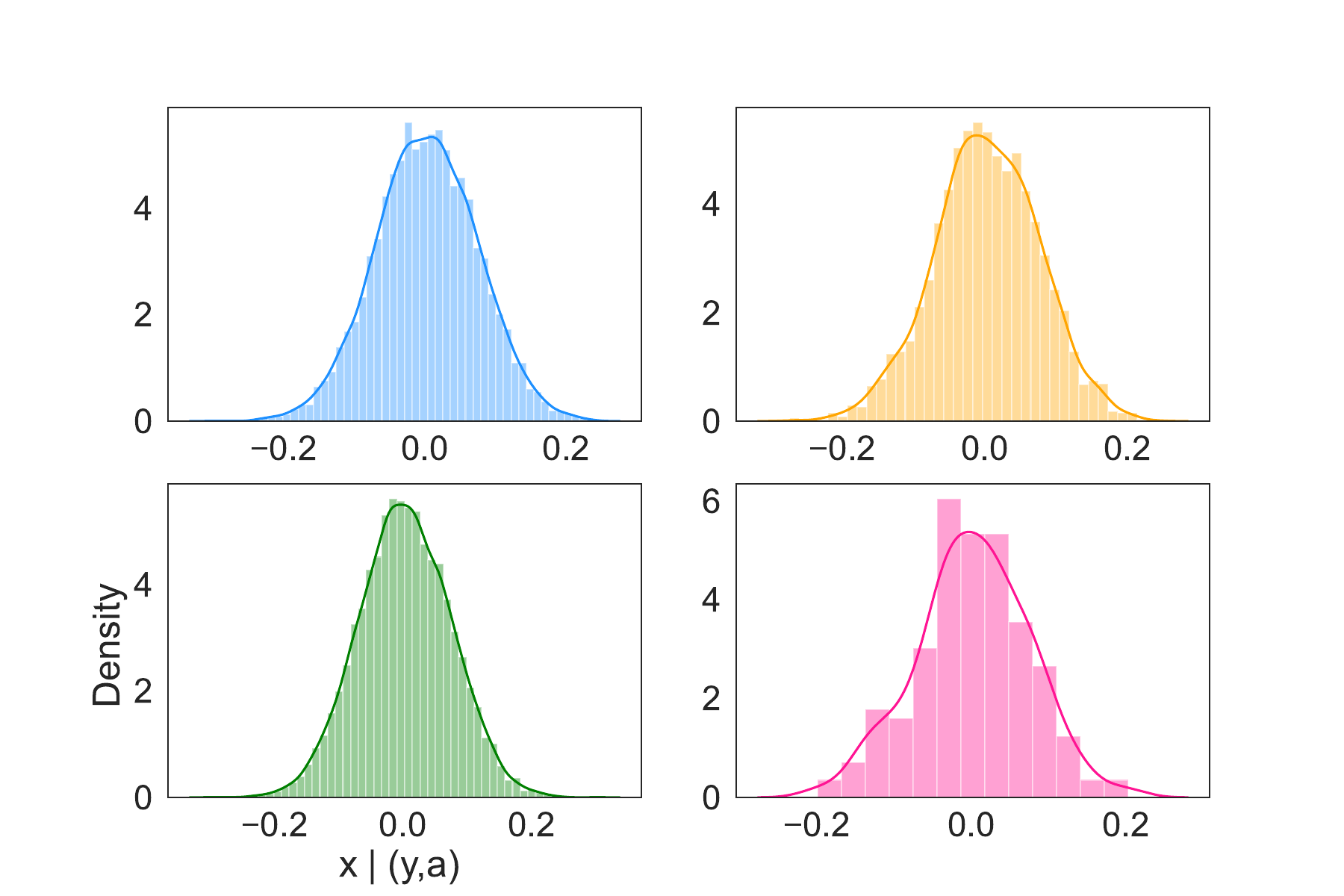}
          \caption{CelebA: Distribution of the third highest PCA feature.}
          \label{fig:tails_cb3}
\end{figure}

\begin{figure}[!h]
\centering
\includegraphics[width=3.5in]{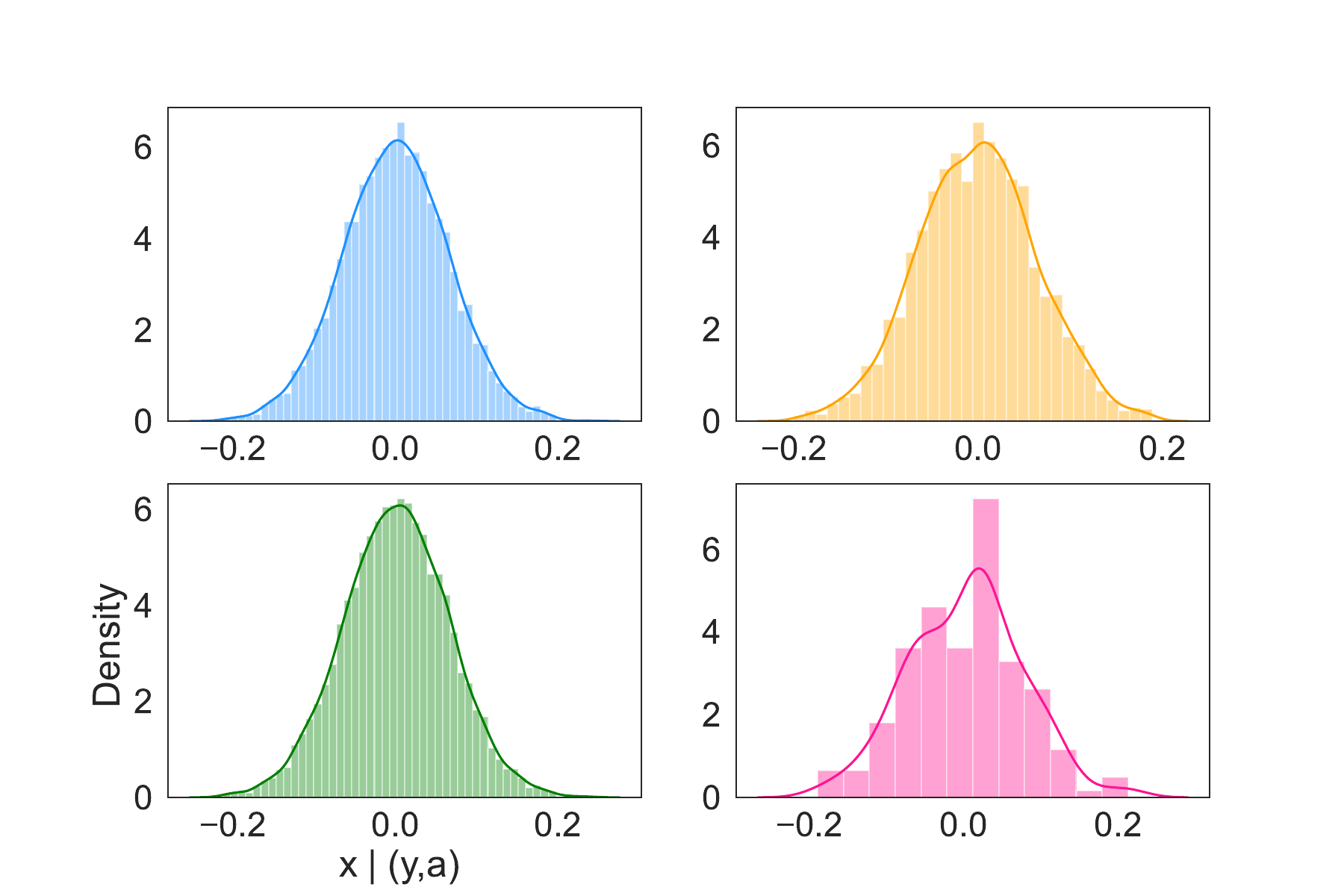}
          \caption{CelebA: Distribution of the fourth highest PCA feature.}
          \label{fig:tails_cb4}
\end{figure}

\end{document}